\documentclass[10pt]{article}
\usepackage[accepted]{tmlr}

\usepackage{amsmath,amssymb,amsthm,mathtools}
\usepackage{graphicx}
\usepackage{subcaption}
\usepackage{caption}
\usepackage{url}
\usepackage[ruled,vlined,linesnumbered]{algorithm2e}
\usepackage[most]{tcolorbox}
\usepackage{rotating}
\usepackage[table]{xcolor}
\usepackage{colortbl}
\usepackage{siunitx}
\usepackage{array}
\usepackage{booktabs}
\usepackage{multirow}
\usepackage{makecell}
\usepackage{enumitem}
\usepackage{thmtools}
\usepackage{thm-restate}
\usepackage{hyperref}
\usepackage{cleveref}
\usepackage{makecell}
\usepackage{MnSymbol}
\usepackage{tikz}
\usetikzlibrary{positioning, arrows.meta, shapes.geometric, fit, backgrounds, calc}

\definecolor{PCD}{HTML}{6A1B9A}
\definecolor{WS}{HTML}{EF6C00}
\definecolor{CAGrad}{HTML}{E53935}
\definecolor{PCGrad}{HTML}{6D4C41}
\definecolor{FAMO}{HTML}{1565C0}
\definecolor{MGDA}{HTML}{2E7D32}
\definecolor{AuxiNash}{HTML}{D66F97}
\definecolor{DBGD}{HTML}{0097A7}
\definecolor{DBGDa}{HTML}{9E9D24}
\definecolor{CooperC}{HTML}{37474F}
\definecolor{ProxGLc}{HTML}{8E8E8E}

\renewcommand{\thefootnote}{\fnsymbol{footnote}}

\definecolor{pcdhighlight}{RGB}{230,245,255}
\definecolor{headergray}{RGB}{240,240,240}
 
\newcolumntype{C}[1]{>{\centering\arraybackslash}p{#1}}
\newcolumntype{L}[1]{>{\raggedright\arraybackslash}p{#1}}

\newtheorem{theorem}{Theorem}[section]
\newtheorem{lemma}[theorem]{Lemma}
\newtheorem{proposition}[theorem]{Proposition}
\newtheorem{corollary}[theorem]{Corollary}
\newtheorem{definition}[theorem]{Definition}
\newtheorem{assumption}[theorem]{Assumption}
\newtheorem{remark}[theorem]{Remark}
\newtheorem{example}[theorem]{Example}
\newtheorem*{theorem*}{Theorem}

\newcommand{\R}{\mathbb{R}}
\newcommand{\conv}{\mathrm{conv}}
\newcommand{\cone}{\mathrm{cone}}

\newcommand{\norm}[1]{\left\|#1\right\|}
\newcommand{\T}{\top}
\newcommand{\vth}{\boldsymbol{\theta}}

\newcommand{\vw}{\mathbf{w}}
\newcommand{\vd}{\mathbf{d}}
\newcommand{\vg}{\mathbf{g}}
\newcommand{\vzero}{\mathbf{0}}
\newcommand{\va}{\mathbf{a}}
\newcommand{\vb}{\mathbf{b}}

\newcommand{\vp}{\mathbf{p}}
\newcommand{\vx}{\mathbf{x}}
\newcommand{\vy}{\mathbf{y}}

\newcommand{\gt}{\tilde{\mathbf{g}}}
\newcommand{\dt}{\tilde{\mathbf{d}}}
\newcommand{\dts}{\dt^{*}}
\newcommand{\vz}{\mathbf{z}}
\newcommand{\vu}{\mathbf{u}}
\newcommand{\vv}{\mathbf{v}}

\newtcolorbox{mybluebox}{
  colback=blue!5!white,
  colframe=blue!75!black,
  sharp corners,
  boxrule=1pt,
  fonttitle=\bfseries,
  enhanced
}

\title{Not All Objectives Are Born Equal: Priority-Constrained Descent for Hierarchical Multi-Objective Optimization}

\author{\name Dara Varam\textsuperscript{1,\,2\,*} \email b00081313@aus.edu \\
      \name Mohamed I. AlHajri\textsuperscript{1\,*} \email mialhajri@aus.edu \\[0.06in]
      \addr $^1$ Department of Computer Science \& Engineering, American University of Sharjah \\
      $^2$ Senseable City Laboratory, Massachusetts Institute of Technology \\
      $^{*}$ Equal contribution.
      }

\def\month{09}
\def\year{2026}
\def\openreview{\url{https://openreview.net/forum?id=HT01yGHLEt&}}

\begin{document}

\maketitle

\renewcommand{\thefootnote}{\fnsymbol{footnote}}
\renewcommand{\thefootnote}{\arabic{footnote}}

\begin{abstract}

    Deep learning problems rarely involve objectives that are equal in importance. A primary objective defines the goal, whilst secondary objectives, such as sparsity, compression, or robustness, constrain the solution. While existing multi-objective methods have proven effective in practice, they have a clear symmetry problem and neglect the inherent objective hierarchy built into these objective spaces. We introduce Priority-Constrained Descent (PCD), a gradient-based optimization framework designed to explicitly exploit hierarchical objective structures. PCD preserves the direction of primary descent whilst allowing for the minimal distortion necessary to guarantee first-order progress on secondary objectives, controlled by a single $\tau \in [0,1]$ that dictates the strength of the distortion. The resulting formulation is asymptotically invariant to objective scaling and admits exact closed-form solutions for problems with two and three objectives. We evaluate PCD within structured network compression settings, unstructured sparsity and low-rankness, and across a variety of synthetic experiments, showing stronger trade-offs and better per-objective performance than existing multi-objective methods as well as tuned application-specific baselines, further exhibiting the interpretable trade-off that $\tau$ provides.
   
\end{abstract}

\section{Introduction}

\begin{figure*}[t]
    \centering
    \includegraphics[width=1\linewidth]{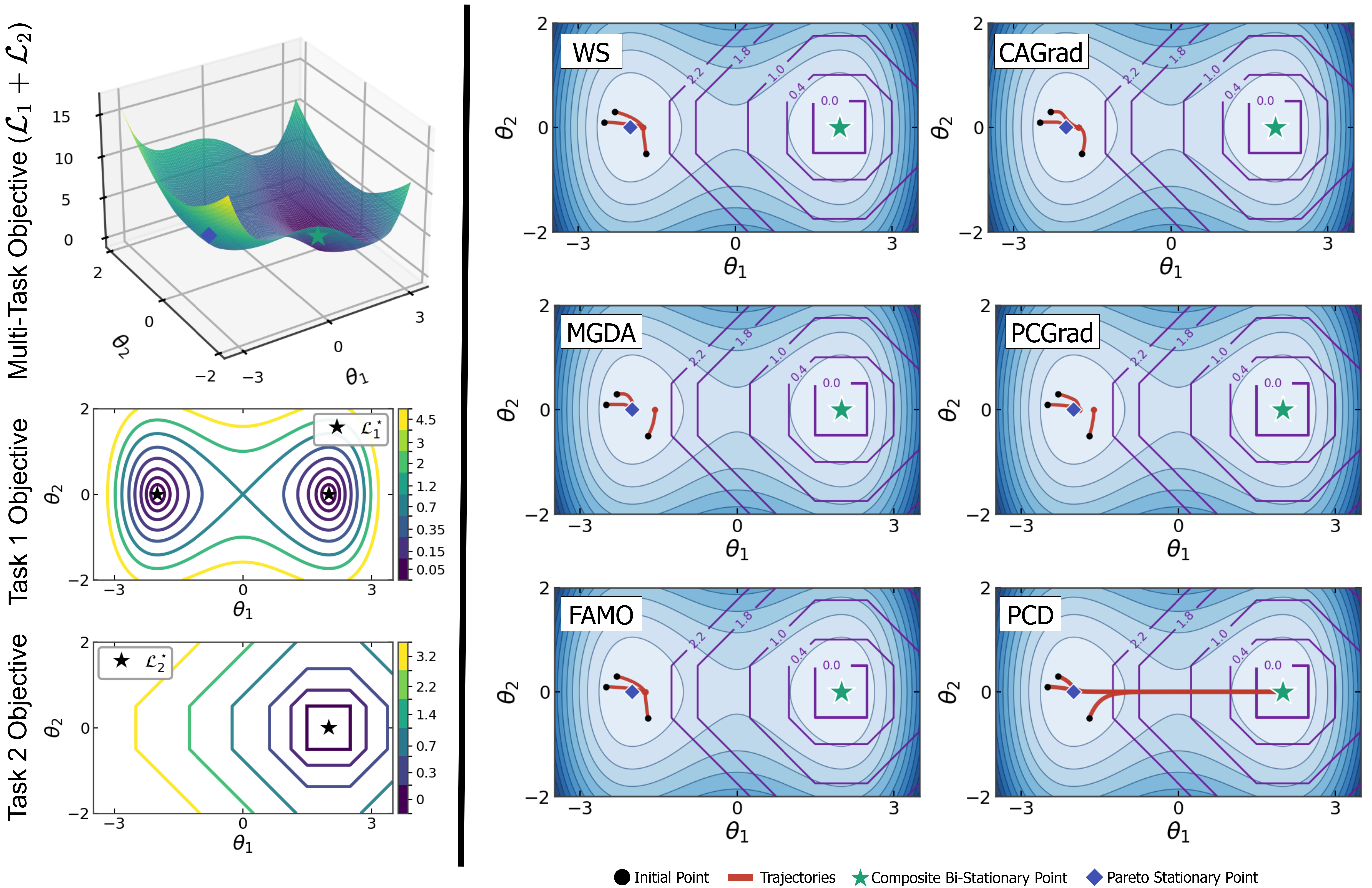}
    \caption{PCD on a 2D toy problem with two local minima of $\mathcal{L}_1$, only one of which is also a minimum of $\mathcal{L}_2$. From three initializations, weighted sum, MGDA, PCGrad, CAGrad, and FAMO all terminate at a Pareto-stationary point at which $\mathcal{L}_2$ remains non-zero (the conflict equilibrium of Sec.~\ref{sec:prelim}). From all three initializations, PCD crosses the saddle barrier in parameter space and reaches the composite bi-stationary point where $\nabla \mathcal{L}_1 = \nabla \mathcal{L}_2 = 0$, which is a minimizer for both $\mathcal{L}_1$ and $\mathcal{L}_2$.}
    \label{fig:demonstration}
\end{figure*}

Optimization in modern deep learning is fundamentally hierarchical. While models are typically evaluated based on their primary objective performance (such as accuracy in the case of classification, mean squared error for regression, or fidelity for generative modeling), solutions often need to satisfy a number of \emph{secondary} conditions to be usable in practice. In the case of pruning, the primary task loss is paired with a sparsity-inducing regularizer that heavily constrains the final parameterization~\citep{he2023structured}. Low-rankness pairs the primary objective with a rank-constraining regularizer to force the network parameters into a low-dimensional subspace~\citep{idelbayev2020low}. Quantization-aware training anchors the task against a fidelity term penalizing drift from discrete values~\citep{abushahla2025neural}. Similarly, continual learning balances fitness on new data against a replay objective preserving past knowledge~\citep{wang2024comprehensive}, while federated learning manages local accuracy alongside communication constraints~\citep{zhang2021survey}. In all these examples, the structure of the optimization problem is the same: the relationship between objectives is asymmetric by design. One loss represents the core capability, while the others dictate the boundaries within which that capability must be learned. Nonetheless, standard multi-objective optimization (MOO) frameworks ignore this structural truth.

The conventional paradigm for handling multiple objectives is to collapse all objectives into a single weighted summation $\mathcal{L}=\sum_iw_i \mathcal{L}_i$~\citep{marler2010weighted}. Here, importance is encoded implicitly through the weights $w_i$, but this requires fragile hyperparameter sweeps and struggles to access non-convex regions of the Pareto front. Gradient-based methods~\citep{desideri2009multiple, pcgrad, cagrad, liu2023famo} avoid these pitfalls but share the issue of treating objectives symmetrically. Consequently, they target Pareto stationarity, where no objective can be improved without harming others. For hierarchical problems, we argue that this is too restrictive a criterion. It can lead to what we term a \emph{conflict equilibrium}, a sub-optimal state where progress on the primary objective is halted because its gradient is counterbalanced by the push-back of constraints with equal authority (formalized in Sec.~\ref{sec:prelim}). This is a failure mode and not a feasible solution for practitioners. 

What hierarchical optimization problems call for is an optimizer that imposes priority at the gradient level, avoiding the issue of halting prematurely by overly enforcing secondary objectives. To this end, we introduce Priority-Constrained Descent (PCD), a framework designed to follow the primary gradient as closely as possible, deflecting only as much as necessary to keep secondary objectives on track.  A single scalar tolerance $\tau \in [0,1]$ controls this deflection. Three outcomes follow from this design (developed in Sec.~\ref{sec:pcd}): first, the update anchors to the primary task and reduces to plain gradient descent whenever primary descent alone provides enough progress on all secondaries  (Cor.~\ref{cor:inactive}). Second, every step guarantees a $\tau$-fraction of first-order progress on each secondary for the direction it computes (Prop.~\ref{prop:progress}), and third, the normalized update direction vanishes only at \emph{composite multi-stationary} (CMS) points, a strictly stronger condition than Pareto stationarity. At CMS points, every objective is individually stationary (Thm.~\ref{thm:cms}), so conflict equilibria are never fixed points of the PCD direction.

This geometric advantage is immediately visible in a 2D landscape (Fig.~\ref{fig:demonstration}). When faced with a saddle barrier separating a local and a global minimum, all symmetric baselines stall at a Pareto-stationary conflict equilibrium where $\mathcal{L}_2$ remains non-zero. PCD is the only method tested that pushes through the secondary constraints to reach the composite bi-stationary point where $\nabla \mathcal{L}_1 = \nabla \mathcal{L}_2 = 0$, successfully minimizing both objectives. We note that although $\nabla \mathcal{L}_1$ comes close to (but never truly reaches) $0$, PCD continues to explore the combined landscape until both objectives are minimized.

We summarize the core contributions of this study as follows:

\begin{itemize}

    \item Priority-Constrained Descent: A gradient-based optimizer that solves a small convex quadratic program (QP) to minimally distort the primary gradient, enabling progress for the primary objective whilst guaranteeing first-order progress on secondaries for the direction it computes. The anchored-projection construction follows the established dynamic-barrier algorithm of \citet{gong2021bi} and \citet{ye2022pareto} whilst expanding to per-objective, scale-normalized margins that make one dimensionless tolerance meaningful across heterogeneous secondaries. We derive exact closed-form solutions for $2$ and $3$ objectives, a feasibility characterization at general $K$, asymptotic invariance to per-objective scaling, and the fixed-point analysis of the resulting update.
    
    \item Structured Network Compression: We show the dominance of PCD over existing MOO frameworks across a variety of architectures and datasets for structured pruning tasks, showing minimal drop in accuracy compared to unpruned baselines. On CIFAR-100 trained on the ResNet-34 architecture, PCD maintains a 70.9\% accuracy at 90\% sparsity (compared to an unpruned initial accuracy of 73.3\%), where baselines collapse. Against tuned pruning-specific and constrained baselines, PCD exceeds accuracy at matched sparsity throughout the compression regime, with other comparisons coming within seed noise. 
    
    \item Unstructured Pruning and Low-Rankness: On the tri-objective task with $\ell_1$ and nuclear norms, PCD is decisive on two of the three axes and is competitive on the third, maintaining accuracies of 71\% for CIFAR-100 trained on ResNet-34 with an unstructured sparsity of 92.6\% and an effective rank of 24.5, whilst baselines once again collapse in terms of accuracy and over-compress. 
    \item Synthetic Verification: In known objective spaces, we verify the mechanisms of PCD directly, showing asymptotic scale invariance, feasibility as we increase the number of objectives, and the fixed-point separation from symmetric methods. We observe empirically that PCD reaches CMS points in all cases without infeasibility issues. 
    \end{itemize}

\section{Related Works}

Modern gradient-based MOO rests on a target established by~\citet{Gale1951}. A point $\theta^*$ is \emph{Pareto stationary} if the origin lies in the convex hull of the objective gradients, meaning no common descent direction exists. Traditional MOO frameworks aim at reaching a Pareto-stationary point, with more nuanced techniques allowing for some degree of preference over the Pareto front itself. In the pursuit of Pareto stationarity, the field has branched into two distinct strategies: scalarization and direct gradient manipulation. 

Classical weighted-sum scalarization \citep{marler2010weighted} collapses the objective space prior to optimization. While any stationary point with strictly positive weights (assigned to each objective) is technically Pareto stationary, this approach cannot reach non-convex regions of the Pareto front \citep{miettinen1999nonlinear} and requires brittle hyperparameter tuning. A recent study reports that a well-tuned weighted sum frequently matches the performance of more elaborate alternatives, suggesting that the gains attributed to complex MOO algorithms are often artifacts of the training process rather than fundamental algorithmic superiority~\citep{xin2022current}. Regardless, to bypass these pathologies, \citet{fliege2000steepest} proposed a second strategy: discarding scalar weights entirely to compute the steepest Pareto-descent direction directly from the gradients. Nearly every modern gradient-based MOO method inherits this direct-direction strategy. However, while they successfully escape scale sensitivity and access broader regions of the Pareto front, they preserve a structural assumption that contradicts the reality of hierarchical deep learning: that objectives are symmetric and inherently interchangeable.

\paragraph{In Pursuit of Symmetric Compromise.} The direct-direction lineage largely stems from MGDA \citep{desideri2009multiple}, which operationalizes Pareto stationarity by identifying the minimum-norm element of the gradients' convex hull. \citet{sener2018multi} scaled MGDA to deep learning via a Frank--Wolfe solver \citep{pmlr-v28-jaggi13}, while MoCo \citep{fernando2023mitigating, fernando2024variance} mitigated the gradient bias inherent to stochastic MGDA variants, recovering provable convergence in non-convex settings. In parallel, a separate family of methods takes the gradients as given and attempts to condition them before they are combined. Magnitude-based approaches balance objectives by scaling losses via learned homoscedastic noise \citep{kendall2018multi}, driving gradient norms to a common rate \citep{chen2018gradnorm}, ensuring equal gradient projections \citep{liu2021towards}, or tracking loss histories to balance step decreases \citep{liu2023famo}. Geometry-based approaches attempt to resolve gradient misalignment by projecting conflicting task gradients onto each other's normal planes \citep{pcgrad}, or by finding an optimal direction within an averaged-gradient ball \citep{cagrad, senushkin2023independent}.

However, even when methods attempt to prioritize certain outcomes, they do so symmetrically. Preference-steering methods like EPO \citep{mahapatra2021exact}, PMGDA \citep{zhang2025pmgda}, and hypernetwork adaptations \citep{navon2020learning} target specific Pareto-optimal points via controlled ascent or test-time network mapping. Yet, the preference vector enters these formulations as a symmetric annotation on the objectives, not as a structural distinction. Across all these methods, whether modifying directions, magnitudes, or geometries, the fundamental goal is compromise. As noted earlier, enforcing this symmetric compromise in hierarchical settings is precisely what traps optimizers in sub-optimal conflict equilibria.

\paragraph{Breaking The Symmetry.} The shift away from symmetric MOO has begun to surface in recent game-theoretic and constrained-optimization literature. Nash MTL \citep{navon2022multi} applies the Nash Bargaining solution to produce proportionally fair update directions. Building on this, AuxiNash \citep{shamsian2023auxiliary} marks a critical conceptual pivot: it extends Nash-MTL to asymmetric settings by learning task-specific bargaining power based on each auxiliary's contribution to the primary objective. It explicitly names the limitation that symmetric methods treat structurally unequal objectives as peers. Similarly, constrained-optimization frameworks like Cooper \citep{gallego2025cooper} establish hierarchy through Lagrangian multiplier dynamics, updating learned dual variables alongside primal parameters.

\paragraph{Anchored Projections and Dynamic Barriers.} Closest to our construction is the family of methods that anchor on one gradient and impose inner-product constraints on others. Dynamic Barrier Gradient Descent \citep[DBGD;][]{gong2021bi} is PCD's $K=2$ instance on raw gradients, developed for constrained and lexicographic optimization with convergence guarantees in the raw-gradient setting. Pareto Navigation Gradient Descent \citep{ye2022pareto} similarly introduces an anchored program with general $K$ objectives, with a shared and scheduled approach rather than per-objective margins. Within the field of continual learning, methods such as GEM \citep{lopezpaz2017gem} and A-GEM \citep{chaudhry2019agem} have been developed to combat catastrophic forgetting, but are algorithmically equivalent to PCD with $\tau = 0$. $\varepsilon$-SOFT-GEM \citep{hu2020softgem} adds a unit-normalized positive margin at $K=2$ in the same continual-learning setting, while on the application side \citet{gallegoposada2022controlled} make the argument closest to ours, controlling the compression level directly rather than tuning a penalty coefficient, via constrained $L_0$ sparsity. Within this line, PCD contributes the per-objective scale-normalized margins that make a single dimensionless parameter $\tau$ meaningful across a variety of applications. 

While AuxiNash and Cooper recognize the asymmetric reality of modern deep learning, and anchored projection methods present application-specific alternatives, they enforce hierarchy through learned bargaining weights, Lagrangian dynamics, or by exploiting application-specific mechanisms. PCD discards the symmetry assumption at a much lower level. Rather than relying on learned coefficients, PCD embeds hierarchy directly into the geometry of the update. By calculating the closed-form Euclidean projection of the primary gradient onto the secondary-progress polyhedron, PCD yields a per-step first-order progress guarantee on secondary objectives for the direction it feeds the optimizer (Sec.~\ref{sec:pcd-scope}). It fundamentally rejects the symmetric compromise of Pareto stationarity: the fixed points of its normalized direction are composite multi-stationary rather than merely Pareto stationary (Sec.~\ref{sec:pcd-cms}).

\section{Preliminaries and Definitions}
\label{sec:prelim}

\subsection{Problem Setup}

We consider optimization problems over a parameter space $\vth \in \R^n$ with $K \geq 2$ objectives arranged in an asymmetric hierarchy:
\begin{equation}\label{eq:objectives}
  \underbrace{\mathcal{L}_1(\vth)}_{\text{primary}}, \quad
  \underbrace{\mathcal{L}_2(\vth), \ldots, \mathcal{L}_K(\vth)}_{\text{secondary}} .
\end{equation}
We use $[K] = \{1,\ldots,K\}$ and $[K]_{-1} = \{2,\ldots,K\}$ for index sets, $\norm{\cdot}$ for the Euclidean norm, $\partial \mathcal{L}$ for the convex subdifferential, $\conv\{\cdot\}$ for the convex hull, and $\cone\{\cdot\}$ for the conic hull.

\begin{assumption}[Objective regularity]\label{ass:reg} \hfill
\begin{enumerate}[label=(\roman*),nosep,leftmargin=*]
  \item \emph{Primary:} $\mathcal{L}_1: \mathbb{R}^n \to \mathbb{R}$ is continuously differentiable with gradient $\mathbf{g}_1 = \nabla \mathcal{L}_1(\boldsymbol{\theta})$.
  \item \emph{Secondary:} For each $j \in [K]_{-1}$, $\mathcal{L}_j: \mathbb{R}^n \to \mathbb{R} \cup \{+\infty\}$ is proper, convex, and lower semi-continuous, with subgradient $\mathbf{g}_j \in \partial \mathcal{L}_j(\boldsymbol{\theta})$ available at every $\boldsymbol{\theta} \in \operatorname{dom}(\mathcal{L}_j)$.
\end{enumerate}
\end{assumption}

This assumption encompasses a broad class of convex objectives frequently encountered in optimization and machine learning, including both smooth and non-smooth formulations. More generally, any proper, convex, lower semi-continuous objective with accessible subgradients falls within the framework. In particular, the secondary objectives may represent structural regularizers, auxiliary convex penalties, or constraint-like terms commonly used in modern machine learning, such as the $\ell_1$ norm, Group Lasso, nuclear norm, total variation, and hinge-type losses. The convexity assumption ensures a well-defined subgradient geometry and preserves the convex feasibility structure underlying the PCD projection step. 

\subsection{Stationarity Concepts}

\begin{definition}[Pareto stationarity]\label{def:pareto}
A point $\vth$ is \emph{Pareto stationary} if there exists no common ascent direction $\vd \in \R^n$ satisfying $\vg_i^\T \vd > 0$ for all $i \in [K]$ simultaneously, where $\vg_1 = \nabla \mathcal{L}_1(\vth)$ and $\vg_i \in \partial \mathcal{L}_i(\vth)$ for $i \geq 2$. Equivalently, 
\begin{equation}\label{eq:pareto-convhull}
  \vzero \in \conv\{\vg_1, \vg_2, \ldots, \vg_K\}.
\end{equation}
\end{definition}

\begin{definition}[Composite multi-stationarity]\label{def:cms}
A point $\vth$ is \emph{composite multi-stationary} (CMS) with respect to the hierarchy $(\mathcal{L}_1; \mathcal{L}_2, \ldots, \mathcal{L}_K)$ if
\begin{equation}\label{eq:cms}
  \nabla \mathcal{L}_1(\vth) = \vzero \quad\text{and}\quad \vzero \in \partial \mathcal{L}_j(\vth) \quad\text{for all } j \in [K]_{-1}.
\end{equation}
For $K=2$, we use \emph{composite bi-stationarity} (CBS).
\end{definition}

CMS asks that every objective be individually stationary at $\vth$. Pareto stationarity, by contrast, admits \emph{conflict equilibria} at which competing gradients cancel within the convex hull (Eq.~\eqref{eq:pareto-convhull}) even though no individual $\vg_i$ vanishes. Every CMS point is Pareto stationary, but not conversely: in dimension one, $\mathcal{L}_1(\theta) = \tfrac{1}{2}\theta^2$ and $\mathcal{L}_2(\theta) = \tfrac{1}{2}(\theta-1)^2$ at $\theta = \tfrac{1}{2}$ have gradients $\tfrac{1}{2}$ and $-\tfrac{1}{2}$ that cancel in convex combination, so the point is Pareto stationary, but $\nabla \mathcal{L}_1(\tfrac{1}{2}) = \tfrac{1}{2} \neq 0$, so it is not CMS. Conflict equilibria of this kind are exactly the points at which symmetric MOO methods such as MGDA \citep{desideri2009multiple} can halt without making any single objective stationary, and exactly the regime PCD's QP is designed to rule out for its normalized direction (Sec.~\ref{sec:pcd-cms}; the deployed update carries the scope caveats of Sec.~\ref{sec:pcd-scope}).

\section{Priority-Constrained Descent}
\label{sec:pcd}

Based on the principle of anchoring to the primary gradient and admitting corrections along the secondaries to ensure progress, we now present the algorithmic framework for PCD. This framing of ``corrections'' should be carefully defined: in Euclidean distance, it is simply the resultant direction that lies as close as possible to the primary gradient vector. Coupled with the role of the tolerance parameter $\tau \in [0,1]$, we formalize the PCD update rule in this section. The development proceeds as follows: Sec.~\ref{sec:pcd-qp} states the QP and identifies its solution as the Euclidean projection of the primary gradient onto a polyhedron of secondary-progress half-spaces. Sec.~\ref{sec:pcd-geom} extracts the KKT identity that makes the asymmetric structure of the update explicit, and Sec.~\ref{sec:pcd-progress} draws the per-step secondary-progress guarantee directly from it. Sec.~\ref{sec:pcd-ema} then introduces the gradient normalization that gives $\tau$ a scale-invariant interpretation across objectives of vastly different magnitude. Sec.~\ref{sec:pcd-k2} specializes to two objectives, where the QP admits a closed-form solution and $\tau$'s geometric meaning becomes quantitative through an analytical threshold at which primary descent breaks down, and Sec.~\ref{sec:pcd-algorithm} states the full algorithm. Sec.~\ref{sec:pcd-scope} delimits the scope of the guarantees, and Sec.~\ref{sec:pcd-cms} closes with the endpoint characterization of the normalized direction, under which conflict equilibria are not fixed points. Comparison with existing gradient-based MOO methods is deferred to Sec.~\ref{sec:comparison}.

\subsection{The PCD Quadratic Program}
\label{sec:pcd-qp}

\begin{mybluebox}
\begin{definition}[PCD quadratic program]\label{def:pcd-qp}
Given normalized gradients $\gt_1, \ldots, \gt_K$ and a scalar tolerance $\tau \in [0,1]$, the PCD update direction is
\begin{equation}\label{eq:pcd-qp}
  \dts \;=\; \arg\min_{\dt \in \R^n} \; \tfrac{1}{2}\norm{\dt - \gt_1}^2 \quad\text{s.t.}\quad \gt_j^\T \dt \;\geq\; \tau \norm{\gt_j}^2 \;\;\forall\, j \in [K]_{-1}.
\end{equation}
\end{definition}
\end{mybluebox}

Among all directions guaranteeing at least a $\tau$-fraction of maximum normalized first-order progress on every secondary, $\dts$ is the one closest in Euclidean distance to the primary gradient $\gt_1$. The secondary objectives thus act as \emph{half-space constraints} on the primary update, not as symmetric competitors. The feasible set
\begin{equation}\label{eq:feas}
  F_\tau(\vth) \;=\; \bigcap_{j=2}^{K} H_j, \qquad H_j = \{\dt \in \R^n : \gt_j^\T \dt \;\geq\; \tau \norm{\gt_j}^2\},
\end{equation}
is a closed convex polyhedron, possibly empty. For $K = 2$ with $\gt_2 \neq \vzero$, $F_\tau(\vth)$ is nonempty for every $\tau \in [0,1]$. For $K > 2$, infeasibility can occur when the secondary-progress constraints are mutually incompatible. Geometrically, for $\tau > 0$ this occurs exactly when the origin admits a nontrivial non-negative representation by the nonzero secondary gradients, i.e.\ when $\vzero \in \conv\{\gt_j : j \in [K]_{-1},\ \gt_j \neq \vzero\}$ (Theorem~\ref{thm:feas}). This is an intrinsic limitation of imposing simultaneous directional-progress requirements rather than a failure of the solver. The full feasibility analysis via Farkas' lemma~\citep{dinh2014farkas}, together with a sufficient pairwise-angle condition for $K > 2$, is given in App.~\ref{app:feasibility}.

\paragraph{Relation to dynamic-barrier methods.} The construction of Eq.~\eqref{eq:pcd-qp} is not itself new. For $K= 2$, it coincides with the dynamic-barrier update of \citet{gong2021bi} instantiated with the barrier $\phi = \beta\norm{\nabla g}^2$, evaluated on raw rather than normalized gradients. The two knobs are related by the exact per-step identity $\beta_j = \tau\, s_j / s_1$ (time-varying under the EMA). The dynamic-barrier update is also \citet{gong2021bi}'s instrument for constrained and \emph{lexicographic} optimization, for which they prove convergence in the raw-gradient setting. Theorem~\ref{thm:kkt} and its KKT characterization should therefore be read as the specialization of a known construction to per-objective, scale-normalized margins. Since Algorithm~\ref{alg:pcd} additionally rescales the direction magnitude, the $K = 2$ equivalence is direction-only: the same search direction under a different, time-varying effective step length. The empirical comparison against \citet{gong2021bi} (including their constrained ($\alpha$) branch) is reported in Sec.~\ref{sec:tuned-baselines}.

\subsection{Geometry and KKT Structure}
\label{sec:pcd-geom}

Eq.~\eqref{eq:pcd-qp} is a convex QP with affine inequality constraints. Its solution admits a geometric characterization that locates $\dts$ in space.

\begin{mybluebox}
\begin{theorem}[Projection and KKT characterization]\label{thm:kkt}
Suppose $F_\tau(\vth) \neq \emptyset$. Then $\dts$ is the Euclidean
projection of the primary gradient onto the secondary-progress feasible
set,
\begin{equation}\label{eq:projection}
  \dts \;=\; \mathrm{Proj}_{F_\tau(\vth)}(\gt_1).
\end{equation}
Equivalently, there exist multipliers $\mu_j^* \geq 0$
($j \in [K]_{-1}$) such that
\begin{equation}\label{eq:kkt-stationarity}
  \dts \;=\; \gt_1 \;+\; \sum_{j=2}^{K} \mu_j^* \, \gt_j,
\end{equation}
with primal feasibility $\gt_j^\T \dts \geq \tau \norm{\gt_j}^2$ and
complementary slackness
$\mu_j^*\bigl(\gt_j^\T \dts - \tau \norm{\gt_j}^2\bigr) = 0$ for all
$j \in [K]_{-1}$.
\end{theorem}
\end{mybluebox}

\begin{proof}[Proof sketch]
The projection statement is immediate: Eq.~\eqref{eq:pcd-qp} minimizes squared Euclidean distance from $\gt_1$ over the closed convex set $F_\tau(\vth)$, which is by definition the projection. The KKT form follows from the Lagrangian
\[
  \mathcal{L}(\dt, \boldsymbol{\mu}) \;=\; \tfrac{1}{2}\norm{\dt - \gt_1}^2 \;-\; \sum_{j=2}^{K} \mu_j \bigl( \gt_j^\T \dt - \tau \norm{\gt_j}^2 \bigr), \qquad \mu_j \geq 0,
\]
whose stationarity in $\dt$ gives $\dt - \gt_1 - \sum_j \mu_j \gt_j = \vzero$, hence Eq.~\eqref{eq:kkt-stationarity}. Primal and dual feasibility and
complementary slackness are the standard KKT conditions for
inequality-constrained convex QPs with affine constraints, and no further constraint qualification is required. The active-set determination of $\{\mu_j^*\}$ is detailed in App.~\ref{app:active-set}.
\end{proof}

\paragraph{The Geometric Picture.} 

The projection identity of Eq.~\eqref{eq:projection} makes precise the ``minimum Euclidean distortion from the primary gradient'' framing carried throughout this work. The primary gradient $\gt_1$ sits somewhere in $\R^n$. The feasible polyhedron $F_\tau(\vth)$ is the intersection of the $K-1$ half-spaces of admissible directions, one per secondary, with boundaries determined by $\tau$. PCD locates $\dts$ as the unique point of $F_\tau(\vth)$ closest to $\gt_1$ in
Euclidean distance. Two regimes are immediate from this picture. If $\gt_1 \in F_\tau(\vth)$ already, no projection is needed and $\dts$ coincides with $\gt_1$ exactly. We record this as:

\begin{corollary}[Inactive regime]\label{cor:inactive}
If $\gt_j^\T \gt_1 \geq \tau \norm{\gt_j}^2$ for every $j \in [K]_{-1}$, then $\dts = \gt_1$, and PCD reduces to gradient descent on $\mathcal{L}_1$.
\end{corollary}

The secondaries do not perturb the primary at all when they do not have to: PCD reverts to standard gradient descent on $\mathcal{L}_1$ precisely when the primary direction itself already provides sufficient progress on every secondary. Otherwise $\gt_1$ lies outside $F_\tau(\vth)$ and the projection lands on the boundary of the polyhedron at the closest boundary point, on the face determined by the secondaries whose constraint is currently binding: the \emph{active set} $A(\vth) \subseteq [K]_{-1}$. As $\tau$ grows, the half-spaces $H_j$ recede from the origin in the direction of each $\gt_j$, the polyhedron $F_\tau(\vth)$ shrinks, and $\dts$ is pushed further from $\gt_1$ along faces parametrized by the active set. This is a monotone deformation whose two-objective realization is shown in Fig.~\ref{fig:pcd-cases}.

\begin{figure}[t]
    \centering
    \begin{subfigure}[b]{0.23\textwidth}
        \centering
        \includegraphics[width=\textwidth]{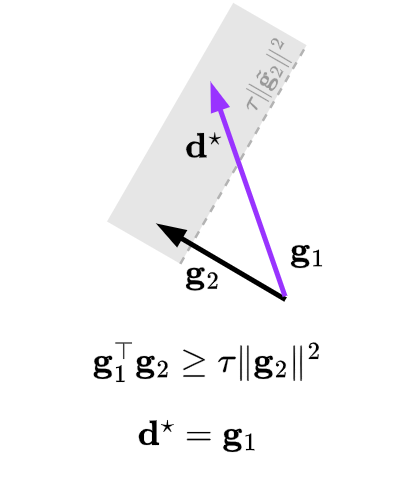}
        \caption{}
        \label{fig:PCD-No-Conflict-1}
    \end{subfigure}
    \hfill
    \begin{subfigure}[b]{0.23\textwidth}
        \centering
        \includegraphics[width=\textwidth]{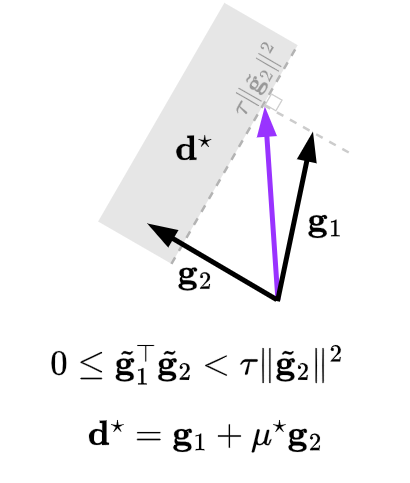}
        \caption{}
        \label{fig:PCD-No-Conflict-2}
    \end{subfigure}
    \hfill
    \begin{subfigure}[b]{0.23\textwidth}
        \centering
        \includegraphics[width=\textwidth]{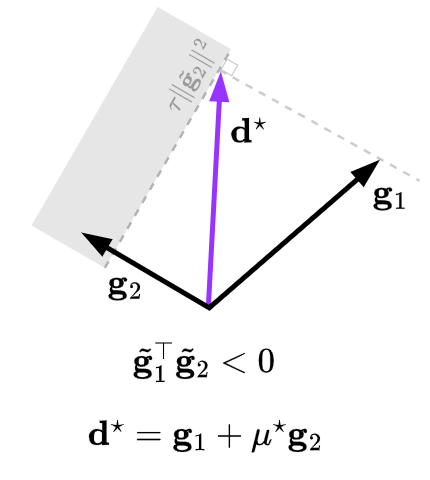}
        \caption{}
        \label{fig:PCD-Conflict}
    \end{subfigure}
    \hfill
    \begin{subfigure}[b]{0.23\textwidth}
        \centering
        \includegraphics[width=\textwidth]{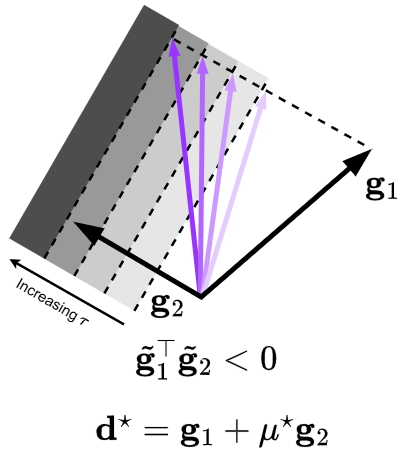}
        \caption{}
        \label{fig:IncreasingTau}
    \end{subfigure}
    \caption{The four canonical regimes of PCD for $K=2$, as predicted by Cor.~\ref{cor:k2}. \textbf{(a) Non-conflicting, sufficient progress}: the constraint is inactive and $\mathbf{d}^\star = \vg_1$. \textbf{(b) Non-conflicting, insufficient progress}: the constraint is active and $\vd^\star$ projects onto its boundary, with $\vd^\star = \vg_1 + \mu^* \vg_2$. \textbf{(c) Conflicting}: the constraint is active and the same projection formula applies, with $\mu^*$ strictly larger than in \textbf{(b)} due to the additional negative term. \textbf{(d)} The effect of increasing $\tau$ at fixed gradient configuration: the feasible set shrinks (Sec.~\ref{sec:pcd-k2}, $\tau$-monotonicity), and $\vd^\star$ moves further from $\vg_1$.}
    \label{fig:pcd-cases}
\end{figure}

This characterization makes the asymmetric structure of PCD's update explicit. Symmetric MOO methods produce updates of the
form $\sum_i \alpha_i \vg_i$, in which the primary and secondary coefficients are determined by the same rule: convex-hull weights for MGDA, fixed weights for weighted sum, projection coefficients for PCGrad and CAGrad. In PCD, the primary's coefficient is structurally pinned to one and the secondaries' coefficients are non-negative, possibly zero, and determined by which constraints currently bind. The two roles (anchor and correction) are visible in the construction itself, rather than recovered from it. The full contrast against symmetric methods is detailed in Sec.~\ref{sec:comparison}.

\subsection{Secondary-Progress Guarantee}
\label{sec:pcd-progress}

The projection characterization yields the central per-step guarantee directly. At each iterate, PCD produces an update direction $\vd^*(\vth)$ and applies
\begin{equation}\label{eq:pcd-update}
  \vth^{+} \;=\; \vth \;-\; \eta\, \vd^*(\vth).
\end{equation}
Under this sign convention, when $\mathcal{L}_j$ is differentiable at $\vth$ the first-order change along the update is $\mathcal{L}_j(\vth^{+}) = \mathcal{L}_j(\vth) - \eta\, \vg_j^\T \vd^*(\vth) + O(\eta^2)$, so $\vg_j^\T \vd^* > 0$ corresponds to first-order \emph{descent} of $\mathcal{L}_j$. When a secondary is nonsmooth at $\vth$, this expansion is \emph{not} valid for an arbitrary subgradient: the exact first-order change is governed by the worst-case element of the subdifferential, $\min_{\vz \in \partial\mathcal{L}_j(\vth)} \vz^\T \vd^*$, and Proposition~\ref{prop:progress}(iii) quantifies the gap between the two. We treat $\vd$ throughout as a \emph{gradient-like} direction whose role is to make $\vg_j^\T \vd$ positive. This is the opposite of the descent-direction convention used classically, and we make it explicit because the following inequality depends on it.

\begin{mybluebox}
\begin{proposition}[Per-step secondary progress]\label{prop:progress}
Suppose $F_\tau(\vth) \neq \emptyset$ and $\tau > 0$, and let $\vg_j$ denote the oracle subgradients from which $\dts$ is computed.
\begin{enumerate}[label=(\roman*),nosep,leftmargin=*]
\item \emph{Directional progress.} The PCD direction satisfies
\begin{equation}\label{eq:progress}
  \gt_j^\T \dts \;\geq\; \tau \norm{\gt_j}^2 \quad\text{for all } j \in [K]_{-1}.
\end{equation}
\item \emph{Smooth case.} If $\mathcal{L}_j$ is differentiable on a neighborhood of $\vth$ with locally Lipschitz gradient, then under the idealized iteration $\vth^{+} = \vth - \eta \dts$,
\begin{equation}\label{eq:taylor-progress}
  \mathcal{L}_j(\vth^{+}) \;\leq\; \mathcal{L}_j(\vth) \;-\; \eta\, s_j^{-1} \tau \norm{\gt_j}^2 \;+\; O(\eta^2),
\end{equation}
so for $\gt_j \neq \vzero$ and sufficiently small $\eta$, $\mathcal{L}_j$ strictly decreases at this step.
\item \emph{General convex case.} If $\vth \in \operatorname{int}\operatorname{dom}(\mathcal{L}_j)$, define the \emph{kink tax} (our term for the standard gap between the oracle's linearization and the exact one-sided directional derivative \citep[Thm.~23.4]{rockafellar1970convex})
\begin{equation}\label{eq:kink-tax}
  \Delta_j(\vd) \;:=\; \vg_j^\T\vd \;-\; \min_{\vz \in \partial\mathcal{L}_j(\vth)} \vz^\T\vd \;\;\geq\; 0 .
\end{equation}
Then
\begin{equation}\label{eq:kink-progress}
  \mathcal{L}_j(\vth^{+}) \;\leq\; \mathcal{L}_j(\vth) \;-\; \eta\bigl( s_j^{-1}\tau\norm{\gt_j}^2 \;-\; \Delta_j(\dts) \bigr) \;+\; o(\eta),
\end{equation}
with strict first-order decrease whenever $\Delta_j(\dts) < s_j^{-1}\tau\norm{\gt_j}^2$.
\end{enumerate}
\end{proposition}
\end{mybluebox}

\begin{proof}
(i) is the primal feasibility condition of Eq.~\eqref{eq:pcd-qp}. (ii) The descent lemma for locally Lipschitz gradients gives $\mathcal{L}_j(\vth^{+}) = \mathcal{L}_j(\vth) - \eta\, \vg_j^\T \dts + O(\eta^2)$, and $\vg_j^\T \dts = s_j^{-1} \gt_j^\T \dts \geq s_j^{-1} \tau \norm{\gt_j}^2$ by (i). (iii) For a proper convex $\mathcal{L}_j$ finite on a neighborhood of $\vth$, the one-sided directional derivative satisfies $\mathcal{L}_j(\vth - \eta\vd) = \mathcal{L}_j(\vth) - \eta \min_{\vz\in\partial\mathcal{L}_j(\vth)} \vz^\T\vd + o(\eta)$ \citep[]{rockafellar1970convex}. Substituting $\min_{\vz}\vz^\T\dts = \vg_j^\T\dts - \Delta_j(\dts)$ and bounding $\vg_j^\T\dts$ via (i) gives Eq.~\eqref{eq:kink-progress}.
\end{proof}

Proposition~\ref{prop:progress} is the practical content of PCD: in the smooth case, every step guarantees \emph{tolerance-controlled} first-order descent on every non-stationary secondary, and in the general convex case the same guarantee holds up to the kink tax, which vanishes unless the update moves mass into currently-zeroed coordinates, groups, or singular directions (closed forms in App.~\ref{app:kink}). The gap is real rather than an artifact of the analysis. App.~\ref{app:kink} exhibits a convex secondary and a standard automatic-differentiation oracle for which the smooth-case bound fails at an active constraint, which is why we state the two cases separately. This is structurally unlike PCGrad (which carries no per-objective progress guarantee in either the conflicting or non-conflicting regime), unlike MGDA (which guarantees descent on the minimum-norm convex hull element, an aggregate quantity), and unlike weighted sum (which guarantees descent only on the scalarized loss). All three are detailed in Sec.~\ref{sec:comparison}.

\subsection{Gradient Normalization}
\label{sec:pcd-ema}

Gradient magnitudes from different objectives may differ by orders of magnitude and vary throughout training. To ensure that the tolerance parameter $\tau$ has a consistent, asymptotically scale-invariant interpretation (Thm.~\ref{thm:scale-inv}), PCD normalizes all gradients via exponential moving averages (EMA).

\begin{definition}[EMA normalization]\label{def:ema}
Fix $\beta \in (0,1)$ and $\varepsilon > 0$. For each $i \in [K]$ and iteration $t$:
\begin{equation}\label{eq:ema-system}
v_{i,t} = \beta\, v_{i,t-1} + (1-\beta) \norm{\vg_{i,t}}^2, \quad
\hat{v}_{i,t} = \tfrac{v_{i,t}}{1-\beta^t}, \quad
s_{i,t} = \tfrac{1}{\sqrt{\hat{v}_{i,t} + \varepsilon}}, \quad
\gt_{i,t} = s_{i,t}\, \vg_{i,t} ,
\end{equation}
with $v_{i,0}=0$.
\end{definition}

In steady state ($t \gg \dfrac{1}{1-\beta}$), $\hat{v}_{i,t} \approx \mathbb{E}[\norm{\vg_i}^2]$, so $\norm{\gt_i} \approx 1$. Hence a constraint of the form $\gt_j^\T \dt \geq \tau \norm{\gt_j}^2$ demands a $\tau$-fraction of the maximum normalized secondary progress, with $\tau$ directly interpretable on $[0,1]$ regardless of the raw magnitudes of $\mathcal{L}_2, \ldots, \mathcal{L}_K$. The mechanism is not new to single-objective optimization: Eq.~\eqref{eq:ema-system} is the bias-corrected second-moment estimate of Adam and RMSProp \citep{kingma2014adam, tieleman2012rmsprop}, applied at the granularity of one scalar per \emph{objective} rather than one per coordinate, so PCD composes with, rather than replaces, the optimizer's own per-coordinate adaptation.

\paragraph{Scale invariance.}
Under per-objective rescaling $L_i \to c_i L_i$ ($c_i > 0$), the EMA buffers transform as $\hat v_{i,t} \to c_i^2 \hat v_{i,t}$, so $s_{i,t} \to s_{i,t}/c_i$ in the limit $\varepsilon \to 0$ (or once $\hat v_{i,t} \gg \varepsilon/c_i^2$), and the normalized gradients $\gt_i = s_{i,t}\vg_i$ are unchanged. Consequently, the QP in Eq.~\eqref{eq:pcd-qp} and its solution $\dts$ are invariant in the same limit. By contrast, $\vd_{\mathrm{WS}} = \sum_i w_i \vg_i$ transforms to $\sum_i w_i c_i \vg_i$, so its optimal weight ratio depends on the relative magnitudes of the raw objectives. Scale invariance is a structural advantage of PCD whenever the raw $L_i$ differ by orders of magnitude (typical when a task loss is paired with a structural regularizer). The exact finite-$\varepsilon$ residual and an $\varepsilon$-free normalization for which invariance is exact are in App.~\ref{app:scale}.

\subsection{Closed Form Solutions and the Role of \texorpdfstring{$\tau$}{tau}}
\label{sec:pcd-k2}

For $K=2$, Eq.~\eqref{eq:pcd-qp} is a projection onto a single half-space, and Theorem~\ref{thm:kkt} reduces to a closed-form update that exposes the geometric meaning of $\tau$.

\begin{corollary}[$K=2$ closed form]\label{cor:k2}
For $K=2$ with $\gt_2 \neq \vzero$,
\begin{equation}\label{eq:k2-cf}
  \dts \;=\;
  \begin{cases}
    \gt_1 & \text{if } \;\gt_2^\T \gt_1 \;\geq\; \tau \norm{\gt_2}^2 \quad\text{(constraint inactive),}\\[4pt]
    \gt_1 \;+\; \Bigl(\tau \;-\; \dfrac{\gt_2^\T \gt_1}{\norm{\gt_2}^2}\Bigr) \gt_2 & \text{otherwise (constraint active).}
  \end{cases}
\end{equation}
\end{corollary}

In the active regime, PCD adds \emph{exactly} the component along $\gt_2$ needed to satisfy the constraint at equality, and no more. Geometrically, $\dts$ lands on the boundary $\{\dt : \gt_2^\T \dt = \tau \norm{\gt_2}^2\}$ at the point closest to $\gt_1$. The four canonical regimes of Eq.~\eqref{eq:k2-cf} as $\gt_1^\T \gt_2$ and $\tau$ vary are illustrated in Fig.~\ref{fig:pcd-cases}.

For $K=3$, the active-regime multipliers also admit a closed form via Cramer's rule (App.~\ref{app:k3-cf}). For general $K$, the active set is determined by exhaustive KKT enumeration over working sets (App.~\ref{app:active-set}; at most $2^{K-1}$ solves of size $\leq K-1$, or the dual active-set method of \citet{goldfarb1983numerically} for large $K$), and $\dts$ remains the orthogonal projection of $\gt_1$ onto $F_\tau(\vth)$.

\paragraph{Primary preservation.}
The two-objective formula admits a clean analytical characterization of when PCD continues to make first-order descent on the primary, and when it does not. Assume $\norm{\gt_1} = \norm{\gt_2} = 1$ (the EMA steady state) and write $c = \gt_1^\T \gt_2$. In the active regime, Eq.~\eqref{eq:k2-cf} gives $\dts = \gt_1 + (\tau - c)\gt_2$.

\begin{proposition}[Primary preservation, $K=2$]\label{prop:primary-preservation}
Under the active regime of Cor.~\ref{cor:k2} with $\norm{\gt_1} = \norm{\gt_2} = 1$ and $c = \gt_1^\T \gt_2$,
\begin{equation}\label{eq:primary-preservation}
  \gt_1^\T \dts \;=\; 1 \;+\; \tau\, c \;-\; c^2.
\end{equation}
Consequently the PCD update preserves first-order primary descent ($\gt_1^\T \dts > 0$) iff $1 + \tau c - c^2 > 0$. In the non-conflicting case $c \geq 0$ this holds for every $\tau \in [0,1]$. In the conflicting case $c < 0$ it holds iff
\begin{equation}\label{eq:primary-preservation-conflict}
  \tau \;<\; \frac{1 - c^2}{-c}.
\end{equation}
\end{proposition}

\begin{proof}
Substituting $\dts = \gt_1 + (\tau - c)\gt_2$ into $\gt_1^\T \dts$, using $\norm{\gt_1}^2 = 1$ and $\gt_1^\T \gt_2 = c$, gives $\gt_1^\T \dts = 1 + (\tau - c) c = 1 + \tau c - c^2$. The conditions follow from $1 + \tau c - c^2 > 0$.
\end{proof}

Proposition~\ref{prop:primary-preservation} is the analytical fingerprint of PCD's behavior. Small $\tau$ leaves $\gt_1^\T \dts$ close to $1$, keeping the update nearly aligned with the primary gradient. Larger $\tau$ deflects $\dts$ further from $\gt_1$, and in sufficiently severe conflict ($c$ close to $-1$, with $\tau$ exceeding the threshold in Eq.~\eqref{eq:primary-preservation-conflict}), $\gt_1^\T \dts$ can fall to zero or below. Geometrically, $\tau$ acts as a \emph{secondary-pressure parameter}: it controls how much of the secondary direction is forced into the update, and accordingly how much of the primary direction is sacrificed in severe conflict. The threshold Eq.~\eqref{eq:primary-preservation-conflict} predicts the breakdown observed empirically as $\tau$ is increased past a problem-dependent operating boundary.

\subsection{Algorithm}
\label{sec:pcd-algorithm}

Algorithm~\ref{alg:pcd} summarizes a single PCD step. The QP of Eq.~\eqref{eq:pcd-qp} admits a closed-form solution for $K = 2$ (Cor.~\ref{cor:k2}), and is solved by KKT enumeration over working sets for $K \geq 3$ (App.~\ref{app:active-set}); Algorithm~\ref{alg:pcd} also states the three boundary branches ($\vg_1=\vzero$, $\dts=\vzero$, $F_\tau(\vth)=\emptyset$) explicitly. Per-step cost is dominated by the $K$ backpropagations needed to obtain $\vg_1, \ldots, \vg_K$. The QP itself is $K \times K$ and adds negligible overhead.

\begin{algorithm}[t]
\SetAlgoLined
\DontPrintSemicolon
\KwIn{Initial $\vth_0$; scalar tolerance $\tau \in [0,1]$; EMA decay $\beta$, stabilizer $\varepsilon$; step size $\eta$; optimizer $\mathcal{O}$.}
Initialize $v_{i,0}\leftarrow 0$ for $i\in[K]$\;
\For{$t=1,2,\ldots,T$}{
  Compute $\vg_1 \leftarrow \nabla L_1(\vth)$, $\vg_j \leftarrow \mathrm{subgrad}(L_j,\vth)$ for $j\geq 2$\;
  \For{$i=1,\ldots,K$}{
    Update $v_{i,t}, \hat v_{i,t}, s_{i,t}$ via Eq.~\eqref{eq:ema-system}; set $\gt_i \leftarrow s_{i,t}\vg_i$\;
  }
  \uIf{$\norm{\vg_1}=0$}{
    $\vd^* \leftarrow \vzero$ \tcp*[r]{primary-stationary halt, Sec.~\ref{sec:pcd-scope}}
  }
  \uElseIf{$\gt_j^\T\gt_1 \geq \tau\norm{\gt_j}^2$ for all $j\in[K]_{-1}$}{
    $\dts \leftarrow \gt_1$ \tcp*[r]{constraints inactive, Cor.~\ref{cor:inactive}}
  }
  \Else{
    Solve Eq.~\eqref{eq:pcd-qp} via Cor.~\ref{cor:k2} ($K{=}2$) or KKT enumeration (App.~\ref{app:active-set}) $\to \dts$\;
    \If{no working set is KKT-valid, i.e.\ $F_\tau(\vth)=\emptyset$ (Thm.~\ref{thm:feas})}{
      $\dts \leftarrow \gt_1$ \tcp*[r]{fallback: drop the secondary constraints this step}
    }
  }
  \uIf{$\norm{\dts}=0$}{
    $\vd^* \leftarrow \vzero$ \tcp*[r]{fixed point: $\vg_1=\vzero$ when $\tau>0$ (Thm.~\ref{thm:cms}); Lemma~\ref{lem:tau-zero} when $\tau=0$}
  }
  \Else{
    $\vd^* \leftarrow \dts \cdot \norm{\vg_1}/\norm{\dts}$ \tcp*[r]{magnitude rescaling, Sec.~\ref{sec:pcd-scope}}
  }
  $\vth \leftarrow \mathcal{O}.\mathrm{step}(\vth,\vd^*,\eta)$\;
}
\Return $\vth_T$\;
\caption{Priority-Constrained Descent (PCD), scalar $\tau$, general $K$.}
\label{alg:pcd}
\end{algorithm}

\subsection{Scope of Guarantees}
\label{sec:pcd-scope}

Our theoretical claims attach to the \emph{normalized PCD direction} $\dts$ and to its idealized first-order step $\vth^{+} = \vth - \eta \dts$. They are \emph{local} and \emph{directional}: Theorem~\ref{thm:kkt} characterizes what direction the QP computes, and Proposition~\ref{prop:progress} characterizes the first-order progress this direction induces on each secondary.

The practical update in Algorithm~\ref{alg:pcd} also rescales $\dts$ to match the raw primary-gradient magnitude, $\vd^* = \dts \cdot \norm{\vg_1}/\norm{\dts}$, before passing it to the optimizer $\mathcal{O}$. Positive scalar rescaling preserves direction, hence preserves the signs of $\gt_j^\T \vd^*$ and the directional descent properties of Prop.~\ref{prop:progress}. The magnitude rescaling exists for compatibility with standard optimizer step-size scheduling and is treated throughout as an implementation choice. We discuss its interaction with the convergence picture in App.~\ref{app:sufficient-decrease}.

This separation between $\dts$ (the analyzed object) and $\vd^*$ followed by an optimizer step (the deployed object) is the same separation used in PCGrad, CAGrad, MGDA, FAMO, and other gradient-manipulation methods. As in those methods, our guarantees characterize the geometric direction-construction step. The empirical evaluation in Sec.~\ref{sec:experiments} reports PCD as deployed. 

We prove no convergence theorem for Algorithm~\ref{alg:pcd}: neither convergence of its iterates to CMS points, nor rates. By contrast, \citet{gong2021bi} prove convergence guarantees for the raw-gradient dynamic-barrier update in constrained and lexicographic settings for 2 objectives. 

Three consequences are explicitly stated below. 

\paragraph{The deployed fixed point.} Since $\vd^* = \dts \cdot \norm{\vg_1}/\norm{\dts}$, the deployed update is $\vzero$ whenever $\vg_1 = \vzero$, even when $\dts \neq \vzero$. The zero does not come from the QP: at $\gt_1 = \vzero$ the program returns the least-norm feasible point, which Lemma~\ref{lem:boundary}(ii) shows is nonzero exactly when some secondary remains non-stationary. It comes from the rescale, whose factor $\norm{\vg_1}$ multiplies that direction by zero. Algorithm~\ref{alg:pcd} therefore has primary-stationary fixed points that are not CMS, and Theorem~\ref{thm:cms} characterizes the fixed points of $\dts$, not of the deployed iteration (see also App.~\ref{app:sufficient-decrease}). We note, however, that reaching these points is a measure-zero event, as it requires initializations where $\vg_1 = \vzero$. Two repairs exist. Flooring the rescale at $\max(\norm{\vg_1}, \gamma)$ makes the deployed fixed points coincide with those of $\dts$, but it breaks scale invariance, introduces a hyperparameter, and binds intermittently: a time-varying rescale to which adaptive optimizers are not invariant. Rescaling by $1/s_1$ instead preserves scale equivariance, but the EMA stabilizer $\varepsilon$ then acts as a floor in disguise, and the identity $\norm{\vd^*} = \norm{\vg_1}$ invoked in App.~\ref{app:sufficient-decrease} is forfeited. We therefore state the limitation rather than alter the update. The EMA normalization neither causes nor masks this branch: $s_{1,t} > 0$ at every step, so $\gt_1 = \vzero$ exactly when the instantaneous $\vg_1 = \vzero$, and the rescale factor is the raw $\norm{\vg_1}$ in any case. Empirically, the branch is never approached: over all $330$ runs of the tuned-baseline study (Sec.~\ref{sec:tuned-baselines}, $300$ epochs each), $\min_t \norm{\vg_{1,t}} = 4.7\times10^{-4}$, against typical values $\approx 2$.

\paragraph{Direction versus realized step.} Positive rescaling preserves the signs $\gt_j^\T \vd^*$, but the optimizer $\mathcal{O}$ does not need to follow strict positive rescaling. A diagonal preconditioner and momentum can produce a realized step $\Delta\vth$ with $\operatorname{sign}\langle \gt_j, -\Delta\vth\rangle \neq \operatorname{sign}\langle\gt_j, \vd^*\rangle$. Our guarantees are statements about $\vd^*$, not about $\Delta\vth$. We measure this gap directly in App.~\ref{app:optimizer}: under the deployed Adam optimizer the per-step sign of the secondary constraint survives on $92.3$--$99.8\%$ of steps and the realized alignment with the secondary is at least five times stronger than the guaranteed margin, while the SGD-with-momentum control flips the sign on $47$--$73\%$ of steps, landing on the same accuracy--sparsity frontier. The mechanism is a trajectory-level bias rather than a per-step invariant. This gap is shared by every gradient-manipulation method that hands a constructed direction to an adaptive optimizer, including PCGrad, CAGrad, MGDA, and FAMO.

\paragraph{Pointwise characterization.} Theorem~\ref{thm:cms} is a pointwise endpoint characterization and carries no lower bound on $\norm{\dts}$ near non-CMS points: with $\gt_1$ anti-parallel to a single secondary gradient of norm $\delta$, the QP returns $\norm{\dts} = \tau\delta \to 0$ as $\delta \to 0$ while $\norm{\gt_1}$ stays fixed. Vanishingly small directions therefore occur arbitrarily close to configurations that are not CMS, and Theorem~\ref{thm:cms} should not be read as a stability result. This caveat does not apply specifically to PCD, but rather to the entire family of constrained-direction methods.

\subsection{Relation to Pareto Stationarity}
\label{sec:pcd-cms}

PCD's fixed-point structure provides a sharp endpoint characterization that distinguishes it from symmetric MOO methods. The key observation is as follows: for symmetric methods, a Pareto-stationary point can occur through \emph{gradient cancellation}, $\vzero \in \conv\{\vg_1, \ldots, \vg_K\}$, even when no individual $\vg_i$ vanishes. PCD's QP rules this out for the normalized direction $\dts$.

\begin{mybluebox}
\begin{theorem}[Fixed-point characterization]\label{thm:cms}
Let $\tau > 0$ and suppose $F_\tau(\vth) \neq \emptyset$. Then
\begin{equation}\label{eq:fp-equiv}
  \dts(\vth) \;=\; \vzero \quad\Longleftrightarrow\quad \gt_i(\vth) \;=\; \vzero \;\;\text{for all } i \in [K].
\end{equation}
\end{theorem}
\end{mybluebox}

\begin{proof}[Proof sketch]
($\Leftarrow$) If every $\gt_i = \vzero$ then $\dt = \vzero$ is feasible and attains the objective minimum $\norm{\vzero - \vzero}^2 = 0$.
($\Rightarrow$) Primal feasibility with $\dts = \vzero$ gives $0 \geq \tau \norm{\gt_j}^2$, and since $\tau > 0$ this forces $\gt_j = \vzero$ for every $j \geq 2$. KKT stationarity then forces $\gt_1 = \vzero$ as well.
The full proof, including the precise dependence on Theorem~\ref{thm:kkt}, is in App.~\ref{app:cms-proof}.
\end{proof}

The right-hand side of Eq.~\eqref{eq:fp-equiv} is a statement about the oracle outputs $\gt_i$. Since every returned $\vg_j$ lies in $\partial\mathcal{L}_j(\vth)$, it always \emph{implies} CMS (Def.~\ref{def:cms}). On the other hand, the converse identification that every CMS point satisfies the right-hand side requires returning $\vzero$ whenever $\vzero \in \partial\mathcal{L}_j(\vth)$. This can be violated in the case of standard automatic differentiation (Rem.~\ref{rem:cms-oracle}). CMS is by definition strictly stronger than Pareto stationarity, ruling out the cancellation equilibria at which symmetric methods can halt. Theorem~\ref{thm:cms} is best read as an \emph{endpoint characterization} of the deterministic, normalized QP: it identifies the only points at which the idealized PCD direction can vanish, but is not a stochastic convergence claim. Whether the iterates of Algorithm~\ref{alg:pcd} approach such an endpoint depends on optimizer, step size, batch noise, and feasibility along the path, none of which are subsumed by the QP characterization. A complementary endpoint result at $\tau = 0$ recovers Pareto stationarity rather than CMS (App.~\ref{app:tau-zero}).

\paragraph{Operating curves in $\tau$.} As $\tau$ varies in $[0,1]$ the feasible set $F_\tau(\vth)$ varies monotonically. On asymmetric problems where joint stationarity is structurally possible (for instance, regularizer-style secondaries that vanish at minima of the primary in the over-parameterized regime), this monotone family yields a continuous operating curve interpolating between the Pareto-stationary endpoint ($\tau = 0$) and CMS endpoints ($\tau > 0$). On genuine trade-off problems with no common stationary point, no CMS point exists and the PCD direction with $\tau > 0$ has no fixed point. Iterates need not converge in this regime, and the framework is best understood as targeting the regularizer-style hierarchy of Eq.~\eqref{eq:objectives}.

\subsection{Comparison with Existing Methods}
\label{sec:comparison}

\begin{figure}[h]
    \centering
    \includegraphics[width=1\linewidth]{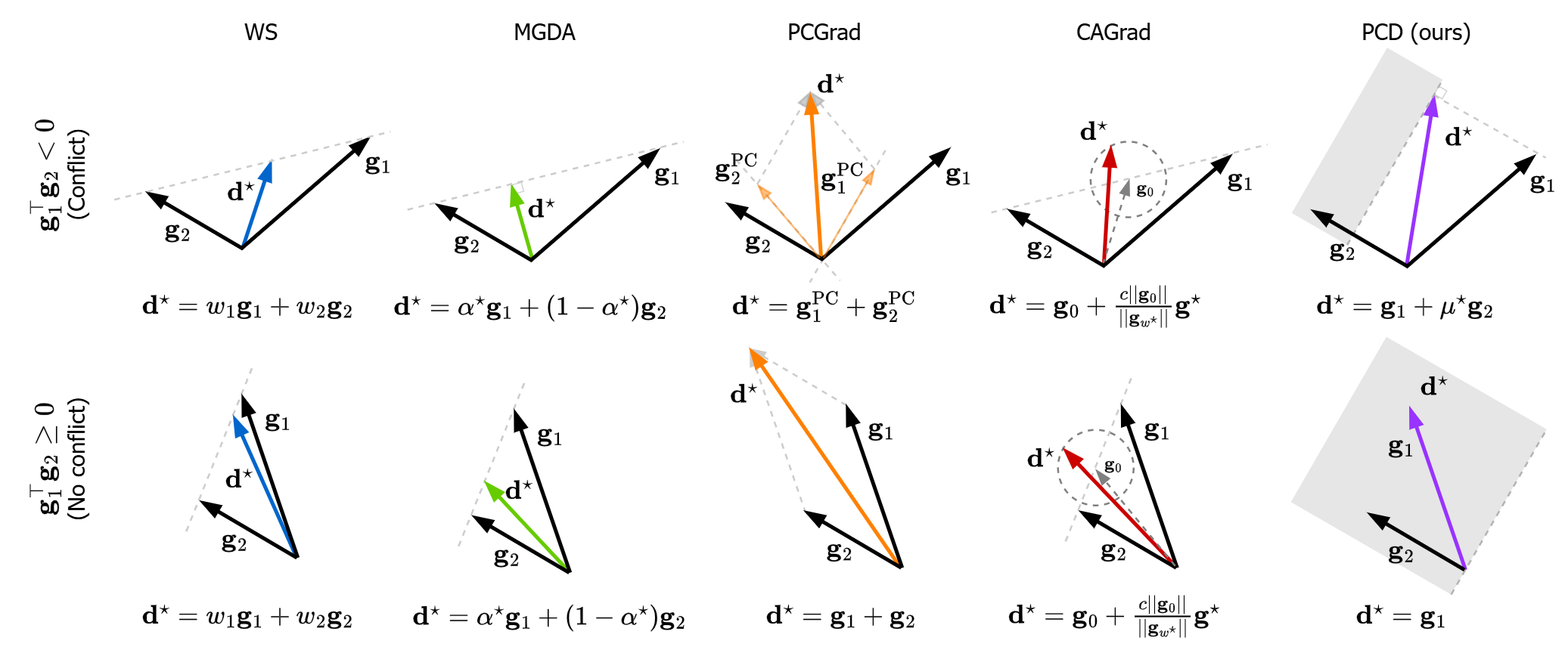}
    \caption{Comparison of PCD against established gradient-based multi-objective methods. \textbf{Top Row}: Conflicting case, $\vg_1^\T \vg_2 < 0$. \textbf{Bottom Row}: Non-conflicting case, $\vg_1^\T \vg_2 \geq 0$. For PCD in the non-conflicting case, $\tau = 0$ is used to align conflict definitions.}
    \label{fig:comparisons}
\end{figure}

Fig.~\ref{fig:comparisons} contrasts PCD against established gradient-based MOO methods in the standard conflicting ($\vg_1^\T \vg_2 < 0$) and non-conflicting ($\vg_1^\T \vg_2 \geq 0$) regimes. The contrast is structural: existing methods combine objective gradients symmetrically, producing updates in which $\vg_1$ and $\vg_2$ are interchangeable. Weighted sum~\citep{marler2010weighted} forms $w_1\vg_1 + w_2\vg_2$ at fixed weights, MGDA~\citep{desideri2009multiple} returns the minimum-norm convex combination $\alpha^*\vg_1 + (1-\alpha^*)\vg_2$, PCGrad~\citep{pcgrad} projects out conflicting components pairwise and averages, reducing to $\vg_1 + \vg_2$ when the gradients do not conflict, and CAGrad~\citep{cagrad} anchors at the average $\vg_0 = \tfrac{1}{2}(\vg_1 + \vg_2)$ and improves the worst case within a trust region. For MGDA, PCGrad, and CAGrad the coefficient rule is itself permutation-invariant, so swapping the labels of $\vg_1$ and $\vg_2$ leaves the update unchanged; for the weighted sum the weights are supplied by the user rather than derived from the objectives' roles, so the construction likewise encodes no hierarchy.

PCD breaks this symmetry by construction. The KKT identity of Eq.~\eqref{eq:kkt-stationarity} pins to the primary gradient and admits the secondaries only through non-negative multipliers that vanish on inactive constraints. The QP objective $\tfrac{1}{2}\norm{\dt - \gt_1}^2$ and the constraints involving only $\gt_j$ for $j \geq 2$ encode the asymmetry directly, not through preference weights. Two consequences follow. First, PCD exhibits \emph{three} update regimes rather than two: constraint inactive ($\dts = \gt_1$), constraint active without conflict, and constraint active with conflict. This is determined by where $\gt_2^\T \gt_1$ sits relative to $\tau \norm{\gt_2}^2$ (Fig.~\ref{fig:pcd-cases}, Cor.~\ref{cor:k2}). Second, the MGDA failure at $\vg_1 = \vzero$, where $\vzero \in \conv\{\vzero, \vg_2, \ldots, \vg_K\}$ collapses the combination to $\vzero$ regardless of the remaining $K-1$ secondaries (App.~\ref{app:per-step}, Prop.~\ref{prop:mgda}), does not occur for the PCD direction with $\tau > 0$: the QP continues to drive every non-stationary secondary toward stationarity (Lemma~\ref{lem:boundary}). The deployed magnitude rescale, however, returns $\vd^* = \vzero$ exactly at $\vg_1 = \vzero$, so this advantage attaches to $\dts$ rather than to Algorithm~\ref{alg:pcd} as deployed (Sec.~\ref{sec:pcd-scope}).

\section{Experimental Results}
\label{sec:experiments}
In this section, we conduct experiments over two application domains and one set of synthetic experiments to demonstrate the superiority of PCD in asymmetric applications. In particular, we conduct experiments involving structured pruning and compression of neural networks (Sec.~\ref{sec:structured-pruning}) with $K = 2$ objectives, unstructured sparse and low-rank networks (Sec.~\ref{sec:k3-experiments}) with $K = 3$ objectives, and a set of synthetic experiments to validate PCD's advantages (Sec.~\ref{sec:synthetic}).

\subsection{Structured Pruning}
\label{sec:structured-pruning}

Structured neural-network compression is a natural instance of asymmetric multi-objective optimization. We set the primary objective $\mathcal{L}_1$ to be the standard cross-entropy loss and the secondary objective $\mathcal{L}_2$ to be the Group Lasso penalty:
\begin{equation}
    \mathcal{L}_2(\theta) \;=\; \sum_{g} \|\theta_g\|_2 ,
\end{equation}
where the sum runs over groups $g$ defined at the output-channel level of each layer. The formulation is asymmetric by construction: task accuracy is the deliverable, while structural sparsity is a secondary objective whose pressure is controlled by $\tau$. After training, groups whose norm $\|\theta_g\|_2$ falls below a threshold are removed and the resulting network is evaluated directly without fine-tuning. To produce results at a specified parameter-reduction target $r \in \{80\%, 85\%, 90\%, 95\%, 99\%\}$, the threshold is selected post hoc, per (method, architecture, dataset, target), as the smallest value for which the fraction of removed parameters reaches $r$.

\paragraph{Why Group Lasso?} Group Lasso \citep{yuan2006group} penalizes defined groups of parameters so that entire groups are driven to zero, rather than individual weights. With each group set to one output channel, a group at zero is a removable filter: removal yields a smaller \emph{dense} network that runs on standard hardware with no sparse kernels, no special runtime, and no change of format. Channel-level Group Lasso is the standard regularizer for exactly this reason. \citet{wen2016ssl} introduced it as Structured Sparsity Learning, and \citet{liu2017slimming} apply the same principle in Network Slimming. This choice makes our secondary objective comparable to that line of work rather than to unstructured magnitude pruning \citep{han2015learning}. It is also what makes the deployment measurements below meaningful: FLOPs, latency, and file size in Tables~\ref{tab:densenet-cifar10} and~\ref{tab:resnet-cifar100} are measured on CPU inference after removing the zeroed filters, with no fine-tuning and no sparse kernels. An unstructured penalty would not permit that measurement.

\subsubsection{Comparison with General Multi-Objective Methods}
\label{sec:moo-baselines}

\paragraph{Experimental Protocol. }

We evaluate on four architectures spanning a range of structural properties: DenseNet-121~\citep{huang2017densely}, ResNet-34~\citep{resnet}, an Inception-style model~\citep{inception}, and MobileNetV2~\citep{howard2017mobilenets, sandler2018mobilenetv2}, each trained on CIFAR-10 and CIFAR-100~\citep{krizhevsky2009learning}. All models are trained from scratch with Adam~\citep{kingma2014adam}, learning rate $10^{-3}$, batch size 128, cosine annealing, and 300 epochs. Each configuration is run for 5 independent seeds, with the reported numbers being means. PCD is compared against six multi-objective baselines: PCGrad, MGDA, FAMO, CAGrad, weighted sum (WS), and AuxiNash. For baselines with a tunable hyperparameter (the weight $w$ for WS, the conflict coefficient $c$ for CAGrad, and $p$ for AuxiNash), we sweep over $\{0.1, 0.2, \ldots, 0.9\}$ and report the best operating point per configuration. MGDA, PCGrad, and FAMO have no such hyperparameter and are run at their default settings. For PCD we use a denser $\tau$ sweep $\in  \{0.01, 0.02, \dots, 0.09, 0.1, 0.2, \dots, 0.9, 1\}$.

Fig.~\ref{fig:pruning-comparison} reports test accuracy as a function of parameter reduction across all eight configurations. PCD (purple, \textcolor{PCD}{$\filledstar$}) sits above every baseline across all eight panels at every compression target up to $95\%$. The single exception above that level is MobileNetV2/CIFAR-100 at the $99\%$ target, where the weighted sum reaches $1.9\%$ against PCD's $1.2\%$ and CAGrad ties, at accuracies that have collapsed for every method. The gap is widest on the harder benchmarks: on ResNet-34/CIFAR-100, PCD retains $73.0\%$ accuracy at $80$--$85\%$ parameter reduction, while the symmetric baselines collapse: MGDA to $1.2\%$ and FAMO to $2.7\%$, slightly above the $1\%$ chance level, and PCGrad to $7.4\%$ and CAGrad to $7.3\%$, roughly seven times chance, across all compression targets. AuxiNash is the strongest baseline at $44.4\%$, $28.6\%$  worse than PCD. 

The degradation observed in the extreme compression cases for harder benchmarks is expected, and this occurs during training rather than as a consequence of pruning. Baseline methods reach $99.3$--$100\%$ native group sparsity before any threshold is applied, so the pruning step reveals a network that is already compromised. Since PCD is the only method in this comparison that bounds the progress per-step, that bound acts as an inherent limit on native sparsity. The cause is objective scale rather than gradient conflict: the raw regularizer gradients carry roughly one to two orders of magnitude the norm of the task gradient across our configurations, while the cosine between task and regularizer gradients never exceeds $1.2\times10^{-3}$. Here, we see the benefit of PCD's per-objective scaling. Tables~\ref{tab:densenet-cifar10} and~\ref{tab:resnet-cifar100} give the full numerical comparison for these two configurations.

\begin{figure}[h]
    \centering
    \includegraphics[width=\linewidth]{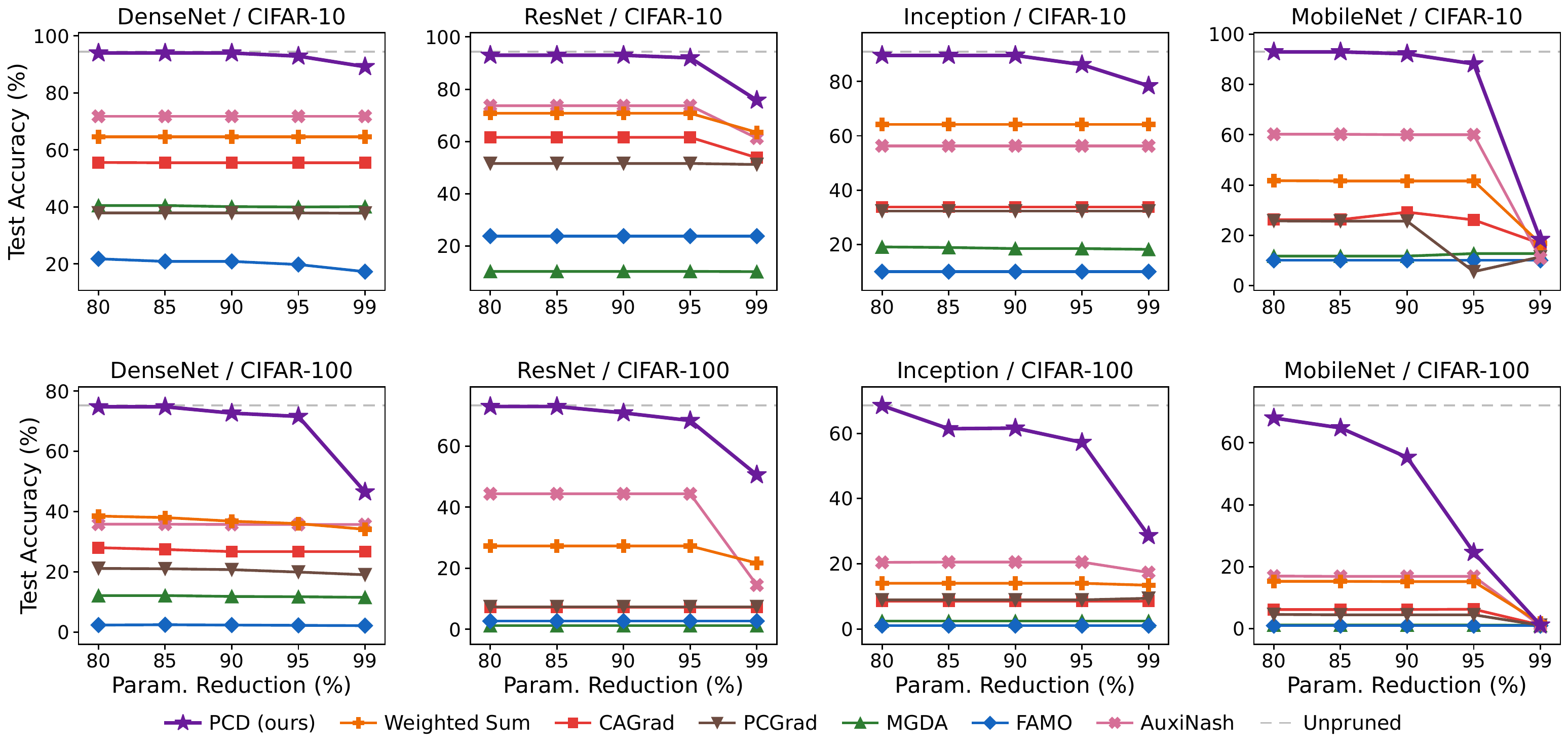}
    \caption{Test accuracy as a function of parameter reduction for DenseNet-121, ResNet-34, Inception, and MobileNet-v2, trained on CIFAR-10 (top row) and CIFAR-100 (bottom row).}
    \label{fig:pruning-comparison}
\end{figure}

\begin{table}[h]
\centering\scriptsize
\setlength{\tabcolsep}{2.6pt}
\caption{Test accuracy (\%) and deployment efficiency for all methods on DenseNet-121\,/\,CIFAR-10 (unpruned: 94.5\%, 1.82\,G FLOPs, 32.4\,ms latency, 27.2\,MB). FLOPs, CPU latency, and file size are measured per method on the physically pruned network at each target, with reductions relative to the unpruned model in parentheses. Accuracy values are means over 5 seeds.}
\begin{tabular}{ll | ccccccc}
  \toprule
  \textbf{Target} & \textbf{Metric} & \textbf{PCD (Ours)} & \textbf{WS} & \textbf{CAGrad} & \textbf{MGDA} & \textbf{PCGrad} & \textbf{FAMO} & \textbf{AuxiNash} \\
  \midrule
  \textbf{Unpruned}
  & Acc.  & \textbf{94.5} & 94.5 & 94.5 & 94.5 & 94.5 & 94.5 & 94.5 \\
  & FLOPs (G) & \cellcolor{gray!15}1.82 & \multicolumn{6}{c}{\cellcolor{gray!15}} \\
  & Lat.\ (ms) & \cellcolor{gray!15}32.4 & \multicolumn{6}{c}{\cellcolor{gray!15}$\uparrow$ same} \\
  & Size (MB) & \cellcolor{gray!15}27.2 & \multicolumn{6}{c}{\cellcolor{gray!15}} \\
  \midrule
  \textbf{80\%}
  & Acc.  & \textbf{94.0} & 64.6 & 55.6 & 40.5 & 37.9 & 21.8 & 71.8 \\
  & FLOPs (G) & 0.46 (4.0$\times$) & 0.22 (8.2$\times$) & 0.19 (9.3$\times$) & 0.23 (8.0$\times$) & 0.21 (8.6$\times$) & 0.22 (8.2$\times$) & 0.21 (8.7$\times$) \\
  & Lat.\ (ms) & 21.6 (1.5$\times$) & 17.3 (1.9$\times$) & 16.0 (2.0$\times$) & 17.3 (1.9$\times$) & 16.9 (1.9$\times$) & 17.1 (1.9$\times$) & 16.4 (2.0$\times$) \\
  & Size (MB) & 5.6 (4.9$\times$) & 5.6 (4.8$\times$) & 5.6 (4.8$\times$) & 5.6 (4.8$\times$) & 5.6 (4.8$\times$) & 5.6 (4.8$\times$) & 5.6 (4.8$\times$) \\
  \midrule
  \textbf{85\%}
  & Acc.  & \textbf{94.0} & 64.6 & 55.5 & 40.5 & 37.9 & 20.9 & 71.8 \\
  & FLOPs (G) & 0.40 (4.5$\times$) & 0.17 (11.0$\times$) & 0.14 (12.7$\times$) & 0.17 (10.9$\times$) & 0.16 (11.6$\times$) & 0.16 (11.4$\times$) & 0.16 (11.6$\times$) \\
  & Lat.\ (ms) & 20.7 (1.6$\times$) & 15.8 (2.1$\times$) & 14.5 (2.2$\times$) & 15.5 (2.1$\times$) & 15.5 (2.1$\times$) & 15.4 (2.1$\times$) & 14.1 (2.3$\times$) \\
  & Size (MB) & 4.3 (6.4$\times$) & 4.3 (6.4$\times$) & 4.3 (6.3$\times$) & 4.3 (6.3$\times$) & 4.3 (6.3$\times$) & 4.3 (6.4$\times$) & 4.3 (6.4$\times$) \\
  \midrule
  \textbf{90\%}
  & Acc.  & \textbf{94.0} & 64.6 & 55.5 & 40.1 & 37.9 & 20.9 & 71.8 \\
  & FLOPs (G) & 0.34 (5.3$\times$) & 0.11 (16.4$\times$) & 0.095 (19.1$\times$) & 0.11 (15.8$\times$) & 0.11 (17.3$\times$) & 0.11 (16.2$\times$) & 0.11 (16.6$\times$) \\
  & Lat.\ (ms) & 17.6 (1.8$\times$) & 13.8 (2.3$\times$) & 12.6 (2.6$\times$) & 13.6 (2.4$\times$) & 13.3 (2.4$\times$) & 13.5 (2.4$\times$) & 12.7 (2.5$\times$) \\
  & Size (MB) & 2.9 (9.4$\times$) & 2.9 (9.3$\times$) & 3.0 (9.2$\times$) & 3.0 (9.2$\times$) & 3.0 (9.2$\times$) & 3.0 (9.1$\times$) & 3.0 (9.2$\times$) \\
  \midrule
  \textbf{95\%}
  & Acc.  & \textbf{92.9} & 64.6 & 55.5 & 40.0 & 37.9 & 19.8 & 71.8 \\
  & FLOPs (G) & 0.14 (12.6$\times$) & 0.064 (28.5$\times$) & 0.055 (33.1$\times$) & 0.063 (28.7$\times$) & 0.057 (31.8$\times$) & 0.063 (28.9$\times$) & 0.066 (27.3$\times$) \\
  & Lat.\ (ms) & 13.9 (2.3$\times$) & 11.5 (2.8$\times$) & 10.3 (3.1$\times$) & 11.0 (2.9$\times$) & 10.7 (3.0$\times$) & 11.1 (2.9$\times$) & 10.5 (3.1$\times$) \\
  & Size (MB) & 1.6 (17.2$\times$) & 1.6 (17.3$\times$) & 1.6 (16.8$\times$) & 1.6 (16.6$\times$) & 1.6 (17.2$\times$) & 1.6 (17.0$\times$) & 1.6 (16.6$\times$) \\
  \midrule
  \textbf{99\%}
  & Acc.  & \textbf{89.2} & 64.6 & 55.5 & 40.1 & 37.8 & 17.3 & 71.8 \\
  & FLOPs (G) & 0.040 (45.4$\times$) & 0.024 (74.2$\times$) & 0.023 (79.4$\times$) & 0.026 (70.4$\times$) & 0.025 (71.4$\times$) & 0.027 (67.0$\times$) & 0.028 (64.1$\times$) \\
  & Lat.\ (ms) & 5.6 (5.8$\times$) & 9.0 (3.6$\times$) & 8.4 (3.9$\times$) & 8.7 (3.7$\times$) & 8.9 (3.6$\times$) & 9.0 (3.6$\times$) & 8.2 (4.0$\times$) \\
  & Size (MB) & 0.33 (83.0$\times$) & 0.48 (56.4$\times$) & 0.55 (49.9$\times$) & 0.52 (52.6$\times$) & 0.52 (52.5$\times$) & 0.53 (51.1$\times$) & 0.53 (51.7$\times$) \\
  \bottomrule
\end{tabular}
\label{tab:densenet-cifar10}
\end{table}

\begin{table}[h]
\centering\scriptsize
\setlength{\tabcolsep}{2.6pt}
\caption{Test accuracy (\%) and deployment efficiency for all methods on ResNet-34\,/\,CIFAR-100 (unpruned: 73.3\%, 2.33\,G FLOPs, 24.5\,ms latency, 81.5\,MB). FLOPs, CPU latency, and file size are measured per method on the physically pruned network at each target, with reductions relative to the unpruned model in parentheses. The $\approx$99\% row reflects the maximum achievable reduction ($\approx$98.5\%) for ResNet-34 under this pruning scheme. Accuracy values are means over 5 seeds.}
\begin{tabular}{ll | ccccccc}
  \toprule
  \textbf{Target} & \textbf{Metric} & \textbf{PCD (Ours)} & \textbf{WS} & \textbf{CAGrad} & \textbf{MGDA} & \textbf{PCGrad} & \textbf{FAMO} & \textbf{AuxiNash} \\
  \midrule
  \textbf{Unpruned}
  & Acc.  & \textbf{73.3} & 73.3 & 73.3 & 73.3 & 73.3 & 73.3 & 73.3 \\
  & FLOPs (G) & \cellcolor{gray!15}2.33 & \multicolumn{6}{c}{\cellcolor{gray!15}} \\
  & Lat.\ (ms) & \cellcolor{gray!15}24.5 & \multicolumn{6}{c}{\cellcolor{gray!15}$\uparrow$ same} \\
  & Size (MB) & \cellcolor{gray!15}81.5 & \multicolumn{6}{c}{\cellcolor{gray!15}} \\
  \midrule
  \textbf{80\%}
  & Acc.  & \textbf{73.0} & 27.3 & 7.3 & 1.2 & 7.4 & 2.7 & 44.4 \\
  & FLOPs (G) & 0.67 (3.5$\times$) & 0.66 (3.5$\times$) & 0.52 (4.5$\times$) & 0.41 (5.6$\times$) & 0.42 (5.5$\times$) & 0.45 (5.2$\times$) & 0.73 (3.2$\times$) \\
  & Lat.\ (ms) & 11.5 (2.1$\times$) & 11.5 (2.1$\times$) & 10.2 (2.4$\times$) & 9.6 (2.6$\times$) & 9.7 (2.5$\times$) & 9.8 (2.5$\times$) & 11.5 (2.1$\times$) \\
  & Size (MB) & 16.3 (5.0$\times$) & 16.3 (5.0$\times$) & 16.4 (5.0$\times$) & 16.3 (5.0$\times$) & 16.4 (5.0$\times$) & 16.4 (5.0$\times$) & 16.3 (5.0$\times$) \\
  \midrule
  \textbf{85\%}
  & Acc.  & \textbf{73.0} & 27.3 & 7.3 & 1.2 & 7.4 & 2.7 & 44.4 \\
  & FLOPs (G) & 0.57 (4.0$\times$) & 0.53 (4.4$\times$) & 0.40 (5.8$\times$) & 0.30 (7.8$\times$) & 0.32 (7.2$\times$) & 0.34 (6.9$\times$) & 0.59 (3.9$\times$) \\
  & Lat.\ (ms) & 10.5 (2.3$\times$) & 10.4 (2.4$\times$) & 9.1 (2.7$\times$) & 8.5 (2.9$\times$) & 8.7 (2.8$\times$) & 8.7 (2.8$\times$) & 10.4 (2.3$\times$) \\
  & Size (MB) & 12.3 (6.6$\times$) & 12.3 (6.6$\times$) & 12.3 (6.6$\times$) & 12.3 (6.6$\times$) & 12.4 (6.6$\times$) & 12.3 (6.6$\times$) & 12.3 (6.6$\times$) \\
  \midrule
  \textbf{90\%}
  & Acc.  & \textbf{70.9} & 27.3 & 7.3 & 1.2 & 7.4 & 2.7 & 44.4 \\
  & FLOPs (G) & 0.37 (6.2$\times$) & 0.38 (6.1$\times$) & 0.27 (8.6$\times$) & 0.19 (12.0$\times$) & 0.22 (10.5$\times$) & 0.23 (10.2$\times$) & 0.42 (5.5$\times$) \\
  & Lat.\ (ms) & 8.7 (2.8$\times$) & 8.9 (2.8$\times$) & 7.8 (3.1$\times$) & 7.4 (3.3$\times$) & 7.7 (3.2$\times$) & 7.6 (3.2$\times$) & 8.9 (2.8$\times$) \\
  & Size (MB) & 8.3 (9.9$\times$) & 8.3 (9.9$\times$) & 8.2 (9.9$\times$) & 8.3 (9.8$\times$) & 8.3 (9.8$\times$) & 8.3 (9.8$\times$) & 8.1 (10.0$\times$) \\
  \midrule
  \textbf{95\%}
  & Acc.  & \textbf{68.4} & 27.3 & 7.3 & 1.2 & 7.4 & 2.7 & 44.4 \\
  & FLOPs (G) & 0.26 (9.1$\times$) & 0.19 (12.2$\times$) & 0.14 (17.1$\times$) & 0.096 (24.3$\times$) & 0.11 (20.3$\times$) & 0.12 (19.9$\times$) & 0.23 (10.2$\times$) \\
  & Lat.\ (ms) & 7.4 (3.3$\times$) & 6.9 (3.5$\times$) & 6.4 (3.8$\times$) & 6.2 (3.9$\times$) & 6.4 (3.8$\times$) & 6.4 (3.8$\times$) & 7.1 (3.5$\times$) \\
  & Size (MB) & 4.1 (19.9$\times$) & 4.2 (19.6$\times$) & 4.1 (19.8$\times$) & 4.1 (19.7$\times$) & 4.2 (19.4$\times$) & 4.2 (19.6$\times$) & 4.1 (20.0$\times$) \\
  \midrule
  $\approx$\textbf{99\%}
  & Acc.  & \textbf{50.6} & 21.7 & 7.3 & 1.2 & 7.4 & 2.7 & 14.5 \\
  & FLOPs (G) & 0.081 (28.9$\times$) & 0.041 (57.0$\times$) & 0.037 (63.7$\times$) & 0.037 (63.7$\times$) & 0.037 (63.7$\times$) & 0.037 (63.7$\times$) & 0.047 (49.1$\times$) \\
  & Lat.\ (ms) & 5.7 (4.3$\times$) & 5.3 (4.6$\times$) & 5.3 (4.6$\times$) & 5.3 (4.6$\times$) & 5.3 (4.6$\times$) & 5.3 (4.6$\times$) & 5.3 (4.6$\times$) \\
  & Size (MB) & 1.4 (59.6$\times$) & 1.3 (64.2$\times$) & 1.3 (64.9$\times$) & 1.3 (64.9$\times$) & 1.3 (64.9$\times$) & 1.3 (64.9$\times$) & 1.3 (63.9$\times$) \\
  \bottomrule
\end{tabular}
\label{tab:resnet-cifar100}
\end{table}

\paragraph{Efficiency and Pareto frontiers.}
Tables~\ref{tab:densenet-cifar10} and~\ref{tab:resnet-cifar100} also report the hardware impact of PCD's compression on the two representative architectures, measured on CPU inference for every method. On DenseNet-121/CIFAR-10, $\tau = 0.01$ yields a $5.3\times$ FLOPs reduction, $1.8\times$ latency reduction, and $9.4\times$ file-size reduction while matching unpruned accuracy ($94.0\%$ vs.\ $94.5\%$). Increasing to $\tau = 0.05$ reaches $45.4\times$ FLOPs reduction and $83\times$ file-size reduction at $89.2\%$ accuracy. ResNet-34/CIFAR-100 shows the same pattern: small accuracy costs traded for order-of-magnitude resource savings.

\begin{figure}[h]
    \centering
    \includegraphics[width=1\linewidth]{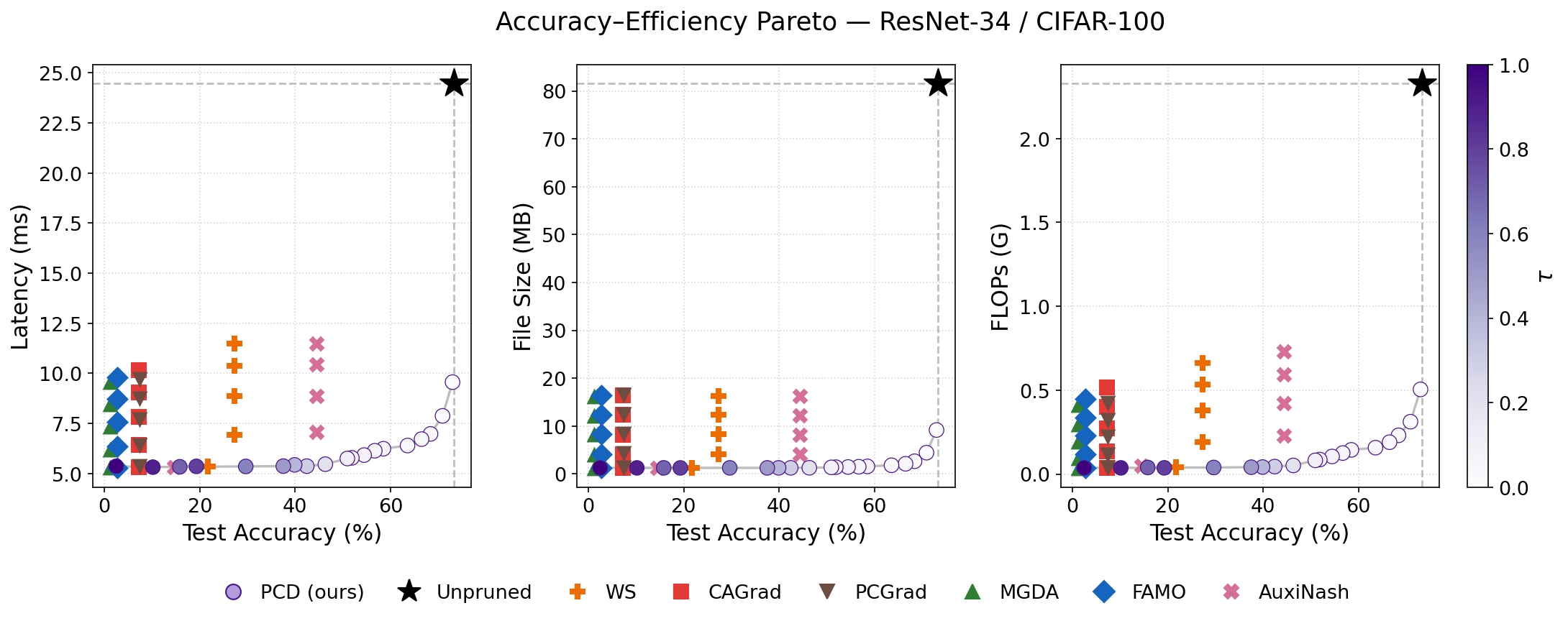}
    \caption{Accuracy–efficiency Pareto frontiers for PCD and all baselines on ResNet-34/CIFAR-100 across inference latency (ms), model file size (MB), and FLOPs (G). Lower is better. Each baseline contributes its physically pruned network at the 80, 85, 90, 95, and 99\% targets, with cost measured per method.}
    \label{fig:resnet-cifar100-pareto}
\end{figure}

\begin{figure}[h]
    \centering
    \includegraphics[width=1\linewidth]{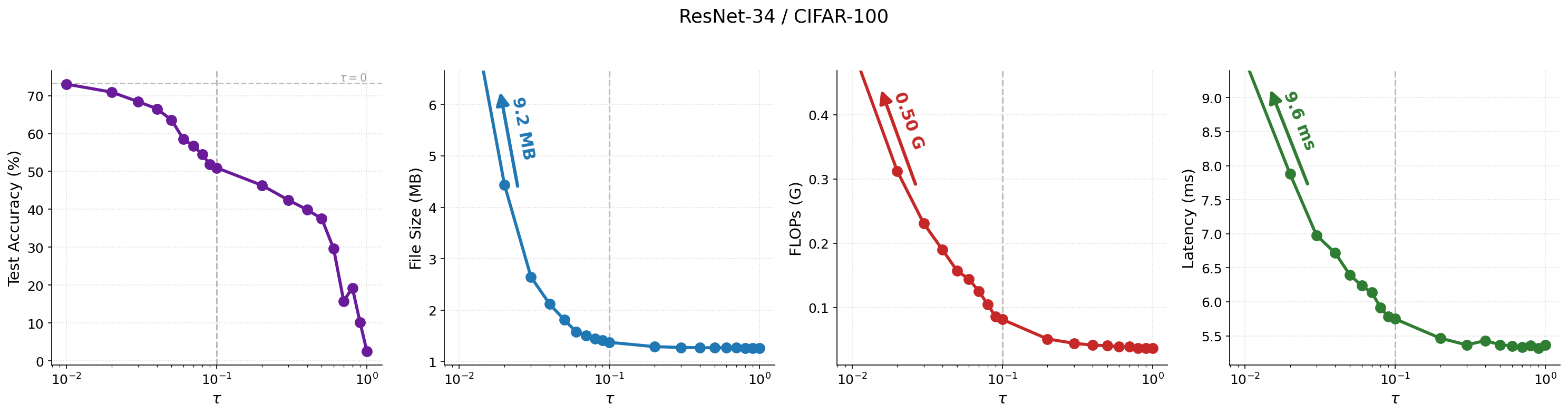}
    \caption{Accuracy, file size (MB), FLOPs (G) and latency (ms) as a function of the $\tau$ parameter for ResNet-34 trained on the CIFAR-100 dataset. Arrows pointing up show values for the respective efficiency metrics at $\tau = 0.01$, since that point lies above the plotted range of the 3 cost panels.}
    \label{fig:tau-sweep}
\end{figure}

Fig.~\ref{fig:resnet-cifar100-pareto} shows accuracy--efficiency Pareto frontiers against all baselines for ResNet-34/CIFAR-100. PCD's $\tau$ sweep traces a curve across the (accuracy, resource) plane on each of the three resource axes (latency, file size, FLOPs), while each baseline contributes a single operating point per (compression target) row of Tab.~\ref{tab:resnet-cifar100}. We plot the best operating point reached over its hyperparameter sweep for the swept baselines, and the default-setting operating point for the others. On every axis and at every compression level shown, baseline operating points lie above the PCD curve. The comparison is asymmetric in structure (a swept curve against a small set of points), but it is the natural comparison given that the swept baselines' best operating points already saturate their accuracy ceiling: increasing $w$ for WS or $c$ for CAGrad beyond the chosen value does not yield further accuracy improvement. The takeaway is operational rather than topological: not only can $\tau$ be tuned to land at a chosen point along PCD's frontier, but also PCD's operating points are simultaneously better than the baselines' best operating points across all three resource axes.

A practical consequence is that $\tau$ acts as a single, interpretable compression dial. Rather than searching over weight ratios (WS), conflict coefficients (CAGrad), or separate regularization schedules (MGDA / PCGrad / FAMO), a practitioner can sweep $\tau \in (0, 1]$ to trace a continuous Pareto curve at training time. Small $\tau$ (light purple in Fig.~\ref{fig:resnet-cifar100-pareto}) yields light compression with near-baseline accuracy, whereas large $\tau$ (dark purple) yields aggressive compression. The curve is monotone to within seed noise on most configurations, with no large discontinuous jumps as $\tau$ increases. The largest deviation is Inception/CIFAR-100, where accuracy across successive targets runs $68.5$, $61.4$, $61.6$, $57.2$, $28.6$. This regularity simplifies deployment-time selection based on a target resource budget. As Fig.~\ref{fig:tau-sweep} shows for ResNet-34/CIFAR-100, the largest resource gains come from $\tau \in (0, 0.1]$, with curves plateauing beyond this. Accuracy drops are also gentlest within this range. For $\tau > 0.1$ on this configuration, accuracy degrades sharply. The same qualitative pattern holds across the other seven (architecture, dataset) combinations. Per-configuration $\tau$ sweeps are reported in App.~\ref{app:expanded-structured-pruning}. Sec.~\ref{sec:tuned-baselines} complements this comparison against general MOO methods with the tuned, pruning-specific, and constrained baselines.

\subsubsection{Comparison with Tuned Pruning and Constrained Baselines}
\label{sec:tuned-baselines}

The baselines presented in Sec.~\ref{sec:moo-baselines} are general multi-objective methods. In this section, we evaluate PCD against task-specific pruning methods: \emph{tuned group-lasso scalarization} \citep{wen2016ssl} (the weighted sum of Sec.~\ref{sec:structured-pruning} with its effective penalty swept on a logarithmic grid), \emph{proximal group lasso} (Adam on cross-entropy followed by a per-group soft-threshold, in the spirit of ProxGen \citep{yun2021proxgen}), the \emph{constrained Lagrangian} route of Cooper \citep{gallego2025cooper}, and \emph{DBGD} \citep{gong2021bi} in both of its published branches (the $\beta$ branch, which Sec.~\ref{sec:pcd-qp} shows is PCD's raw-gradient instance, and the $\alpha$ branch). DBGD also serves as the representative of the lexicographic family: its lexicographic convergence analysis concerns exactly the update evaluated here, and Cooper covers the constrained route. The comparison is restricted to training-time, native-sparsity methods under one identical protocol. Post-hoc magnitude pruning with fine-tuning \citep{han2015learning}, BatchNorm-scale slimming \citep{liu2017slimming}, and $L_0$ regularization \citep{louizos2018l0} are the adjacent families whose additional mechanisms (fine-tuning, per-layer scaling, stochastic gates) the protocol deliberately excludes, so that the optimizer, and not the pipeline, is what differs. 

\paragraph{Experimental Protocol.} We evaluate these methods on ResNet-34 and DenseNet-121 on both CIFAR-10 and CIFAR-100, across 3 seeds over 300 epochs. Optimizers, learning rates, schedules, and batch sizes are identical across all runs, with no weight decay or warm-up. We report \emph{native group sparsity}: the fraction of output-channel groups driven exactly to zero by training itself, with no post-hoc thresholding. For tuned weighted sum, this section sweeps $\lambda_{\mathrm{eff}} \in \{10^{-4}, 3\!\times\!10^{-4}, 10^{-3}, 3\!\times\!10^{-3}, 10^{-2}\}$, with DBGD's $\beta$ on the same logarithmic grid plus its published default $\beta = 1$, $\alpha \in \{0.01, 0.1, 1\}$, Cooper's $\epsilon$ on six values in $[0.02, 0.9]$ plus a dual-learning-rate sensitivity sweep, and the proximal $\lambda$ on five values in $[0.01, 3]$. Comparisons are made at matched compression: each baseline point is paired with the PCD $\tau$ of nearest mean sparsity (accuracy is never consulted in the matching), and the PCD curve in Fig.~\ref{fig:tuned-zoom} is thinned to exactly those matched $\tau$ values, so no method contributes a denser curve than any other.

In total the study covers 28 baseline operating points per configuration (330 runs). Note that weighted sum combines $(1-w)\vg_1 + w\vg_2$, so its effective group-lasso coefficient is $\lambda_{\mathrm{eff}} = w/(1-w)$, and the grid of Sec.~\ref{sec:structured-pruning} ($w \geq 0.1$) floors $\lambda_{\mathrm{eff}}$ at $0.111$, 2 to 4 orders of magnitude above the range in which group-lasso scalarization operates best.

\begin{figure}[t]
    \centering
    \includegraphics[width=0.92\linewidth]{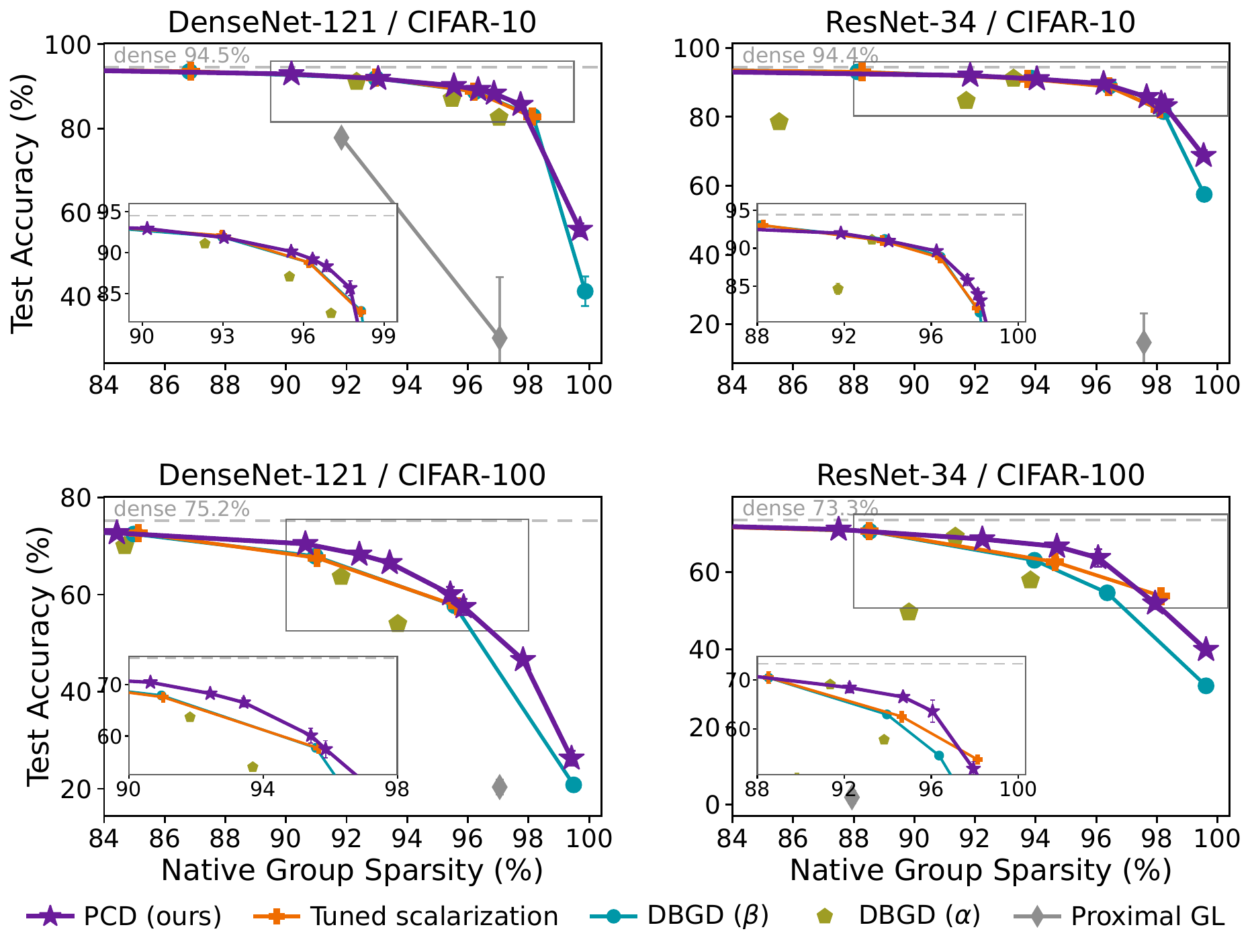}
    \caption{Test accuracy vs.\ native group sparsity in the compression band ($84$--$100\%$) across all 4 configurations. PCD (\textcolor{PCD}{purple}, $\filledstar$) against tuned group-lasso scalarization (\textcolor{WS}{orange}), DBGD's $\beta$ (\textcolor{DBGD}{cyan}) and $\alpha$ (\textcolor{DBGDa}{olive}) branches, and proximal group lasso (\textcolor{ProxGLc}{gray}). Cooper never exceeds $83.5\%$ native sparsity and therefore does not appear in the band (full range: Fig.~\ref{fig:tuned-full}, App.~\ref{app:tuned-expanded}).}
    \label{fig:tuned-zoom}
\end{figure}

Fig.~\ref{fig:tuned-zoom} reports the compression band and Table~\ref{tab:tuned-scoreboard} reports the matched head-to-head record. The full sparsity range, the full-range record, the per-configuration tables with seed standard deviations, and the Cooper sensitivity sweep are in App.~\ref{app:tuned-expanded}. In the compression regime ($\geq 85\%$ sparsity), PCD is better in $28$ of $47$ matched comparisons, worse in $1$, and within seed noise in $18$. $25$ of the $47$ baseline points are Pareto-dominated by some point of PCD's own sweep. The single loss is a genuine crossing, visible in Fig.~\ref{fig:tuned-zoom}: tuned scalarization reaches $53.8\%$ accuracy at $98.1\%$ sparsity on ResNet-34/CIFAR-100, against PCD's $51.9\%$ at $97.9\%$ ($\tau = 0.09$).

\begin{table}[h]
\centering\small
\caption{Matched head-to-head record of the tuned-baseline study. A comparison counts as a win or loss only when the accuracy gap exceeds the standard deviation of the difference, $\sqrt{\sigma_{\mathrm{PCD}}^2+\sigma_{\mathrm{base}}^2}$, and the winner is not behind on sparsity by more than one point. Otherwise, it is within noise. A baseline point is \emph{dominated} when some point of PCD's sweep attains greater-or-equal accuracy and sparsity, at least one strictly.}
\begin{tabular}{l|cccc}
\toprule
\textbf{Regime} & \textbf{PCD better} & \textbf{PCD worse} & \textbf{Within noise} & \textbf{Dominated} \\
\midrule
Full range (102 comparisons) & 44 & 4 & 54 & 49/102 \\
Compression regime, $\geq 85\%$ sparsity (47) & \textbf{28} & \textbf{1} & 18 & 25/47 \\
\bottomrule
\end{tabular}
\label{tab:tuned-scoreboard}
\end{table}

We note the advantage of normalization here briefly: within the same setup and the same nominal knob value $\tau = \beta = 0.01$ (the per-step identity $\beta_j = \tau s_j/s_1$ of Sec.~\ref{sec:pcd-qp}), the two parameterizations land at very different points of each problem's trade-off (Table~\ref{tab:knob-transfer}): PCD's accuracy retention spans a $1.09\%$ range across the four configurations, while DBGD's spans $13.23\%$. The same nominal knob buys a consistent accuracy price under PCD and a configuration-dependent one under DBGD. To expand on this further, the normalization scheme itself is ablated across five variants in App.~\ref{app:norm-ablation}.

\begin{table}[h]
\centering\small
\caption{Same nominal hyperparameter, same code, same seeds: $\tau = \beta = 0.01$. Accuracy is mean $\pm$ seed standard deviation, retention is accuracy as a fraction of the dense model's.}
\begin{tabular}{l|c|cc|cc}
\toprule
\textbf{Configuration} & \textbf{Dense} & \textbf{PCD acc.} & \textbf{DBGD ($\beta$) acc.} & \textbf{PCD ret.} & \textbf{DBGD ret.} \\
\midrule
DenseNet-121 / CIFAR-10  & 94.5 & $94.03 \pm 0.19$ & $82.87 \pm 0.48$ & 99.50\% & 87.69\% \\
DenseNet-121 / CIFAR-100 & 75.2 & $74.82 \pm 0.28$ & $57.71 \pm 0.79$ & 99.50\% & 76.74\% \\
ResNet-34 / CIFAR-10     & 94.4 & $93.02 \pm 0.33$ & $81.50 \pm 0.42$ & 98.53\% & 86.33\% \\
ResNet-34 / CIFAR-100    & 73.3 & $73.03 \pm 0.49$ & $54.58 \pm 0.23$ & 99.63\% & 74.46\% \\
\midrule
\textbf{Spread} & & & & \textbf{1.09} & \textbf{13.23} \\
\bottomrule
\end{tabular}
\label{tab:knob-transfer}
\end{table}

\subsection{Unstructured Sparsity and Low-Rankness}
\label{sec:k3-experiments}

The structured-pruning study of Sec.~\ref{sec:structured-pruning} pairs the task loss with a single secondary. We now move to a genuine three-objective setting ($K=3$), in which the primary cross-entropy loss $\mathcal{L}_1$ is constrained \emph{simultaneously} by two structurally distinct regularizers. The first is an $\ell_1$ penalty, which drives unstructured weight sparsity ($\mathcal{L}_2(\theta) \;=\; \|\theta\|_1 \;=\; \sum_i |\theta_i|$), and the second is a nuclear-norm penalty, which drives each layer toward low rank ($\mathcal{L}_3(\theta) \;=\; \sum_{\ell} \|W_\ell\|_* \;=\; \sum_{\ell} \sum_k \sigma_k(W_\ell)$ where $W_\ell$ is the weight matrix of layer $\ell$ and $\sigma_k(\cdot)$ its $k$-th singular value). As in the structured case the hierarchy is asymmetric by construction: task accuracy is the deliverable, while sparsity and low-rankness are secondary structural objectives whose pressure is modulated by $\tau$.

Importantly, the two secondaries are \emph{not} aligned. The $\ell_1$ penalty acts entry-wise, zeroing individual weights wherever they are small and without regard to their position, whereas the nuclear norm acts on the spectrum, collapsing entire singular directions and thereby favoring weight matrices whose rows and columns are linearly dependent. Although in heavily over-compressed regimes the two objectives can appear correlated, their useful operating ranges impose competing pressures on shared parameters. A single $\tau$ is meaningful here as it is not a weight applied to two penalties. The constraint $\gt_j^\T \dt \geq \tau\norm{\gt_j}^2$ is stated on EMA-normalized gradients, so it demands the same fraction of each secondary's \emph{own} maximal normalized progress.

\paragraph{Experimental Protocol. }
We evaluate on ResNet-34~\citep{resnet} and an Inception-style model~\citep{inception}, each trained on CIFAR-10 and CIFAR-100~\citep{krizhevsky2009learning}, following the protocol of Sec.~\ref{sec:structured-pruning}. We train from scratch with Adam, learning rate $10^{-3}$ with cosine annealing for 300 epochs. Reported numbers are means across 3 seeds, compared against PCGrad, MGDA, FAMO, CAGrad, weighted sum (WS), and AuxiNash. For baselines with a tuneable hyperparameter, we sweep $\{0.1, 0.2, \ldots, 1.0\}$ (for a total of 10 values), while baselines with no tuneable hyperparameter are evaluated on their default settings. Regardless of the cardinality asymmetry, equal-cardinality comparisons that remove this asymmetry are reported in App.~\ref{app:k3-expanded}.

We summarize each method in two complementary ways. First, on the raw and directly interpretable axes of test accuracy, parameter sparsity, and effective rank, we trace PCD's operating point as a function of $\tau$ and overlay each swept baseline's \emph{full} operating curve, drawn against its own hyperparameter on the shared axis (Fig.~\ref{fig:hero-tausweep}). To compare the full achievable sets in a single scalar, we report the dominated hypervolume (Fig.~\ref{fig:hypervolume}). Given a method's set of operating points and a fixed reference point, the hypervolume is the volume of this normalized space that the set dominates. It rewards a method both for approaching the joint ideal on all three axes at once and for covering more of the trade-off surface, so a larger hypervolume indicates that a method's achievable set is jointly superior rather than merely better on one axis in isolation. Low-rankness is indicated by the \emph{effective rank}~\citep{roy2007effective}.

    \begin{figure*}[h]
    \centering
    \includegraphics[width=1\linewidth]{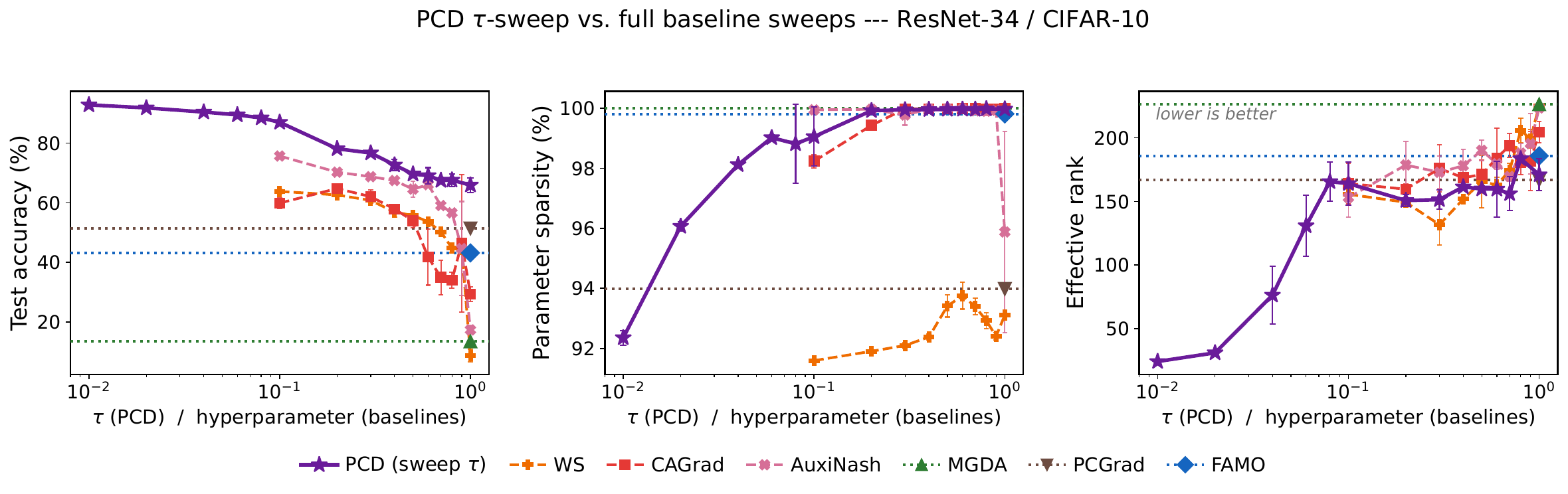}
    \caption{PCD's $\tau$-sweep on the three deployment-relevant axes for ResNet-34/CIFAR-10: test accuracy, parameter sparsity, and effective rank (lower is better). PCD (\textcolor{PCD}{$\filledstar$}, solid) is swept over $\tau\in(0,1]$. Sparsity in this figure is the native final-epoch value.}
    \label{fig:hero-tausweep}
\end{figure*}

\begin{table}[h]
\centering \small
\caption{Operating points on the $K=3$ (cross-entropy\,/\,$\ell_1$\,/\,nuclear norm) task, across all 4 architecture/dataset configurations. For each baseline we report the operating point of highest mean test accuracy. For PCD we report a single fixed tolerance ($\tau=0.02$), a representative low-tolerance setting that trades a small amount of accuracy for stronger compression, rather than its accuracy-optimal point.}
\label{tab:lowrank}
\setlength{\tabcolsep}{4pt}
\begin{tabular}{ll|ccccccc}
  \toprule
  \textbf{Configuration} & \textbf{Metric} & \textbf{PCD (Ours)} & \textbf{WS} & \textbf{CAGrad} & \textbf{MGDA} & \textbf{PCGrad} & \textbf{FAMO} & \textbf{AuxiNash} \\
  \midrule
  \multirow{3}{*}{\begin{tabular}[c]{@{}l@{}}ResNet-34 \\ CIFAR-10\end{tabular}}
   & Acc. (\%) & \textbf{91.8} & 63.8 & 64.8 & 13.5 & 51.4 & 43.2 & 75.7 \\
   & Sparsity (\%) & 95.2 & 90.8 & 98.5 & 99.1 & 93.1 & 98.9 & 99.0 \\
   & Eff. rank $\downarrow$ & 30.9 & 155.6 & 159.6 & 226.4 & 166.8 & 185.8 & 153.1 \\
  \midrule
  \multirow{3}{*}{\begin{tabular}[c]{@{}l@{}}ResNet-34 \\ CIFAR-100\end{tabular}}
   & Acc. (\%) & \textbf{71.0} & 34.2 & 36.1 & 1.1 & 30.0 & 20.3 & 48.7 \\
   & Sparsity (\%) & 92.6 & 90.6 & 98.9 & 98.9 & 92.2 & 98.4 & 98.8 \\
   & Eff. rank $\downarrow$ & 24.5 & 137.9 & 186.3 & 227.6 & 160.0 & 190.5 & 184.8 \\
  \midrule
  \multirow{3}{*}{\begin{tabular}[c]{@{}l@{}}Inception \\ CIFAR-10\end{tabular}}
   & Acc. (\%) & \textbf{89.2} & 52.1 & 38.2 & 23.1 & 33.8 & 10.0 & 60.3 \\
   & Sparsity (\%) & 86.1 & 90.8 & 99.1 & 99.1 & 91.6 & 99.2 & 99.0 \\
   & Eff. rank $\downarrow$ & 31.3 & 39.7 & 51.2 & 76.4 & 55.8 & 100.3 & 18.0 \\
  \midrule
  \multirow{3}{*}{\begin{tabular}[c]{@{}l@{}}Inception \\ CIFAR-100\end{tabular}}
   & Acc. (\%) & \textbf{67.8} & 29.8 & 14.1 & 4.3 & 12.4 & 1.0 & 31.9 \\
   & Sparsity (\%) & 76.1 & 86.6 & 94.3 & 94.5 & 86.8 & 94.5 & 93.1 \\
   & Eff. rank $\downarrow$ & 27.7 & 28.1 & 59.0 & 70.3 & 52.4 & 100.3 & 12.1 \\
  \bottomrule
\end{tabular}
\end{table}

\paragraph{Operating curves.}
Fig.~\ref{fig:hero-tausweep} reports the result for ResNet-34/CIFAR-10, with every method's full sweep drawn. The swept baselines' curves make their limitation directly visible: their hyperparameters barely traverse the space (accuracy varies while sparsity stays pinned near the compressed ceiling) so their 10 operating points cluster where their best-accuracy point already was. PCD's $\tau$, in contrast, traces a continuous curve on each axis, from the dense-accuracy regime to the compressed one.

The comparison is decisive on two axes and competitive on the third. Over the entire low-to-moderate range of $\tau$, PCD's test accuracy exceeds that of every baseline, remaining near the dense-network ceiling at small $\tau$ and only dropping below the strongest baseline once $\tau$ is pushed into the aggressively compressed regime. On the rank axis, PCD's effective rank lies far below every baseline at small $\tau$: it discovers low-dimensional structure rather than the degenerate near-zero spectra that inflate the baselines' effective rank. Parameter sparsity is the lone axis on which the baselines hold an initial edge. PCD starts less sparse and climbs with $\tau$, reaching the baseline cluster near $\tau \approx 0.2$. The baselines attain that sparsity only by collapsing the network and thus the accuracy. 

Table~\ref{tab:lowrank} quantifies these results. Across all configurations PCD retains the highest accuracy relative to the original dense networks, whereas the other methods over-compress, trading accuracy for gains on the $\ell_1$ and nuclear norms. PCD maintains competitive effective rank, with the only exceptions on the Inception-style architecture: AuxiNash reaches lower effective rank on both Inception configurations, despite trailing PCD by 30--40\% in accuracy. At matched unstructured sparsity, PCD's \emph{accuracy} advantage over each baseline that reaches its sparsity range is $+7.4\%$ to $+71.7\%$ depending on configuration and method (App.~\ref{app:k3-expanded}, Table~\ref{tab:k3-matched}). The single competitive baseline overall is AuxiNash, which nonetheless attains a substantially smaller dominated hypervolume than PCD (Fig.~\ref{fig:hypervolume}). 
 
\paragraph{Hypervolume. }
Fig.~\ref{fig:hypervolume} aggregates these comparisons into the dominated hypervolume of each method, per configuration. Every method contributes its full set of operating points: $15$ for PCD, $10$ for each swept baseline, and $1$ for the three hyperparameter-free methods. Across configurations PCD's operating set dominates a substantially larger region of the joint objective space than any baseline's, and the margin is stable across seeds. The hypervolume confirms that the per-axis separation in Fig.~\ref{fig:hero-tausweep} is not an artifact of any single axis or operating point. Since the metric can in principle reward set cardinality and be distorted by the normalization frame, App.~\ref{app:k3-expanded} repeats the comparison under equal cardinality (all $\binom{15}{10}$ subsets of PCD's grid), under a one-point-per-method rule, and under alternative normalization bounds. PCD leads under every variant.

\begin{figure*}[h]
    \centering
    \includegraphics[width=1\linewidth]{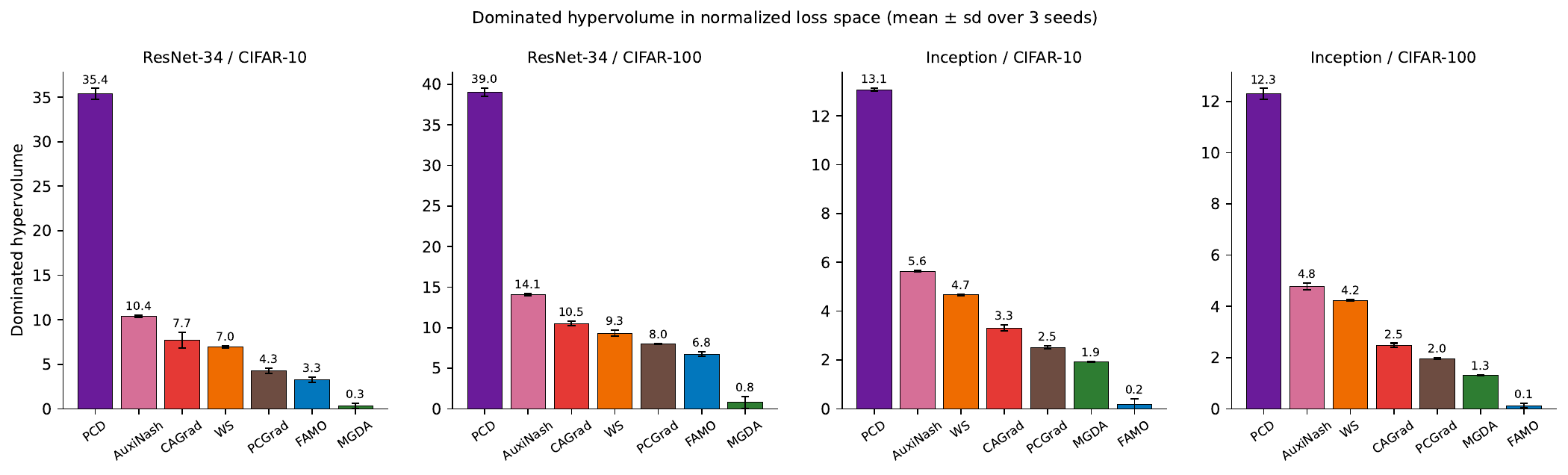}
    \caption{Dominated hypervolume in the normalized $K=3$ objective space, one panel per architecture/dataset configuration. For each configuration the three losses are mapped to goodness coordinates $g=1-\tilde{L}\in[0,1]^3$ (per-configuration $\log_{10}$ min--max normalization of each loss across all methods, ideal at $(1,1,1)$), and we report the volume of $[0,1]^3$ dominated by each method's full set of operating points relative to the origin reference $(0,0,0)$. Values are scaled by $100$ for readability and shown as mean\,$\pm$\,sd over $3$ seeds.}
    \label{fig:hypervolume}
\end{figure*}

\subsection{Synthetic Experiments}
\label{sec:synthetic}

We complement the deep-network results with controlled problems that expose quantities the network experiments cannot measure directly: the scale-dependence of a method's operating point, the feasibility
and active-set structure of the QP as objectives accumulate, and the exact stationarity reached at convergence. Throughout, the primary $\mathcal{L}_1(\theta)=\tfrac12\theta^{\top}H\theta-c^{\top}\theta$ is a convex quadratic in $\mathbb{R}^{50}$ with logarithmically spaced eigenvalues and condition number $\kappa(H)=10^{2}$, $c\sim\mathcal{N}(0,I)$, mimicking the long-tailed spectra of network Hessians \citep{sagun2017empirical,ghorbani2019investigation}. PCD uses the EMA normalization of \Cref{def:ema} with $\beta=0.999$, $\varepsilon=10^{-8}$. We compare against weighted sum (WS), MGDA, PCGrad, CAGrad, and AuxiNash. The first two experiments verify properties of the PCD construction itself: the scale invariance of its tolerance and the feasibility of its QP as objectives accumulate on convex problems where they are cleanly measurable. The third isolates the property that genuinely requires non-convexity: the strictly-stronger fixed-point guarantee of \Cref{thm:cms}, and how the resulting separation from symmetric methods scales with the number of objectives.

\subsubsection{Scale invariance of \texorpdfstring{$\tau$}{tau}}
\label{sec:syn-scale}

The meaning of $\tau$ rests on PCD's gradient normalization. As a result, we test whether a fixed $\tau$ yields a fixed operating point as a secondary is rescaled. Pairing the primary with a single convex secondary and holding $\tau=0.3$ (PCD) and $w=0.5$ (WS), we rescale $\mathcal{L}_2\mapsto c\,\mathcal{L}_2$ over $c\in[10^{-4},10^{6}]$ and record the achieved operating point, in the \emph{original} units of $\mathcal{L}_2$, on the exact Pareto front (Fig.~\ref{fig:syn-scale}). PCD's
operating point is invariant across all ten orders of magnitude (total spread $0.0075$ in
$\mathcal{L}_2$), confirming the asymptotic scale invariance of \Cref{thm:scale-inv}. WS instead traverses the entire front. Its achieved secondary ranges from $18.5$ at $c=10^{-4}$ to effectively $0$ at $c=10^{6}$, so the same nominal weight encodes a completely different priority depending on the arbitrary scale of the penalty, exactly the failure of \Cref{corr:ws-not-scale-invariant}. The EMA normalization is what makes $\tau$ transferrable across objectives whose raw magnitudes differ by orders of magnitude, the
regime typical of a task loss paired with secondary constraints.

\begin{figure}[h]
  \centering
  \begin{subfigure}{0.46\textwidth}\includegraphics[width=\linewidth]{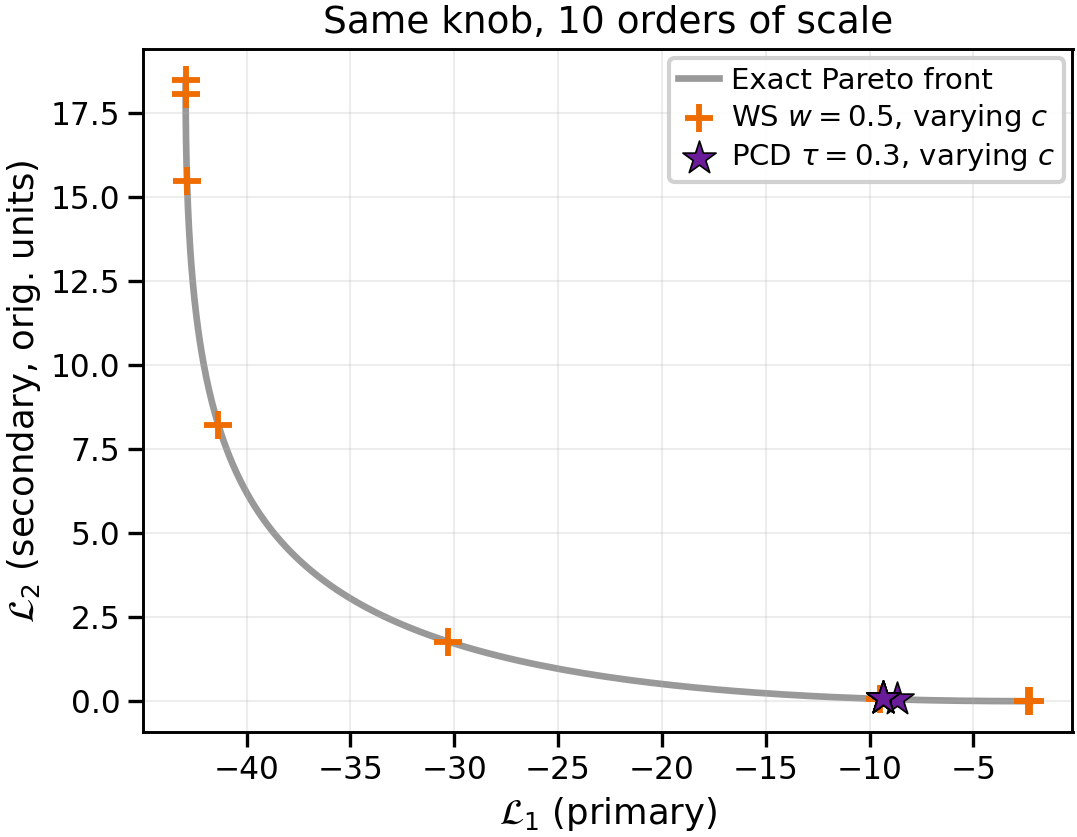}\caption{Operating points on the front}\end{subfigure}\hfill
  \begin{subfigure}{0.46\textwidth}\includegraphics[width=\linewidth]{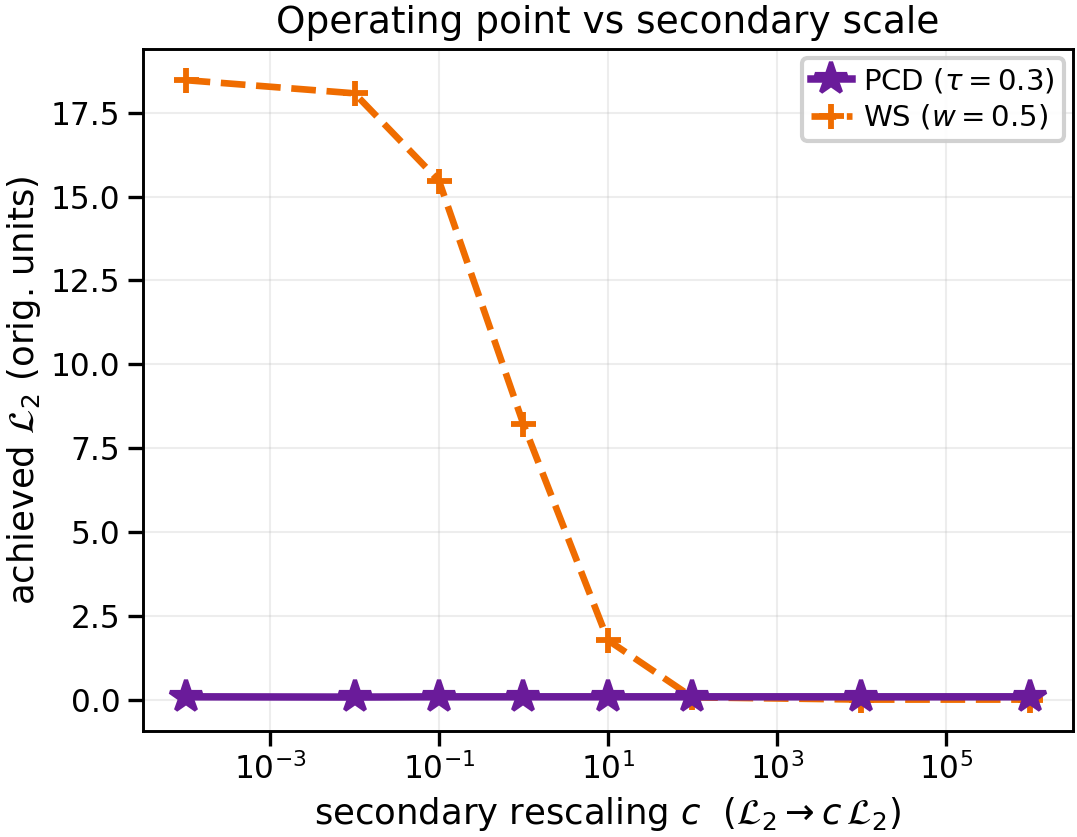}\caption{Achieved secondary vs.\ scale}\end{subfigure}
  \caption{\textbf{Scale invariance of $\tau$.} Rescaling the secondary by $c\in[10^{-4},10^{6}]$ at
  fixed knobs. (a) PCD ($\tau=0.3$) stays at one point on the front (spread $0.0075$) while WS
  ($w=0.5$) is dragged across it. (b) The achieved $\mathcal{L}_2$ in original units is flat for PCD and
  sweeps the full range for WS. PCD's knob has a scale-free meaning (\Cref{thm:scale-inv}), while WS's weight does not
  (\Cref{corr:ws-not-scale-invariant}).}
  \label{fig:syn-scale}
\end{figure}

\subsubsection{Feasibility and active set with more objectives}
\label{sec:syn-feasibility}

The one phenomenon specific to $K>2$ is that the secondary-progress polyhedron $\mathcal{F}_\tau(\theta)$ can in principle be empty (Sec.~\ref{sec:pcd-qp}), which we probe directly. Three quantities certify that the $K$-objective QP is well-posed: whether a feasible update exists at all (Fig.~\ref{fig:syn-feasibility}a), how many secondary constraints the solution enforces (Fig.~\ref{fig:syn-feasibility}b), and whether enforcement is spread across all secondaries or concentrated on a few (Fig.~\ref{fig:syn-feasibility}c). A constraint is \emph{binding} (active) when the solution meets it with equality, $\tilde g_j^{\top}\tilde d^{*}=\tau\lVert\tilde g_j\rVert^2$ (equivalently $\mu_j>0$), so that the secondary actively shapes the update; otherwise it is slack ($\mu_j=0$) and drops out of $\tilde d^{*}=\tilde g_1+\sum_{j\ge2}\mu_j\tilde g_j$. For $K\in\{2,\dots,8\}$ we draw random normalized secondary-gradient configurations and evaluate the Farkas feasibility margin of App.~\ref{app:feasibility} (positive iff $\mathcal{F}_\tau\neq\emptyset$). Fig.~\ref{fig:syn-feasibility}(a) shows the margin shrinking with $K$. More simultaneous directional constraints leave less slack, but the margin remains strictly positive throughout, with $0$ infeasible events in $1785$ checks, empirically consistent with infeasibility being a Lebesgue measure-zero event in the high-dimensional regime (\Cref{cor:infeas}(iii)). Figs.~\ref{fig:syn-feasibility}(b) and (c) characterize the solution at $K=6$: the active set grows monotonically with $\tau$, from $2.54$ active constraints at $\tau=0$ to $4.96$ of $5$ (since we have one primary and 5 secondaries for $K = 6$) as $\tau\to1$. The receding half-spaces tighten the polyhedron (Sec.~\ref{sec:pcd-geom}), and every secondary binds with high frequency once $\tau\gtrsim0.05$. Fig.~\ref{fig:syn-feasibility}(c) shows this binding distributed roughly uniformly across the five secondaries rather than concentrated on a subset: each is a genuinely active constraint that the projection meets at equality, not one incidentally over-satisfied while a few dominate, so the QP exercises all $K-1$ constraints rather than collapsing to a low-effective-$K$ problem. The $K$-objective QP is thus well-posed in practice and behaves as the geometry of Sec.~\ref{sec:pcd} predicts.

\begin{figure}[t]
  \centering
  \begin{subfigure}{0.325\textwidth}\includegraphics[width=\linewidth]{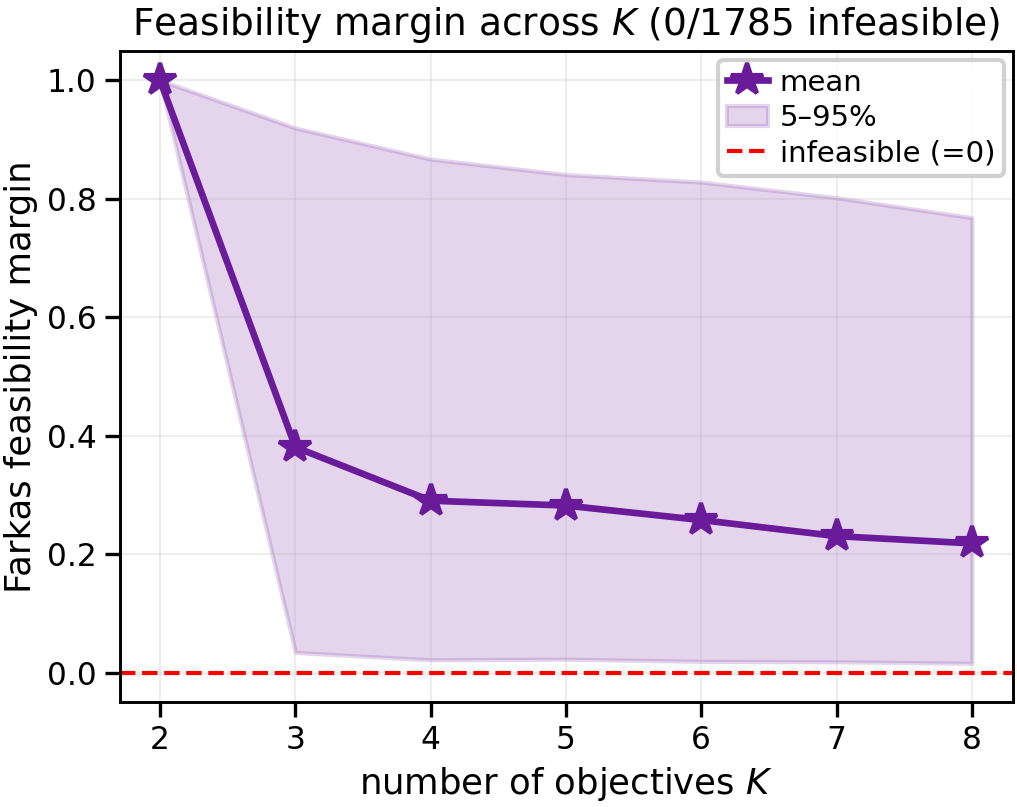}\caption{Feasibility margin vs.\ $K$}\end{subfigure}\hfill
  \begin{subfigure}{0.325\textwidth}\includegraphics[width=\linewidth]{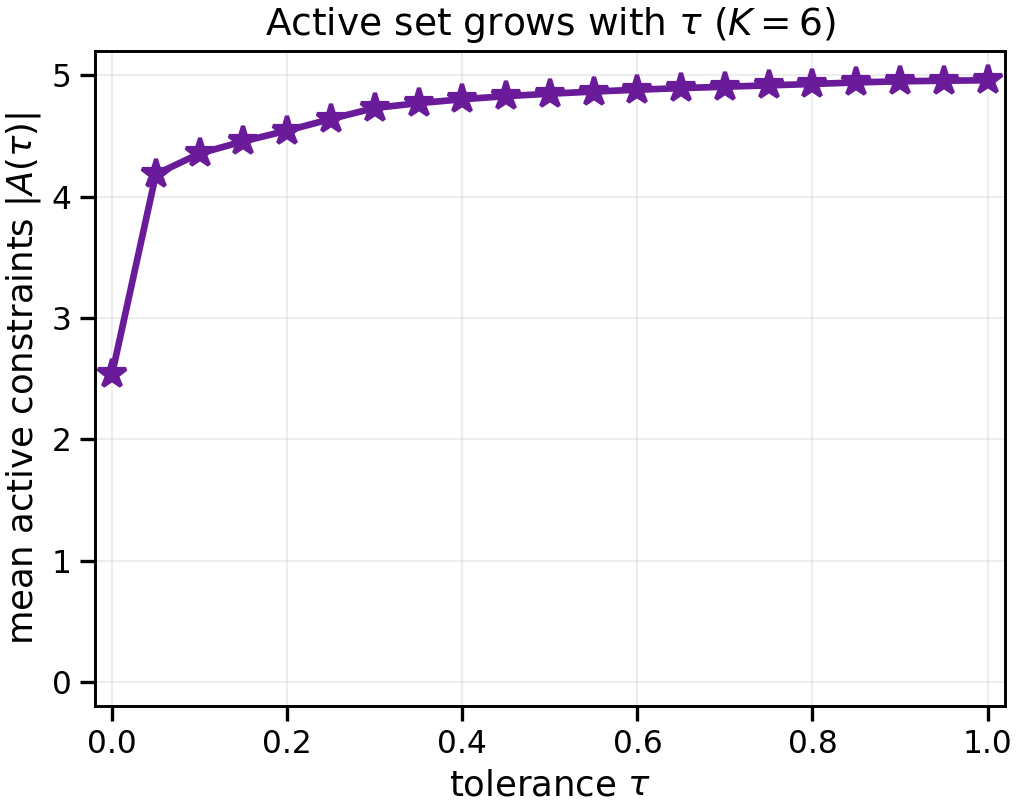}\caption{Active set vs.\ $\tau$ ($K{=}6$)}\end{subfigure}\hfill
  \begin{subfigure}{0.325\textwidth}\includegraphics[width=\linewidth]{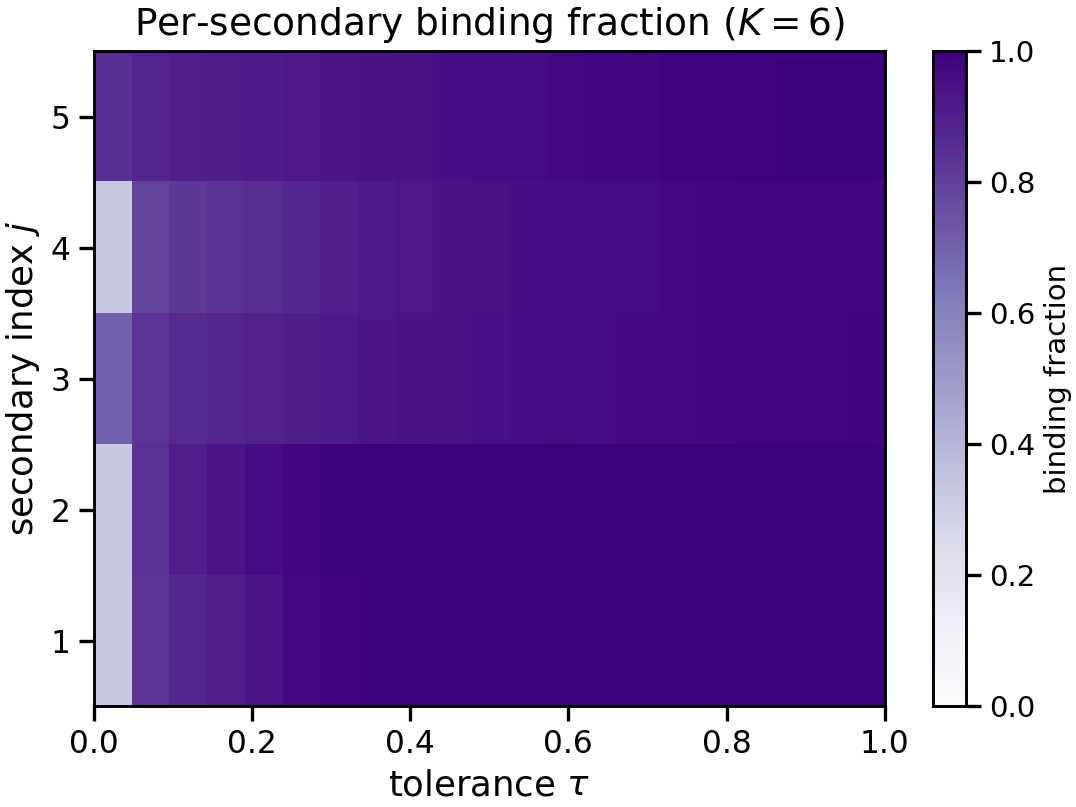}\caption{Per-secondary binding ($K{=}6$)}\end{subfigure}
  \caption{\textbf{Feasibility and active-set behavior at $K>2$.} (a) The Farkas feasibility margin
  decreases with $K$ but stays strictly positive; $0$ of $1785$ configurations were infeasible
  (\Cref{cor:infeas}(iii)). (b) The active set grows monotonically with $\tau$ ($2.54\!\to\!4.96$ of $5$). (c)
  Each secondary binds with high frequency for $\tau\gtrsim0.05$.}
  \label{fig:syn-feasibility}
\end{figure}

\subsubsection{Conflict equilibria and CMS escape as the number of objectives grows}
\label{sec:syn-conflict}

The final experiment isolates the property that separates PCD's fixed-point set from that of symmetric methods (\Cref{thm:cms}) and measures how that separation scales with $K$. We construct a \emph{non-convex} primary with a double-well along each of $K-1$ random orthonormal directions $u_j$ (plus a convex bowl in the orthogonal complement), together with $K-1$ secondaries, the $j$-th pulling $\langle u_j,\theta\rangle\to+1$. The secondaries are jointly satisfiable, so a composite multi-stationary point (at which every objective is individually stationary) exists at
$\theta^{\star}=\sum_j u_j$. This is what makes the problem an \emph{escape} problem rather than a trade-off. All methods are initialized in a wrong basin ($\langle u_j,\theta\rangle=-0.6$ for every $j$), where the canceled primary gradient is largest. Concretely, writing $s_j=\langle u_j,\theta\rangle$, the primary is
\[
  \mathcal{L}_1(\theta)=\sum_{j=1}^{K-1}\bigl(s_j^2-1\bigr)^2+\tfrac12\lVert P_\perp\theta\rVert^2,
\]
where $P_\perp$ projects onto the orthogonal complement of $\operatorname{span}\{u_j\}$. Each well $W(s)=(s^2-1)^2$ has minima at $s=\pm1$ and a barrier at $s=0$, and the $j$-th secondary pulls $s_j\to+1$, so the joint optimum $\theta^{\star}=\sum_j u_j$ lies across the barrier from the initialization. At the shared start $s_j=-0.6$ the primary-gradient component along each $u_j$ has magnitude $\lvert W'(-0.6)\rvert=\lvert 4(-0.6)\bigl((-0.6)^2-1\bigr)\rvert=1.536$, near the steepest point of the barrier. Since the $u_j$ are orthonormal, these $K-1$ components add in quadrature, giving $\lVert\nabla\mathcal{L}_1\rVert=1.536\sqrt{K-1}$ (the dashed line in Fig.~\ref{fig:syn-conflict}b). We sweep $K\in\{2,\dots,7\}$ over $5$ seeds.

\textbf{Combination methods are pinned at the equilibrium.} A conflict equilibrium is a point at which
$0\in\operatorname{conv}\{g_1,\dots,g_K\}$ while no objective is individually stationary. This halting is not specific to our setting: any method whose update lies in $\operatorname{conv}\{g_i\}$ (MGDA exactly, and CAGrad and WS up to their trust region and weighting) returns $0$ precisely when $0\in\operatorname{conv}\{g_i\}$, which is the defining condition of a conflict equilibrium. MGDA returns the minimum-norm element of $\operatorname{conv}\{g_i\}$, CAGrad anchors on their average within a trust region, and WS forms a fixed combination. What \emph{is} specific to our construction is the rate at which the uncanceled primary gradient grows with $K$: Fig.~\ref{fig:syn-conflict}(a) and (b) show MGDA and CAGrad halting exactly at the equilibrium, the canceled primary gradient tracking the $1.536\sqrt{K-1}$ law derived above (dashed line) across the entire range. 

The primary loss they leave unconverged grows linearly in $K$ (Fig.~\ref{fig:syn-conflict}(c)). WS is omitted from the sweep because its behavior is weight-dependent and does not reduce to a single curve in $K$. For any primary-respecting weight it traps with the same combination-method geometry.

\textbf{PCD reaches CMS for every $K$.} PCD's update is $\tilde g_1+\sum_{j\ge2}\mu_j\tilde g_j$ with the primary coefficient pinned at one and $\mu_j\ge0$ (\Cref{thm:kkt}). It is structurally \emph{not} a convex combination of the gradients, and \Cref{thm:cms} guarantees the normalized direction vanishes only when every oracle gradient does. Across all $K$, PCD empirically crosses every barrier and converges to the CMS point for every $\tau >0$: its primary gradient at convergence is at machine precision ($\sim\!10^{-15}$) and its primary loss is $\sim\!10^{-29}$, against the unbounded growth of the combination methods. The learning curve at $K=5$ (Fig.~\ref{fig:syn-conflict}(a)) makes the distinction visible. PCD drives the primary gradient to zero within a few hundred steps while MGDA and CAGrad stay pinned at $3.07$.

\begin{figure}[h]
  \centering
  \begin{subfigure}{0.325\textwidth}\includegraphics[width=\linewidth]{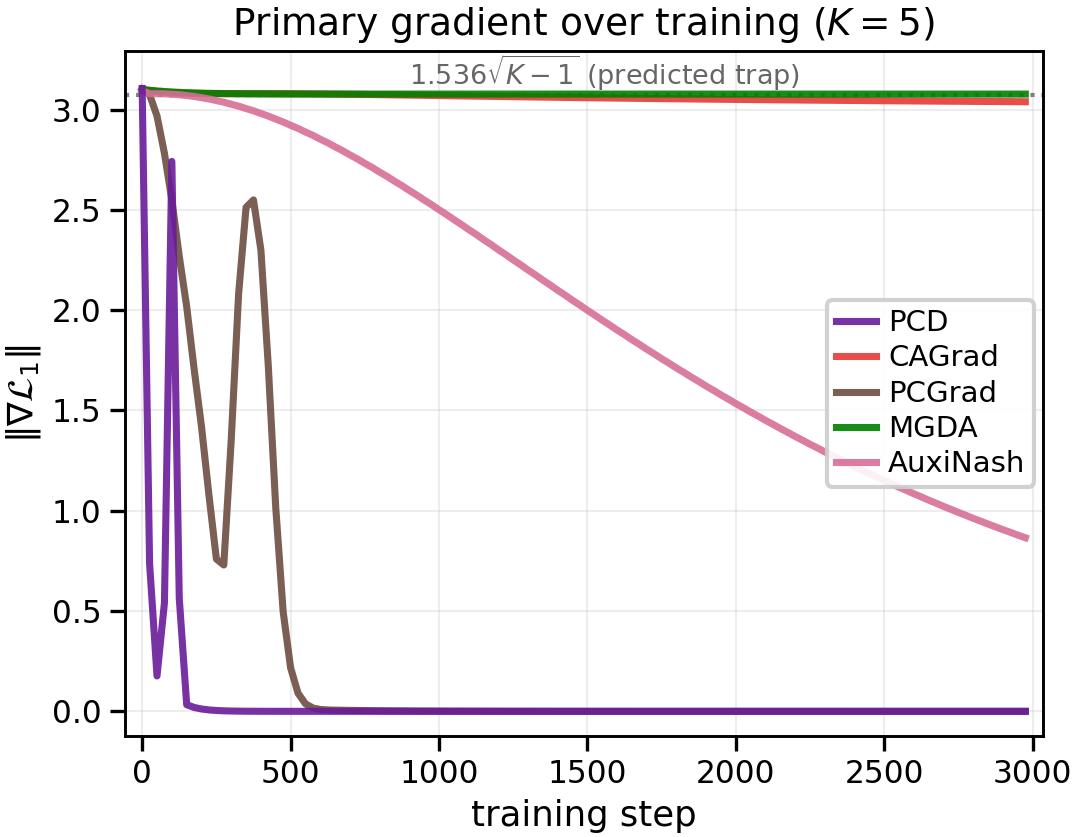}\caption{$\|\nabla\mathcal{L}_1\|$ over training ($K{=}5$)}\end{subfigure}\hfill
  \begin{subfigure}{0.325\textwidth}\includegraphics[width=\linewidth]{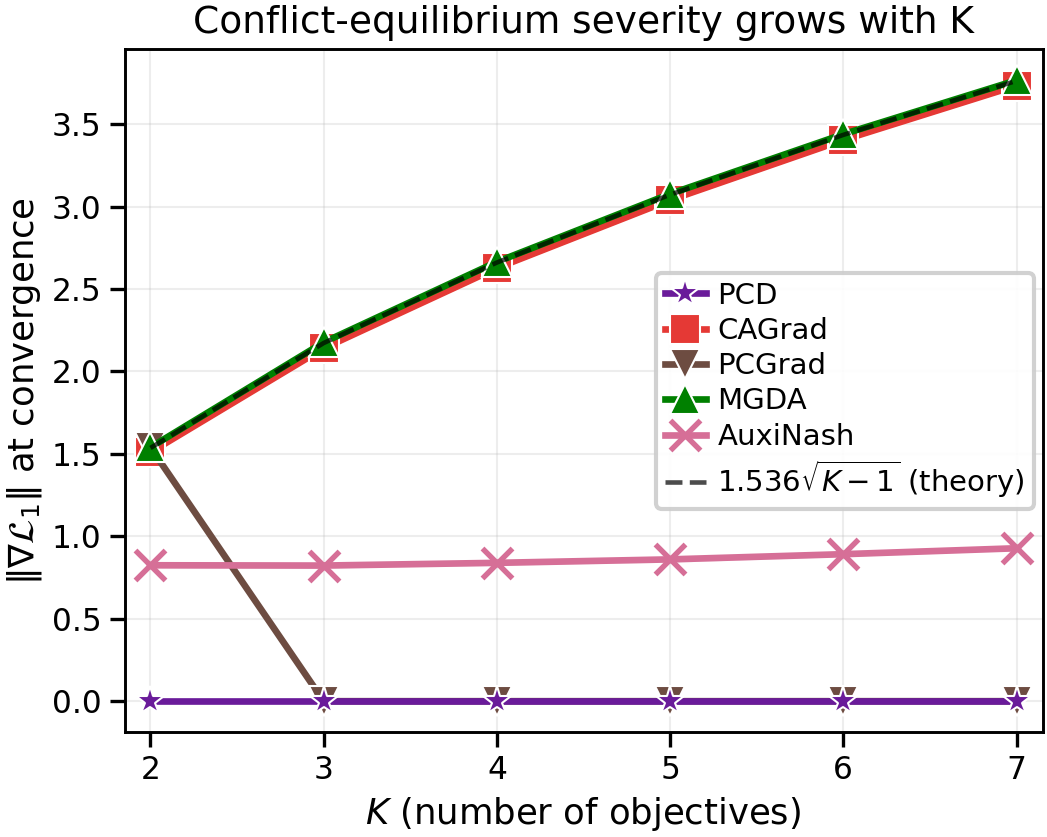}\caption{$\|\nabla\mathcal{L}_1\|$ at convergence vs.\ $K$}\end{subfigure}\hfill
  \begin{subfigure}{0.325\textwidth}\includegraphics[width=\linewidth]{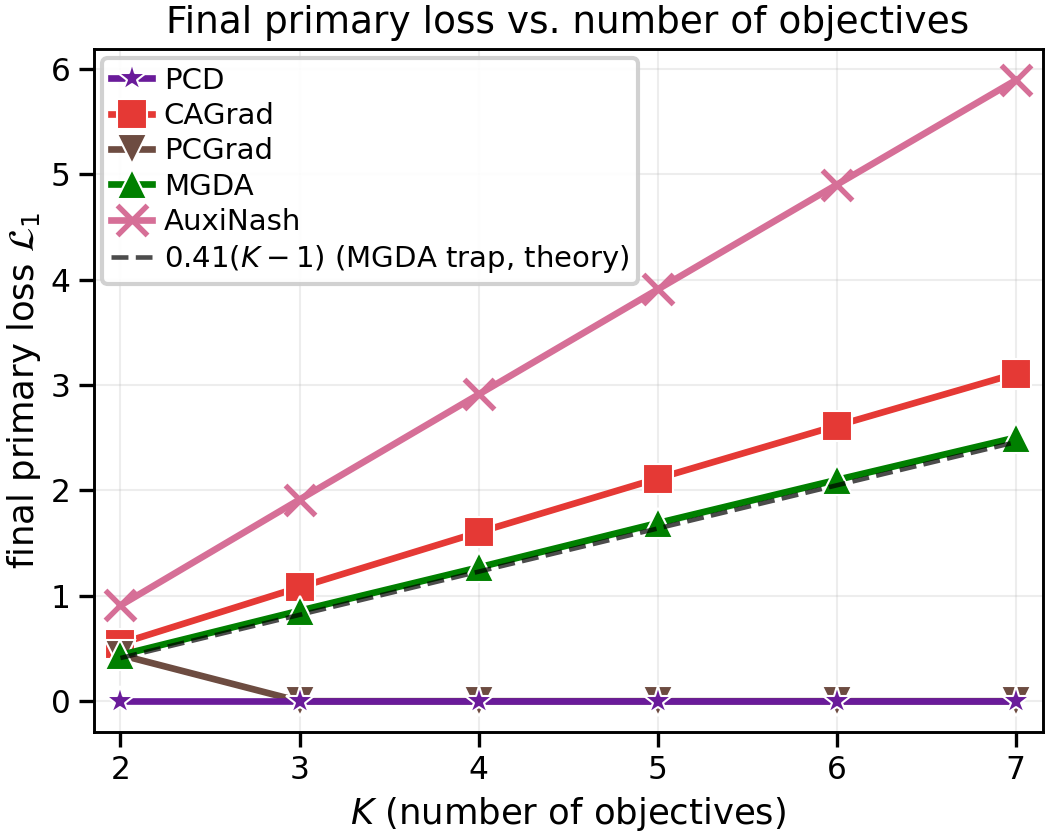}\caption{Primary loss left vs.\ $K$}\end{subfigure}
  \caption{\textbf{Conflict-equilibrium escape as $K$ grows} ($5$ seeds). (a) At $K{=}5$, PCD drives the primary gradient to zero while MGDA/CAGrad remain pinned at $1.536\sqrt{K-1}$, PCGrad crosses after an initial stall and AuxiNash drifts slowly. (b) The canceled primary gradient of the combination methods (MGDA, CAGrad) tracks $1.536\sqrt{K-1}$ (dashed) across all $K$, while PCD sits at zero. (c) The unconverged primary loss grows linearly in $K$ for the trapped methods. PCGrad
  ($K\ge3$) and PCD reach CMS, AuxiNash leaves the most loss. See text for the combination-vs-surgery distinction and the separability caveat.}
  \label{fig:syn-conflict}
\end{figure}

Two qualifications accompany this result. First, PCGrad and AuxiNash are not combination methods and are not bound to the equilibrium: PCGrad traps at $K=2$ but, for $K\ge3$, its pairwise gradient projection lets the orthogonal single-direction secondaries pull through, so it too reaches CMS in this construction. AuxiNash, which learns to down-weight the conflicting auxiliaries, neither sits at the equilibrium nor reaches CMS but drifts into the high-loss barrier region, leaving the largest primary loss of any method. Second, the construction is separable across the $u_j$. This is what makes the $\sqrt{K-1}$ law exactly derivable, and we present it as such. The conflict-equilibrium structure nonetheless requires all $K$ gradients to balance with equal weight. The $\sqrt{K-1}$ law thus
characterizes the combination methods it is derived for, and PCGrad's escape is specific to this clean separable geometry. In the non-separable compression landscapes of Sec.~\ref{sec:structured-pruning}, PCGrad is dominated at every compression target. The conclusion is the one the theory predicts: the fixed points of PCD's normalized direction are exactly the CMS points (Thm.~\ref{thm:cms}, under the oracle hypothesis of Rem.~\ref{rem:cms-oracle}), so PCD is the only method tested whose construction targets the joint optimum rather than reaching it incidentally. Empirically, it reaches it consistently as the number of competing objectives grows.

\section{Conclusion}

Multi-objective optimization in deep learning has long been limited in its applications by an artificial symmetry. By treating primary tasks and secondary constraints as equals, standard MOO methods routinely trap optimization trajectories in conflict equilibria. \emph{Priority-Constrained Descent} (PCD) dismantles this symmetry. By anchoring updates to the primary gradient and admitting only the minimal Euclidean distortion necessary to maintain first-order secondary progress, PCD enforces its constraints on the direction it computes without sacrificing the primary deliverable. Governed by a single, interpretable scalar $\tau$, the framework is asymptotically scale-invariant, admits exact closed-form solutions for practical objective counts, and its normalized direction vanishes only at composite multi-stationary points, with the deployed-update caveats stated in Sec.~\ref{sec:pcd-scope}. As demonstrated across neural network compression (against both general multi-objective methods and tuned pruning-specific baselines) and synthetic experiments, PCD provides what hierarchical pipelines demand: a stable, controllable trade-off curve that consistently outperforms symmetric baselines at matched compression.

The theoretical guarantees of PCD rely on secondary objectives being proper, convex, and lower semi-continuous. Consequently, inherently non-convex constraints such as adversarial losses or hard combinatorial sparsity targets fall outside our formal geometric proofs, which apply locally to the normalized first-order step. Extending the convergence guarantees known for the raw-gradient dynamic-barrier update~\citep{gong2021bi} to the normalized, stochastic update deployed here remains open. Furthermore, while the composite multi-stationarity theorem characterizes the exact fixed-point structure of the deterministic update, the underlying quadratic program could theoretically become infeasible for $K > 2$ if the secondary constraints are strictly and mutually incompatible. Although Sec.~\ref{sec:synthetic} demonstrates this to be improbable, we note it as a possible limitation regardless. While our synthetic experiments probe the feasibility of the constraint polyhedron up to $K = 8$ and compare against traditional MOO baselines for $K \in \{2, \ldots, 7\}$, we note that typical asymmetric optimization applications rarely necessitate such a large number of distinct objectives in practice.

PCD establishes a foundation for hierarchical optimization, opening several immediate avenues for extension. The most direct is the development of \emph{adaptive $\tau$ schedules}, which would dynamically anneal constraint strength based on the training trajectory or current feasibility (as outlined in App.~\ref{app:feasibility}). Beyond the strict primary--secondary dichotomy, extending the formulation to \emph{multi-level priority hierarchies} where an arbitrary number of objectives are strictly ordered with distinct, cascading tolerances offers a natural mathematical next step. More broadly, the per-objective tolerances of Remark~\ref{rem:per-obj-tau} trace a $(K-1)$-dimensional operating manifold within the tolerance hypercube $[0,1]^{K-1}$, of which scalar $\tau$ is the diagonal slice and cascading hierarchies an ordered sub-family. The open challenge is recovering this expressiveness by setting the $\tau_j$ adaptively or via a learned criterion, rather than by exponential grid search over the cube, so that PCD's single-knob interpretability is extended rather than spent. On the application side, deploying PCD under strict, hardware-aware constraints, or adapting it for structured generative modeling where primary output fidelity must survive heavy structural regularizers, represents a highly promising frontier for the framework.

\subsubsection*{Author Contributions}
D. Varam and M. I. AlHajri contributed equally to this work and share joint first authorship. D. Varam was responsible for software development, formal analysis, investigation, data curation, visualization, and original draft preparation. M. I. AlHajri contributed to conceptualization, resources, and project administration. Methodology, validation, review and editing of the manuscript were performed jointly by both authors.

\subsubsection*{Acknowledgments}
This research was supported in part by the High-Performance Computing (HPC) environment at the American University of Sharjah. The authors acknowledge the AUS-HPC team for providing the computational infrastructure and technical support.

\bibliography{main}
\bibliographystyle{tmlr}

\appendix

\section{Roadmap of results}
\label{app:roadmap}

Fig.~\ref{fig:theory-map} provides a clickable map of the paper's theoretical structure: every node is a hyperlink to the corresponding statement, so the figure doubles as a table of contents for the technical material below. Beyond the theory, the experimental appendices are organized as follows: App.~\ref{app:expanded-structured-pruning} expands the structured-pruning results, including the tuned-baseline study's full-range figure and per-configuration tables (App.~\ref{app:tuned-expanded}), the normalization grid (App.~\ref{app:norm-ablation}), the baselines re-run under PCD's normalization (App.~\ref{app:norm-baselines}), and the per-step mechanism measurements (App.~\ref{app:mechanism}). App.~\ref{app:k3-expanded} expands the three-objective results with seed variability, matched-sparsity comparisons, hypervolume robustness, and the failure-mechanism measurements, and App.~\ref{app:optimizer} measures the optimizer-level behavior of the per-step guarantee. The nonsmooth-case analysis added in revision (the kink tax and its closed forms) is in App.~\ref{app:kink}.

\begin{figure*}[h]
\centering
\definecolor{cDef}{RGB}{52,103,152}
\definecolor{cThm}{RGB}{162,49,49}
\definecolor{cProp}{RGB}{196,105,46}
\definecolor{cCor}{RGB}{46,115,68}
\definecolor{cLem}{RGB}{106,55,140}
\definecolor{cAlg}{RGB}{165,130,30}

\resizebox{0.95\textwidth}{!}{%
\begin{tikzpicture}[
  font=\sffamily,
  x=0.95cm, y=0.95cm,
  thmnode/.style={rectangle, draw=cThm, fill=cThm!8, rounded corners=2pt,
                  minimum width=2.1cm, minimum height=0.75cm, text width=2.0cm,
                  align=center, font=\scriptsize, line width=0.6pt, inner sep=1.5pt},
  thmnodebig/.style={thmnode, minimum width=2.4cm, text width=2.3cm,
                  minimum height=0.85cm, line width=1.0pt, font=\scriptsize\bfseries},
  propnode/.style={thmnode, draw=cProp, fill=cProp!8},
  cornode/.style={thmnode, draw=cCor, fill=cCor!8},
  lemnode/.style={thmnode, draw=cLem, fill=cLem!8},
  defnode/.style={thmnode, draw=cDef, fill=cDef!8},
  algnode/.style={thmnode, draw=cAlg, fill=cAlg!15, minimum width=3.0cm,
                  text width=2.85cm, line width=0.9pt, font=\scriptsize\bfseries,
                  minimum height=0.85cm},
  proves/.style={-{Stealth[length=1.4mm,width=1.2mm]}, line width=0.5pt, color=black!65},
  contrast/.style={-{Stealth[length=1.4mm,width=1.2mm]}, line width=0.5pt,
                  color=black!55, dotted, dash pattern=on 0.7pt off 1.4pt},
  group/.style={rectangle, draw=black!15, dashed, rounded corners=3pt,
                inner sep=4pt, line width=0.35pt},
]

\node[defnode] (ass)        at ( 0,  6.8) {\hyperref[ass:reg]{\textbf{Asm~\ref*{ass:reg}}}\\Objective regularity};
\node[defnode] (def-pareto) at ( 3,  6.8) {\hyperref[def:pareto]{\textbf{Def~\ref*{def:pareto}}}\\Pareto stationarity};
\node[defnode] (def-cms)    at ( 6,  6.8) {\hyperref[def:cms]{\textbf{Def~\ref*{def:cms}}}\\CMS};
\node[defnode] (def-qp)     at (12,  6.8) {\hyperref[def:pcd-qp]{\textbf{Def~\ref*{def:pcd-qp}}}\\PCD QP};
\node[defnode] (def-ema)    at (18,  6.8) {\hyperref[def:ema]{\textbf{Def~\ref*{def:ema}}}\\EMA normalization};

\node[thmnode]    (thm-feas)  at ( 3,  5.1) {\hyperref[thm:feas]{\textbf{Thm~\ref*{thm:feas}}}\\Feasibility (Farkas)};
\node[thmnodebig] (thm-kkt)   at (12,  5.1) {\hyperref[thm:kkt]{Thm~\ref*{thm:kkt}}\\Projection $+$ KKT};
\node[thmnode]    (thm-scale) at (18,  5.1) {\hyperref[thm:scale-inv]{\textbf{Thm~\ref*{thm:scale-inv}}}\\Scale invariance};

\node[cornode] (cor-infeas)   at ( 0,  3.4) {\hyperref[cor:infeas]{\textbf{Cor~\ref*{cor:infeas}}}\\When infeasibility};
\node[propnode](prop-Kfeas)   at ( 3,  3.4) {\hyperref[prop:Kfeas]{\textbf{Prop~\ref*{prop:Kfeas}}}\\Angle suff.\ cond.};
\node[lemnode] (lem-proj)     at ( 6,  3.4) {\hyperref[lem:halfspace]{\textbf{Lem~\ref*{lem:halfspace} / Cor~\ref*{cor:k2-halfspace}}}\\Half-space proj.};
\node[cornode] (cor-inactive) at ( 9,  3.4) {\hyperref[cor:inactive]{\textbf{Cor~\ref*{cor:inactive}}}\\Inactive regime};
\node[cornode] (cor-k2)       at (12,  3.4) {\hyperref[cor:k2]{\textbf{Cor~\ref*{cor:k2}}}\\$K{=}2$ closed form};
\node[cornode] (cor-k3)       at (15,  3.4) {\hyperref[cor:k3]{\textbf{Cor~\ref*{cor:k3}}}\\$K{=}3$ closed form};
\node[cornode] (cor-ws)       at (18,  3.4) {\hyperref[corr:ws-not-scale-invariant]{\textbf{Cor~\ref*{corr:ws-not-scale-invariant}}}\\WS \emph{not} scale-inv.};

\node[propnode](prop-prog)    at ( 9,  1.7) {\hyperref[prop:progress]{\textbf{Prop~\ref*{prop:progress}}}\\Per-step sec.\ progress};
\node[propnode](prop-prim)    at (12,  1.7) {\hyperref[prop:primary-preservation]{\textbf{Prop~\ref*{prop:primary-preservation}}}\\Primary preserv.\ ($K{=}2$)};
\node[lemnode] (lem-bdry)     at (15,  1.7) {\hyperref[lem:boundary]{\textbf{Lem~\ref*{lem:boundary}}}\\Boundary cases};
\node[propnode](prop-mgda)    at (18,  1.7) {\hyperref[prop:mgda]{\textbf{Prop~\ref*{prop:mgda}}}\\MGDA failure mode};

\node[thmnodebig] (thm-cms)   at (12,  0.0) {\hyperref[thm:cms]{Thm~\ref*{thm:cms}}\\CMS endpoint ($\tau{>}0$)};
\node[lemnode]    (lem-tau0)  at (16.2, 0.0) {\hyperref[lem:tau-zero]{\textbf{Lem~\ref*{lem:tau-zero}}}\\$\tau{=}0\Rightarrow$ Pareto-stat.};

\node[algnode] (alg) at (12, -1.3) {\hyperref[alg:pcd]{Algorithm~\ref*{alg:pcd}}\\PCD step (full procedure)};

\draw[proves] (def-qp.south) -- (thm-kkt.north);
\draw[proves] (def-ema.south) -- (thm-scale.north);
\draw[proves] (def-qp.south) .. controls (12, 5.8) and (3, 5.8) .. (thm-feas.north);

\draw[proves] (thm-feas.south) to[bend right=15] (cor-infeas.north east);
\draw[proves] (thm-feas.south) -- (prop-Kfeas.north);

\draw[proves] (thm-scale.south) -- (cor-ws.north);

\draw[proves] (thm-kkt.south west) to[bend right=8] (lem-proj.north east);
\draw[proves] (thm-kkt.south west) -- (cor-inactive.north east);
\draw[proves] (thm-kkt.south) -- (cor-k2.north);
\draw[proves] (thm-kkt.south east) -- (cor-k3.north west);

\draw[proves] (thm-kkt.south west) -- (prop-prog.north east);
\draw[proves] (thm-kkt.south east) -- (lem-bdry.north west);

\draw[proves] (cor-k2.south) -- (prop-prim.north);
\draw[proves] (prop-prog.south east) -- (thm-cms.north west);
\draw[proves] (thm-cms.south) -- (alg.north);

\draw[contrast] (prop-mgda.south west) to[out=-150, in=20] (thm-cms.east);

\begin{pgfonlayer}{background}
  \node[group, fit=(ass)(def-ema)] {};
  \node[group, fit=(thm-feas)(thm-scale)] {};
  \node[group, fit=(cor-infeas)(cor-ws)] {};
  \node[group, fit=(prop-prog)(prop-mgda)] {};
  \node[group, fit=(thm-cms)(lem-tau0)] {};
\end{pgfonlayer}

\begin{scope}[shift={(0, -1.3)}]
  \node[defnode,  minimum width=1.4cm, minimum height=0.36cm, text width=1.3cm, font=\tiny, inner sep=0.5pt] at (0.7,  1.7) {Definition};
  \node[thmnode,  minimum width=1.4cm, minimum height=0.36cm, text width=1.3cm, font=\tiny, inner sep=0.5pt] at (2.3,  1.7) {Theorem};
  \node[propnode, minimum width=1.4cm, minimum height=0.36cm, text width=1.3cm, font=\tiny, inner sep=0.5pt] at (3.9,  1.7) {Proposition};
  \node[cornode,  minimum width=1.4cm, minimum height=0.36cm, text width=1.3cm, font=\tiny, inner sep=0.5pt] at (5.5,  1.7) {Corollary};

  \node[lemnode,  minimum width=1.4cm, minimum height=0.36cm, text width=1.3cm, font=\tiny, inner sep=0.5pt] at (0.7,  0.9) {Lemma};
  \node[algnode,  minimum width=1.4cm, minimum height=0.36cm, text width=1.3cm, font=\tiny, inner sep=0.5pt] at (2.3,  0.9) {Algorithm};

  \draw[proves]   (3.4, 0.9) -- (4.2, 0.9);
  \node[font=\tiny, anchor=west] at (4.25, 0.9) {proves/uses};
  \draw[contrast] (3.4, 0.5) -- (4.2, 0.5);
  \node[font=\tiny, anchor=west] at (4.25, 0.5) {contrasts};
\end{scope}

\end{tikzpicture}%
}
\caption{Theory navigation map for PCD.  Foundations (definitions and the regularity assumption) sit at the top. The three core theorems (\hyperref[thm:feas]{Thm~\ref*{thm:feas}}: feasibility via Farkas, \hyperref[thm:kkt]{Thm~\ref*{thm:kkt}}: projection + KKT characterization, \hyperref[thm:scale-inv]{Thm~\ref*{thm:scale-inv}}: scale invariance) occupy the second row.  Their direct corollaries, the per-step guarantees, the asymptotic endpoint (\hyperref[thm:cms]{Thm~\ref*{thm:cms}}), and the $\tau{=}0$ corner case (\hyperref[lem:tau-zero]{Lem~\ref*{lem:tau-zero}}) form the lower tiers, culminating in \hyperref[alg:pcd]{Algorithm~\ref*{alg:pcd}}.  The dotted edge from \hyperref[prop:mgda]{Prop~\ref*{prop:mgda}} marks a contrast: under matched magnitudes MGDA's mixing coefficient collapses, whereas PCD avoids this failure.  Every node is a hyperlink to the corresponding statement in the paper.}
\label{fig:theory-map}
\end{figure*}

\section{Full Details and Proofs for Section~\ref{sec:pcd}}
\label{app:pcd-proofs}

\begin{remark}[Per-objective tolerances]\label{rem:per-obj-tau}
The main text and the appendix use a single scalar tolerance $\tau \in [0,1]$ that applies uniformly to every secondary objective. The framework admits a trivial generalization to per-objective $\tau_j \in [0,1]$, in which the constraint of Eq.~\eqref{eq:pcd-qp} becomes $\gt_j^\T \dt \geq \tau_j \norm{\gt_j}^2$. All proofs and the algorithm carry through verbatim under this substitution. We work with scalar $\tau$ throughout because all experiments and ablations in this paper use a single shared tolerance. Per-objective $\tau_j$ adds expressive flexibility without changing the underlying analysis.
\end{remark}

\subsection{Feasibility Analysis}\label{app:feasibility}

\paragraph{$K=2$ feasibility.} For $K=2$ with $\gt_2 \neq \vzero$, the point $\dt = \tau \gt_2$ satisfies $\gt_2^\T \dt = \tau \norm{\gt_2}^2$, so $F(\tau) = H_2 \neq \emptyset$ for all $\tau \in [0,1]$.

\paragraph{General $K$ via Farkas' lemma.} Define $\widetilde G \in \R^{(K-1)\times n}$ with rows $\gt_2^\T,\ldots,\gt_K^\T$ and $\vb \in \R^{K-1}$ with $b_j = \tau \norm{\gt_j}^2$. By Farkas' lemma, $F(\tau) = \emptyset$ iff $\exists\, \vy \geq \vzero$, $\vy \neq \vzero$, with $\widetilde G^\T \vy = \vzero$ and $\vb^\T \vy > 0$.

\begin{theorem}[Feasibility characterization]\label{thm:feas}
Let $\tau > 0$ and let $N(\vth) = \{ j \in [K]_{-1} : \gt_j \neq \vzero \}$ index the nonzero normalized secondary gradients. Then
\[
F(\tau) = \emptyset \quad\Longleftrightarrow\quad \vzero \in \conv\{\gt_j : j \in N(\vth)\},
\]
with the convention that the hull is empty when $N(\vth) = \emptyset$ (all secondary gradients zero, in which case $F(\tau) = \R^n$). Equivalently, $F(\tau) = \emptyset$ iff there exists $\vy \geq \vzero$, $\vy \neq \vzero$, supported on $N(\vth)$, with $\sum_j y_j \gt_j = \vzero$.
\end{theorem}

\begin{proof}
By the Farkas alternative above, $F(\tau) = \emptyset$ iff $\exists\, \vy \geq \vzero$ with $\widetilde G^\T \vy = \vzero$, i.e.\ $\sum_{j \geq 2} y_j \gt_j = \vzero$, and $\vb^\T\vy = \tau \sum_j y_j \norm{\gt_j}^2 > 0$. Since $\tau > 0$, the last condition holds iff $y_j > 0$ for at least one $j \in N(\vth)$. Coordinates of $\vy$ on indices outside $N(\vth)$ multiply zero vectors and may be dropped, so without loss of generality $\vy$ is supported on $N(\vth)$ and nonzero. Normalizing by $\sum_j y_j > 0$ expresses $\vzero$ as a convex combination of $\{\gt_j : j \in N(\vth)\}$. Conversely, the coefficients of any such convex combination form a valid Farkas certificate.
\end{proof}

\begin{corollary}[When infeasibility occurs]\label{cor:infeas}
\begin{enumerate}[label=(\roman*),nosep,leftmargin=*]
  \item $K=2$: a single nonzero ray cannot positively combine to zero, so $F(\tau) \neq \emptyset$.
  \item $K=3$: $y_2 \gt_2 + y_3 \gt_3 = \vzero$ with $y_2, y_3 \geq 0$ not both zero and $\gt_2, \gt_3 \neq \vzero$ requires $\gt_3 = -\lambda \gt_2$ ($\lambda > 0$), i.e.\ exact anti-parallelism.
  \item General $K$, $n\gg K$: a nontrivial non-negative vanishing combination of $\{\gt_2,\ldots,\gt_K\}$ (equivalently, $\vzero \in \conv$ of the nonzero $\gt_j$) defines a variety of positive codimension in $(\R^n)^{K-1}$. Infeasibility is Lebesgue measure-zero.
\end{enumerate}
\end{corollary}

\subsection{KKT Proof and Active-Set Solver}\label{app:active-set}

\begin{proof}[Proof of Theorem~\ref{thm:kkt}]
The Lagrangian of Eq.~\eqref{eq:pcd-qp} is
\[
\mathcal{L}(\dt,\boldsymbol{\mu}) = \tfrac{1}{2}\norm{\dt - \gt_1}^2 - \sum_{j=2}^K \mu_j(\gt_j^\T\dt - \tau\norm{\gt_j}^2),\quad \mu_j \geq 0.
\]
Setting $\nabla_{\dt}\mathcal{L} = \vzero$ gives $\dt - \gt_1 - \sum_j \mu_j \gt_j = \vzero$, i.e.\ stationarity (i). Primal feasibility (ii), dual feasibility (iii), and complementary slackness (iv) are the standard KKT conditions for inequality-constrained convex QPs with affine constraints. For such problems KKT is necessary and sufficient at every primal optimum without further constraint qualification.
\end{proof}

\paragraph{Active-set linear system.} For active set $\mathcal{A}\subseteq [K]_{-1}$, let $G_\mathcal{A} \in \R^{|\mathcal{A}|\times|\mathcal{A}|}$ be the Gram matrix with $(G_\mathcal{A})_{jk} = \gt_j^\T\gt_k$, and $\vb_\mathcal{A} \in \R^{|\mathcal{A}|}$ with $(\vb_\mathcal{A})_j = \tau\norm{\gt_j}^2 - \gt_j^\T\gt_1$. Substituting $\dts = \gt_1 + \sum_{k\in\mathcal{A}}\mu_k^*\gt_k$ into the active equality constraints $\gt_j^\T\dts = \tau\norm{\gt_j}^2$, $j\in\mathcal{A}$:
\begin{equation}\label{eq:active-system}
G_\mathcal{A}\, \boldsymbol{\mu}_\mathcal{A} = \vb_\mathcal{A}.
\end{equation}
$G_\mathcal{A}$ is positive definite whenever $\{\gt_j\}_{j\in\mathcal{A}}$ are linearly independent (generic in $\R^n$ for $|\mathcal{A}|\leq n$).

\paragraph{Solver: exhaustive KKT enumeration.}
The solver used for every $K=3$ experiment in Sec.~\ref{sec:k3-experiments} does not iterate an active-set method and instead enumerates working sets and tests the KKT conditions of Theorem~\ref{thm:kkt} directly. For a working set $\mathcal{A}\subseteq[K]_{-1}$ with $\{\gt_j\}_{j\in\mathcal{A}}$ linearly independent, $G_\mathcal{A}$ is positive definite and Eq.~\eqref{eq:active-system} has the unique solution $\boldsymbol{\mu}_\mathcal{A}=G_\mathcal{A}^{-1}\vb_\mathcal{A}$. We write $\dt_\mathcal{A}=\gt_1+\sum_{k\in\mathcal{A}}(\boldsymbol{\mu}_\mathcal{A})_k\,\gt_k$, and call $\mathcal{A}$ \emph{KKT-valid} if $\boldsymbol{\mu}_\mathcal{A}\geq\vzero$ and $\gt_j^\T\dt_\mathcal{A}\geq\tau\norm{\gt_j}^2$ for every $j\notin\mathcal{A}$ (the constraints in $\mathcal{A}$ hold with equality by construction). The solver visits the linearly independent working sets in order of increasing size and returns $\dt_\mathcal{A}$ for the first KKT-valid one. If no working set is KKT-valid it returns $\dt=\gt_1$ (deployment rule below).

\begin{proposition}[Correctness of the enumeration]\label{prop:enum}
Let $\tau\geq 0$.
\begin{enumerate}[label=(\alph*)]
\item If $F(\tau)\neq\emptyset$, the minimizer $\dts$ of Eq.~\eqref{eq:pcd-qp} exists, is unique, and satisfies the KKT conditions of Theorem~\ref{thm:kkt} for some $\boldsymbol{\mu}\geq\vzero$. No constraint qualification is required.
\item Some KKT multiplier is supported on a set $S$ of active constraints with $\{\gt_j\}_{j\in S}$ linearly independent, and this $S$ is KKT-valid.
\item Every KKT-valid working set $\mathcal{A}$ satisfies $\dt_\mathcal{A}=\dts$. Hence the first KKT-valid set found, in any enumeration order, yields the optimizer.
\item No working set is KKT-valid if and only if $F(\tau)=\emptyset$. For $\tau>0$ Theorem~\ref{thm:feas} then supplies the certificate $\vzero\in\conv\{\gt_j : j\in N(\vth)\}$.
\item The enumeration performs at most $2^{K-1}$ solves of linear systems of size at most $K-1$ after one $O(K^2 n)$ Gram-matrix computation.
\end{enumerate}
\end{proposition}

\begin{proof}
(a) The objective $\tfrac12\norm{\dt-\gt_1}^2$ is strongly convex and $F(\tau)$ is a nonempty closed convex polyhedron, so $\dts=\mathrm{proj}_{F(\tau)}(\gt_1)$ exists, is unique, and is characterized by the projection inequality $(\gt_1-\dts)^\T(\dt-\dts)\leq 0$ for all $\dt\in F(\tau)$, i.e.\ $\gt_1-\dts$ lies in the normal cone of $F(\tau)$ at $\dts$. For the polyhedron $\{\dt : \gt_j^\T\dt\geq\tau\norm{\gt_j}^2\}$ that normal cone is $\{-\sum_{j\in A(\dts)}\mu_j\gt_j : \mu_j\geq 0\}$, where $A(\dts)=\{j : \gt_j^\T\dts=\tau\norm{\gt_j}^2\}$. This is Farkas' lemma applied to the tangent cone $\{\vw : \gt_j^\T\vw\geq 0,\ j\in A(\dts)\}$ and uses only that the constraints are affine. Hence $\dts=\gt_1+\sum_{j\in A(\dts)}\mu_j\gt_j$ with $\boldsymbol{\mu}\geq\vzero$ supported on the active set, which together with primal feasibility is the KKT system. 

(b) Let $\vv=\dts-\gt_1=\sum_{j\in A(\dts)}\mu_j\gt_j$ with $\boldsymbol{\mu}\geq\vzero$, and among all such representations take one with the fewest strictly positive coefficients, supported on $S$. If $\{\gt_j\}_{j\in S}$ were linearly dependent, pick $\boldsymbol{\alpha}\neq\vzero$ with $\sum_{j\in S}\alpha_j\gt_j=\vzero$ and (replacing $\boldsymbol{\alpha}$ by $-\boldsymbol{\alpha}$ if necessary) some $\alpha_j>0$. With $t=\min_{\alpha_j>0}\mu_j/\alpha_j$, the coefficients $\boldsymbol{\mu}-t\boldsymbol{\alpha}\geq\vzero$ represent $\vv$ with strictly fewer positive entries, a contradiction (conic Carath\'eodory). For this $S$, $G_S$ is positive definite, and substituting $\dts=\gt_1+\sum_{k\in S}\mu_k\gt_k$ into the active equalities $\gt_j^\T\dts=\tau\norm{\gt_j}^2$, $j\in S$, gives $G_S\boldsymbol{\mu}_S=\vb_S$. Thus $\boldsymbol{\mu}_S=G_S^{-1}\vb_S\geq\vzero$, $\dt_S=\dts$ satisfies every constraint, and $S$ is KKT-valid.

(c) Let $\mathcal{A}$ be KKT-valid. For every $\dt\in F(\tau)$,
\[
(\gt_1-\dt_\mathcal{A})^\T(\dt-\dt_\mathcal{A})
= -\sum_{j\in\mathcal{A}}(\boldsymbol{\mu}_\mathcal{A})_j\big(\gt_j^\T\dt-\tau\norm{\gt_j}^2\big)\leq 0,
\]
using $\gt_j^\T\dt_\mathcal{A}=\tau\norm{\gt_j}^2$ on $\mathcal{A}$, $\boldsymbol{\mu}_\mathcal{A}\geq\vzero$, and feasibility of $\dt$. Since $\dt_\mathcal{A}\in F(\tau)$, this is the projection inequality, so $\dt_\mathcal{A}=\mathrm{proj}_{F(\tau)}(\gt_1)=\dts$ by uniqueness. In degenerate configurations, several working sets are KKT-valid with different multipliers, but the direction is the same.

(d) A KKT-valid set produces the feasible point $\dt_\mathcal{A}$, so $F(\tau)=\emptyset$ implies that no set is KKT-valid. Conversely, if $F(\tau)\neq\emptyset$, (a)--(b) exhibit one. The certificate is Theorem~\ref{thm:feas}.

(e) There are $2^{K-1}$ subsets of $[K]_{-1}$. Each costs an $O(|\mathcal{A}|^3)$ solve plus an $O(K^2)$ feasibility check on the $(K-1)\times(K-1)$ Gram matrix, which is built once in $O(K^2 n)$.
\end{proof}

For $K=3$ the enumeration visits $\emptyset$, $\{2\}$, $\{3\}$, $\{2,3\}$: the singletons are Cor.~\ref{cor:k2} applied to one constraint followed by a check of the other, and $\{2,3\}$ is Cor.~\ref{cor:k3}, whose hypothesis (non-collinear $\gt_2,\gt_3$) is exactly the linear-independence requirement above. Parallel secondaries are resolved by a singleton (the more demanding of the two constraints). Anti-parallel secondaries are infeasible (Cor.~\ref{cor:infeas}(ii)).

\paragraph{Deployment rule for $F(\tau)=\emptyset$.}
When no working set is KKT-valid the implementation sets $\boldsymbol{\mu}=\vzero$, i.e.\ $\dt=\gt_1$: the secondary constraints are dropped for that step and PCD takes a pure primary step. Two alternatives are to drop only the constraint carrying the largest weight $y_j$ in the Farkas certificate of Theorem~\ref{thm:feas} and re-solve (at most $K-2$ drops, since a single nonzero constraint is always feasible), or to relax the violated constraints in least squares. We state this rule because it is the one that produced every number in Sec.~\ref{sec:k3-experiments}. It never triggered there: for $K=3$ infeasibility requires exactly anti-parallel $\gt_2,\gt_3$ (Cor.~\ref{cor:infeas}(ii)), whereas the archived multipliers imply $\cos(\gt_{\ell_1},\gt_{\mathrm{nuc}})\in[0.36,0.93]$ over all 180 runs (both constraints bind, so $\boldsymbol{\mu}\approx\tau(1+\cos)^{-1}(1,1)$ in the EMA steady state).

\paragraph{General $K$.}
For $K$ large enough that $2^{K-1}$ solves matter, the enumeration is replaced by the dual active-set method of \citet{goldfarb1983numerically}: it starts from the unconstrained minimizer $\gt_1$, adds violated constraints while maintaining dual feasibility, increases the dual objective strictly at every full step, terminates finitely, and detects infeasibility. Per-step cost is $O(K^2 n)$ for the Gram matrix plus at most $2^{K-1}$ solves of size at most $K-1$, negligible against the $K$ backward passes, which cost $O(KnB)$ for batch size $B$. The ratio is $O(K/B)$, consistent with the ${\leq}2\%$ wall-clock overhead measured at $K=2$, $B=128$ in Sec.~\ref{sec:k3-experiments}.

\subsection{\texorpdfstring{$K{=}2$ and $K{=}3$}{K=2 and K=3} Closed Forms}\label{app:k3-cf}

\begin{proof}[Proof of Cor.~\ref{cor:k2}]
For $K=2$ the active set is either $\emptyset$ (if $\gt_2^\T\gt_1 \geq \tau\norm{\gt_2}^2$, KKT gives $\dts=\gt_1$) or $\{2\}$ (otherwise). In the latter case, Eq.~\eqref{eq:active-system} reduces to the scalar equation $\norm{\gt_2}^2\mu_2^* = \tau\norm{\gt_2}^2 - \gt_2^\T\gt_1$, giving $\mu_2^* = \tau - \gt_2^\T\gt_1/\norm{\gt_2}^2$.
\end{proof}

\begin{corollary}[Closed form for $K=3$, both constraints active]\label{cor:k3}
For $K=3$ with $\mathcal{A}=\{2,3\}$, the $2\times 2$ Gram system gives
\begin{align*}
\mu_2^* &= \tfrac{\norm{\gt_3}^2(\tau\norm{\gt_2}^2 - \gt_2^\T\gt_1) - (\gt_2^\T\gt_3)(\tau\norm{\gt_3}^2 - \gt_3^\T\gt_1)}{\Delta},\\
\mu_3^* &= \tfrac{\norm{\gt_2}^2(\tau\norm{\gt_3}^2 - \gt_3^\T\gt_1) - (\gt_2^\T\gt_3)(\tau\norm{\gt_2}^2 - \gt_2^\T\gt_1)}{\Delta},
\end{align*}
with $\Delta = \norm{\gt_2}^2\norm{\gt_3}^2 - (\gt_2^\T\gt_3)^2 = \norm{\gt_2}^2\norm{\gt_3}^2 \sin^2\theta_{23}$. Uniquely solvable iff $\gt_2,\gt_3$ are not collinear.
\end{corollary}

\begin{proof}
Cramer's rule applied to $G_\mathcal{A}\boldsymbol{\mu}_\mathcal{A} = \vb_\mathcal{A}$ with $G_\mathcal{A}=\begin{psmallmatrix}\norm{\gt_2}^2 & \gt_2^\T\gt_3 \\ \gt_2^\T\gt_3 & \norm{\gt_3}^2\end{psmallmatrix}$.
\end{proof}

\subsection{Geometric Interpretation: Half-Space Projection}\label{app:proj}

\begin{lemma}[Projection onto a half-space]\label{lem:halfspace}
For $H = \{\vx \in \R^n : \va^\T\vx \geq c\}$ with $\va \neq \vzero$, the orthogonal projection of $\vp$ onto $H$ is
\[
\mathrm{proj}_H(\vp) = \begin{cases} \vp & \text{if } \va^\T\vp \geq c, \\
\vp + \dfrac{c - \va^\T\vp}{\norm{\va}^2}\,\va & \text{otherwise}.\end{cases}
\]
\end{lemma}

\begin{proof}
If $\vp \in H$, the closest point is $\vp$ itself. If $\vp \notin H$, the projection lies on the boundary hyperplane $\{\vx:\va^\T\vx = c\}$ at $\vp + \tfrac{c-\va^\T\vp}{\norm{\va}^2}\va$.
\end{proof}

\begin{corollary}[\texorpdfstring{$K=2$}{K=2} as half-space projection]\label{cor:k2-halfspace}
$\dts = \mathrm{proj}_{H_2}(\gt_1)$, the orthogonal projection of $\gt_1$ onto $H_2 = \{\dt : \gt_2^\T\dt \geq \tau\norm{\gt_2}^2\}$.
\end{corollary}

For general $K$, $\dts = \mathrm{proj}_{F(\tau)}(\gt_1)$, the orthogonal projection of $\gt_1$ onto the convex polyhedron $F(\tau)$. The active set $\mathcal{A}$ identifies the faces of $F(\tau)$ containing the projection.

\subsection{Boundary Cases}\label{app:boundary}

\begin{lemma}[Boundary behaviors]\label{lem:boundary}
\begin{enumerate}[label=(\roman*),nosep,leftmargin=*]
    \hfill
  \item All secondary gradients zero ($\gt_j=\vzero$ for all $j\geq 2$): all constraints reduce to $0\geq 0$, the QP is unconstrained, and $\dts = \gt_1$ (gradient descent on $\mathcal{L}_1$).
  \item Primary gradient zero, with $\tau > 0$ and some $\gt_j \neq \vzero$ ($\gt_1 = \vzero$): the QP becomes $\min \tfrac{1}{2}\norm{\dt}^2$ subject to the secondary constraints. The solution $\dts = \sum_{j\in\mathcal{A}}\mu_j^*\gt_j$ is a non-negative combination of secondary gradients—the PCD direction continues to make first-order progress on secondary objectives even though the primary is stationary.
  \item All gradients zero: $\dts = \vzero$ (fixed point).
\end{enumerate}
\end{lemma}

\begin{proof}
(i) Trivial. (ii) With $\gt_1=\vzero$, KKT stationarity gives $\dts = \sum_j \mu_j^* \gt_j$. The objective $\tfrac{1}{2}\norm{\dt - \vzero}^2 = \tfrac{1}{2}\norm{\dt}^2$ is minimized at the closest feasible point to the origin. Since $\tau>0$ and $\gt_j \neq \vzero$ for some $j$, the constraint $\gt_j^\T\dt \geq \tau\norm{\gt_j}^2 > 0$ forces $\dts \neq \vzero$. (iii) Immediate.
\end{proof}

Lemma~\ref{lem:boundary}(ii) is the geometric mechanism by which the PCD direction avoids conflict equilibria: at near-primary-stationary points, MGDA outputs $\vd_{\mathrm{MGDA}}\approx\vzero$ (since $\vzero \in \conv\{\vzero,\vg_2,\ldots,\vg_K\}$), halting all progress, while $\dts$ continues to drive secondary objectives toward stationarity. The deployed rescale of Algorithm~\ref{alg:pcd} re-introduces a halt at exactly $\vg_1 = \vzero$ (Sec.~\ref{sec:pcd-scope}).

\subsection{Scale Invariance Proof}\label{app:scale}

We formally state the scale-invariance result whose paragraph-level summary appears in Sec.~\ref{sec:pcd-k2}.

\begin{theorem}[Asymptotic scale invariance]\label{thm:scale-inv}
Let $c_1, \ldots, c_K > 0$ and define $\mathcal{L}_i'(\vth) = c_i \mathcal{L}_i(\vth)$. As $\varepsilon \to 0$ (or once $\hat v_{i,t} \gg \varepsilon/c_i^2$ for all $i$), the PCD direction computed from $(\mathcal{L}_1', \ldots, \mathcal{L}_K')$ coincides with that computed from $(\mathcal{L}_1, \ldots, \mathcal{L}_K)$. With the alternative normalization $s_{i,t} = (\hat v_{i,t}(1+\varepsilon))^{-1/2}$, scale invariance is exact for all $\varepsilon \geq 0$.
\end{theorem}

\begin{proof}
Under $\mathcal{L}_i \to c_i \mathcal{L}_i$, $\vg_i \to c_i \vg_i$. The EMA buffer transforms as $v_{i,t} \to c_i^2 v_{i,t}$, so $\hat v_{i,t} \to c_i^2 \hat v_{i,t}$ and
\[
s_{i,t} = (\hat v_{i,t} + \varepsilon)^{-1/2} \;\to\; (c_i^2 \hat v_{i,t} + \varepsilon)^{-1/2}.
\]
For $\varepsilon \ll c_i^2 \hat v_{i,t}$ (which holds once the EMA stabilizes), $s_{i,t} \approx 1/(c_i\sqrt{\hat v_{i,t}}) = s_{i,t}^{\mathrm{orig}}/c_i$. Therefore
\[
\gt_i' = s_{i,t}'\,\vg_i' \;\approx\; \tfrac{s_{i,t}^{\mathrm{orig}}}{c_i}\cdot c_i\vg_i \;=\; s_{i,t}^{\mathrm{orig}}\vg_i \;=\; \gt_i.
\]
Since $\gt_i' = \gt_i$ for all $i$, the QP in Eq.~\eqref{eq:pcd-qp} is identical, and $\dts$ is unchanged. For the alternative normalization, the factor $(1+\varepsilon)$ is invariant under rescaling, so $s_{i,t}'/s_{i,t} = 1/c_i$ exactly, and $\gt_i' = \gt_i$ holds for every $\varepsilon \geq 0$.
\end{proof}

\begin{corollary}[Weighted-sum methods are not scale invariant]
\label{corr:ws-not-scale-invariant}
$\vd_{\mathrm{WS}} = \sum_i w_i \vg_i$ transforms to $\sum_i w_i c_i \vg_i$, which generally differs from $\vd_{\mathrm{WS}}$. The WS direction depends on the relative magnitudes of the objectives: an arbitrary, task-dependent quantity.
\end{corollary}

\begin{remark}[Deployed direction and the $\varepsilon$ floor]\label{rem:scale-deployed}
Two practical qualifications must be stated. 
\emph{(i)} The deployed update $\vd^* = \dts \cdot \norm{\vg_1}/\norm{\dts}$ inherits invariance in direction under rescaling of any objective, and in magnitude under rescaling of the secondaries. A primary rescale $\mathcal{L}_1 \to c_1 \mathcal{L}_1$ scales $\norm{\vd^*}$ by exactly $c_1$. This is a learning-rate re-parametrization that adaptive optimizers such as Adam and RMSProp absorb. 

\emph{(ii)} The finite-$\varepsilon$ residual is a floor: invariance degrades once gradient norms approach $\sqrt{\varepsilon}$ ($10^{-4}$ at the deployed $\varepsilon = 10^{-8}$). In a $200$-step simulation with the deployed EMA, rescaling the secondary by $10^{-2}$ rotates $\dts$ by $66^\circ$ at gradient norms $\sim 10^{-4}$, and by less than $0.01^\circ$ whenever $\norm{\vg} \gg \sqrt{\varepsilon}$. Under the exact normalizer of Theorem~\ref{thm:scale-inv} the rotation is $0^\circ$ in both regimes.
\end{remark}

\subsection{Per-Step Progress and MGDA Failure}\label{app:per-step}

\begin{proof}[Proof of Prop.~\ref{prop:progress}(i)]
Immediate from primal feasibility (Theorem~\ref{thm:kkt}(ii)) with $\tau > 0$ and $\norm{\gt_j} > 0$. Parts (ii) and (iii) are proven in the main text. Closed forms of the kink tax $\Delta_j$ for the secondaries used in this paper are derived in App.~\ref{app:kink}.
\end{proof}

\begin{proposition}[MGDA failure mode]\label{prop:mgda}
For $K\geq 2$, $\vd_{\mathrm{MGDA}} = \arg\min_{\vd \in \conv\{\vg_1,\ldots,\vg_K\}}\norm{\vd}^2$. If $\vg_1 = \vzero$, then $\vd_{\mathrm{MGDA}} = \vzero$, regardless of $\vg_2,\ldots,\vg_K$.
\end{proposition}

\begin{proof}
$\vzero = \vg_1 \in \conv\{\vg_1,\ldots,\vg_K\}$, and $\norm{\vzero}^2 = 0$ is the global minimum. Hence $\vd_{\mathrm{MGDA}}=\vzero$.
\end{proof}

This makes precise the structural advantage of the PCD direction over MGDA: MGDA halts whenever \emph{any} single gradient vanishes (even if $K-1$ secondary objectives remain non-stationary), whereas $\dts$ with $\tau > 0$ continues to drive every non-stationary secondary toward stationarity (Lemma~\ref{lem:boundary}(ii)). As discussed in Sec.~\ref{sec:pcd-scope}, the deployed update inherits this advantage only away from $\vg_1 = \vzero$, where the magnitude rescale vanishes.

\subsection{The Nonsmooth Case: the Kink Tax and its Closed Forms}\label{app:kink}

This appendix quantifies the gap between the oracle-linearized progress $\vg_j^\T\vd$ that the QP controls and the exact first-order change of a nonsmooth secondary, completing Proposition~\ref{prop:progress}(iii). The quantity itself is classical: it is the difference between the linearization at the oracle's subgradient and the one-sided directional derivative $\mathcal{L}_j'(\vth; -\vd) = -\min_{\vz\in\partial\mathcal{L}_j(\vth)}\vz^\T\vd$ of convex analysis \citep[Thm.~23.4]{rockafellar1970convex}; \emph{kink tax} is the name this paper attaches to it, not established terminology. Throughout, $\vth \in \operatorname{int}\operatorname{dom}(\mathcal{L}_j)$, $\vg_j \in \partial\mathcal{L}_j(\vth)$ is the oracle's subgradient, and $\Delta_j$ is the kink tax of Eq.~\eqref{eq:kink-tax}.

\paragraph{Why a single subgradient cannot certify descent.} For convex $\mathcal{L}_j$, the subgradient inequality gives $\mathcal{L}_j(\vth - \eta\vd) \geq \mathcal{L}_j(\vth) - \eta\,\vz^\T\vd$ for every $\vz \in \partial\mathcal{L}_j(\vth)$ --- a \emph{lower} bound. Convexity runs the wrong way: no condition on a single subgradient bounds the loss change from above. The correct expansion is through the directional derivative \citep[Thm.~23.4]{rockafellar1970convex}, $\mathcal{L}_j(\vth - \eta\vd) = \mathcal{L}_j(\vth) - \eta \min_{\vz\in\partial\mathcal{L}_j(\vth)}\vz^\T\vd + o(\eta)$, which is Proposition~\ref{prop:progress}(iii) after substituting the kink tax. We note that the hypothesis of a \emph{Lipschitz continuous subgradient selection}, sometimes invoked to bypass this issue, is not a weakening of smoothness: a continuous selection from $\partial\mathcal{L}_j$ on an open set forces the subdifferential to be a singleton there, i.e.\ $\mathcal{L}_j \in C^1$.

\begin{example}[Failure of the smooth-case bound under a standard oracle]\label{ex:kink}
Take $K = 2$, $s_1 = s_2 = 1$, and $\mathcal{L}_2(\vth) = \theta_1 + \theta_1^2 + 3|\theta_2|$: proper, convex, coercive, with global minimum at $(-\tfrac12, 0)$. At $\vth = \vzero$, $\partial\mathcal{L}_2(\vzero) = \{1\}\times[-3,3]$, whose least-norm element is $(1,0)$. This is also what reverse-mode differentiation returns, by the $|\cdot|$-at-zero convention. With $\vg_1 = (0.1, -1)$, $\vg_2 = (1,0)$, and $\tau = \tfrac12$: $\vg_2^\T\vg_1 = 0.1 < \tau\norm{\vg_2}^2 = \tfrac12$, so the constraint is active. Cor.~\ref{cor:k2} gives $\mu^* = 0.4$ and $\dts = (0.5, -1)$, with feasibility at equality, $\vg_2^\T\dts = \tfrac12$. The smooth-case bound Eq.~\eqref{eq:taylor-progress} would promise $\mathcal{L}_2(\vth^+) \leq \mathcal{L}_2(\vth) - \tfrac{\eta}{2} + O(\eta^2)$. The exact value is
\[
\mathcal{L}_2(\vth - \eta\dts) \;=\; -\tfrac{\eta}{2} + \tfrac{\eta^2}{4} + 3\eta \;=\; \tfrac{5\eta}{2} + \tfrac{\eta^2}{4},
\]
strictly \emph{increasing} for every $\eta > 0$, while $\vzero \notin \partial\mathcal{L}_2(\vzero)$, so $\vth$ is not stationary for $\mathcal{L}_2$. The kink tax accounts for the discrepancy exactly: $\min_{\vz\in\partial\mathcal{L}_2(\vzero)} \vz^\T\dts = \tfrac12 - 3 = -\tfrac52$, so $\Delta_2(\dts) = 3$, and Proposition~\ref{prop:progress}(iii) reads $\mathcal{L}_2(\vth^+) \leq \mathcal{L}_2(\vth) + \tfrac{5\eta}{2} + o(\eta)$, which is tight at first order.
\end{example}

\paragraph{Closed forms.} For the secondaries used in this paper, with the oracle conventions of standard reverse-mode differentiation:
\begin{itemize}[nosep,leftmargin=*]
\item \emph{$\ell_1$ penalty}, $\mathcal{L}(\vth) = \norm{\vth}_1$, oracle $\vg = \operatorname{sign}(\vth)$ (zero at zeros): $\Delta(\vd) = \sum_{i:\,\theta_i = 0} |d_i|$.
\item \emph{Group lasso}, $\mathcal{L}(\vth) = \sum_g \norm{\vth_g}$, oracle $\vth_g/\norm{\vth_g}$ on nonzero groups and $\vzero$ on zero groups: $\Delta(\vd) = \sum_{g:\,\vth_g = \vzero} \norm{\vd_g}$.
\item \emph{Nuclear norm}, $\mathcal{L}(M) = \norm{M}_*$ at $M$ with compact SVD $U_r \Sigma_r V_r^\T$: reverse-mode differentiation through the singular values returns $U_r V_r^\T + P$, where $P = \sum_{k > r} \vu_k \vv_k^\T$ runs over the zero-singular-value pairs of the factorization. Then
\[
\Delta(D) \;=\; \langle P,\, D\rangle \;+\; \norm{U_\perp^\T D\, V_\perp}_* ,
\]
with $U_\perp, V_\perp$ orthonormal bases of the orthogonal complements of $\operatorname{range}(U_r)$ and $\operatorname{range}(V_r)$. In particular, at $M = \vzero$ this oracle returns a subgradient of Frobenius norm $\sqrt{\min(m,n)}$, not $\vzero$ (cf.\ Rem.~\ref{rem:cms-oracle}).
\end{itemize}
Derivations: for $\ell_1$, $\partial\mathcal{L} = \prod_i S_i$ with $S_i = \{\operatorname{sign}(\theta_i)\}$ at $\theta_i \neq 0$ and $[-1,1]$ at $\theta_i = 0$, so $\min_{\vz} \vz^\T\vd = \sum_{\theta_i \neq 0}\operatorname{sign}(\theta_i)\, d_i - \sum_{\theta_i = 0}|d_i|$. Subtracting from $\vg^\T\vd$ leaves the zero-coordinate terms. The group-lasso case is identical with the unit ball of $\norm{\cdot}$ on each zero group. For the nuclear norm, $\partial\norm{M}_* = \{U_r V_r^\T + W : \norm{W}_2 \leq 1,\ U_r^\T W = 0,\ W V_r = 0\}$ \citep{watson1992characterization}, and $\min_W \langle W, D\rangle = -\norm{U_\perp^\T D V_\perp}_*$ by spectral--nuclear duality on the orthogonal block. In each case $\Delta$ vanishes unless the update moves mass into currently-zeroed coordinates, groups, or singular directions.

\begin{remark}[The worst-case constraint stays convex]\label{rem:worst-case-qp}
The map $\dt \mapsto \min_{\vz\in\partial\mathcal{L}_j(\vth)} \vz^\T\dt$ is concave (an infimum of linear functions), so its superlevel sets are convex, and a conservative variant of Eq.~\eqref{eq:pcd-qp} that imposes the worst-case constraint $s_j \min_{\vz} \vz^\T\dt \geq \tau\norm{\gt_j}^2$ remains a convex program, with the closed forms above available for its constraints. We do not deploy it, we record only that the projection formulation is not tied to the linearization.
\end{remark}

\begin{remark}[Vanishing descent window under group sparsification]\label{rem:descent-window}
The smooth-case guarantee of Proposition~\ref{prop:progress}(ii) applies on the region where $\mathcal{L}_j$ is smooth, with a step window governed by the local smoothness constant: away from zero groups, the group-lasso Hessian is $(\mathrm{I} - \vth_g\vth_g^\T/\norm{\vth_g}^2)/\norm{\vth_g}$, of spectral norm $1/\norm{\vth_g}$, so the constant grows without bound exactly as groups shrink. The admissible step window collapses precisely as the method succeeds at sparsifying, and the guarantee degrades to the kink-tax form of Proposition~\ref{prop:progress}(iii) in the sparse limit. We state this as a limitation of the smooth-case reading.
\end{remark}

\section{Additional Convergence Properties}
\label{app:additional-conv}

\subsection{Proof of Theorem~\ref{thm:cms}}
\label{app:cms-proof}

We restate Theorem~\ref{thm:cms} for convenience and give the full proof.

\begin{theorem*}[Theorem~\ref{thm:cms}, restated]
Let $\tau > 0$ and suppose $F_\tau(\vth) \neq \emptyset$. Then
\[
  \dts(\vth) \;=\; \vzero \quad\Longleftrightarrow\quad \gt_i(\vth) \;=\; \vzero \;\;\text{for all } i \in [K].
\]
\end{theorem*}

\begin{proof}
($\Leftarrow$) If $\gt_i = \vzero$ for all $i \in [K]$, then $\dt = \vzero$ achieves $\norm{\vzero - \vzero}^2 = 0$ (the global minimum of the QP objective) and trivially satisfies every constraint $\vzero^\T \vzero = 0 \geq 0$. Hence $\dts = \vzero$.

($\Rightarrow$) Suppose $\dts = \vzero$. By primal feasibility (Theorem~\ref{thm:kkt}),
\[
  \gt_j^\T \dts \;\geq\; \tau \norm{\gt_j}^2 \quad \forall\, j \in [K]_{-1}.
\]
Substituting $\dts = \vzero$ gives $0 \geq \tau \norm{\gt_j}^2$. Since $\tau > 0$ and $\norm{\gt_j}^2 \geq 0$, this forces $\gt_j = \vzero$ for all $j \in [K]_{-1}$. By KKT stationarity (Theorem~\ref{thm:kkt}, Eq.~\eqref{eq:kkt-stationarity}),
\[
  \vzero \;=\; \dts \;=\; \gt_1 + \sum_{j=2}^K \mu_j^* \, \gt_j \;=\; \gt_1 + \sum_{j=2}^K \mu_j^* \, \vzero \;=\; \gt_1 ,
\]
so $\gt_1 = \vzero$. Since $\gt_i = s_i \vg_i$ with $s_i > 0$, this yields $\nabla \mathcal{L}_1(\vth) = \vzero$ and $\vzero \in \partial \mathcal{L}_j(\vth)$ for $j \in [K]_{-1}$, i.e.\ $\vth$ is composite multi-stationary.
\end{proof}

The proof uses only primal feasibility with $\tau > 0$ and KKT stationarity. It does not depend on $K$, on $n$, or on algebraic structure of the multipliers or Gram matrix, which is what makes the result hold uniformly across problem sizes.

\begin{remark}[Oracle dependence of the CMS identification]\label{rem:cms-oracle}
Both sides of the displayed equivalence concern the oracle outputs $\gt_i$. The implication ``$\gt_i = \vzero$ for all $i$ $\Rightarrow$ $\vth$ is CMS'' is unconditional, since each returned $\vg_j$ is an element of $\partial\mathcal{L}_j(\vth)$. The converse identification requires the oracle to return $\vzero$ whenever $\vzero \in \partial\mathcal{L}_j(\vth)$: an oracle may legitimately return a nonzero subgradient at a point where $\vzero \in \partial \mathcal{L}_j(\vth)$, in which case $\vth$ is CMS while $\dts \neq \vzero$. Reverse-mode differentiation through the singular values returns, at the zero $m \times n$ matrix, a subgradient of the nuclear norm of Frobenius norm $\sqrt{\min(m,n)}$ rather than $\vzero$ (App.~\ref{app:kink}). We note that this is harmless for the $\ell_1$ and group-lasso penalties, where the automatic-differentiation oracle returns $\vzero$ at kinks and the hypothesis holds.
\end{remark}

\subsection{Fixed Points with \texorpdfstring{$\boldsymbol{\tau}=0$}{tau=0}}\label{app:tau-zero}

\begin{lemma}\label{lem:tau-zero}
If $\tau = 0$, then $\dts = \vzero$ if and only if $-\gt_1 \in \cone\{\gt_2,\ldots,\gt_K\}$. Equivalently, $\dts = \vzero$ iff $\vzero \in \conv\{\gt_1,\gt_2,\ldots,\gt_K\}$ \emph{with strictly positive weight on} $\gt_1$. In particular, $\dts = \vzero \Rightarrow \vth$ is Pareto-stationary in normalized-gradient space.
\end{lemma}

\begin{remark}[On the strict-positivity condition]\label{rem:tau-zero-strict}
The condition $\vzero \in \conv\{\gt_1,\ldots,\gt_K\}$ alone (Pareto stationarity) is \emph{not} equivalent to $\dts = \vzero$. Counter-example: in $\R^2$ with $\gt_1=(1,0)$, $\gt_2=(0,1)$, $\gt_3=(0,-1)$, the QP forces $d_2 = 0$ and yields $\dts = (1,0) = \gt_1 \neq \vzero$, while $\vzero = 0\cdot\gt_1 + \tfrac12\gt_2 + \tfrac12\gt_3$ (Pareto-stationary, but with zero weight on $\gt_1$). Lemma~\ref{lem:tau-zero} uses only the implication $\dts = \vzero \Rightarrow$ Pareto-stationary, which holds without further qualification.
\end{remark}

\begin{proof}
The QP becomes $\min \tfrac{1}{2}\norm{\dt - \gt_1}^2$ s.t.\ $\gt_j^\T\dt \geq 0$ for $j \geq 2$.

($\Leftarrow$) If $-\gt_1 = \sum_{j\geq 2}\mu_j \gt_j$ with $\mu_j \geq 0$, set $\dt = \vzero$. Primal feasibility: $\gt_j^\T\vzero = 0 \geq 0$. Dual feasibility: $\mu_j \geq 0$. Complementary slackness: $\mu_j \cdot 0 = 0$. Stationarity: $\vzero - \gt_1 - \sum_j \mu_j \gt_j = -\gt_1 - (-\gt_1) = \vzero$. Hence $\dts = \vzero$ by KKT.

($\Rightarrow$) If $\dts = \vzero$, KKT stationarity gives $\vzero = \gt_1 + \sum_j \mu_j^* \gt_j$ with $\mu_j^* \geq 0$, i.e.\ $-\gt_1 = \sum_j \mu_j^* \gt_j \in \cone\{\gt_j\}_{j\geq 2}$.

\emph{Equivalence with $\vzero \in \conv$ (with $\alpha_1 > 0$) and Pareto stationarity.} Given $-\gt_1 = \sum_j\mu_j^*\gt_j$ with $\mu_j^* \geq 0$, dividing by $1+\sum_j\mu_j^* > 0$ gives
\[
\vzero = \tfrac{1}{1 + \sum_j \mu_j^*}\,\gt_1 + \sum_j \tfrac{\mu_j^*}{1 + \sum_j \mu_j^*}\,\gt_j \in \conv\{\gt_1,\ldots,\gt_K\},
\]
with weight $\alpha_1 = 1/(1+\sum\mu_j^*) > 0$ on $\gt_1$, hence Pareto-stationary.
\end{proof}

This recovers the MGDA-style guarantee at $\tau = 0$ and contrasts sharply with Theorem~\ref{thm:cms}: positive tolerances collapse the fixed-point set from Pareto-stationarity to CMS only (under the oracle hypothesis of Rem.~\ref{rem:cms-oracle}).

\subsection{Sufficient-Decrease Sketch}\label{app:sufficient-decrease}

Theorem~\ref{thm:cms} characterizes the endpoints of the deterministic, normalized QP but does not by itself certify that the iterates of Algorithm~\ref{alg:pcd} approach such an endpoint. A standard sufficient-decrease argument applied to the primary loss yields the following partial convergence statement, sketched here for completeness. We do not claim it as a contribution of this paper.

Suppose $\mathcal{L}_1$ is $L$-smooth and bounded below, the QP is feasible at every iterate, the iterate sequence is bounded, and there exists $\mu_0 > 0$ such that the cosine alignment $\widetilde\alpha_t := \gt_{1,t}^\T\dts_t/(\norm{\gt_{1,t}}\norm{\dts_t}) \geq \mu_0$ for all $t$. Although $\widetilde\alpha_t$ is stated on the normalized quantities, it is also the cosine between the raw gradient $\vg_{1,t}$ and the deployed direction $\vd^*_t$: $\gt_{1,t} = s_{1,t}\vg_{1,t}$ and $\vd^*_t = \dts_t\,\norm{\vg_{1,t}}/\norm{\dts_t}$ with $s_{1,t} > 0$, and cosines are invariant to positive scalar rescaling of either argument. The descent lemma below is applied to the raw objective with the raw step, and the two calculations agree through this identity. Then for constant step size $\eta \leq \mu_0/L$, the descent lemma applied to $\mathcal{L}_1$ together with the rescaling identity $\norm{\vd_t^*} = \norm{\vg_{1,t}}$ yields $\mathcal{L}_1(\vth_{t+1}) \leq \mathcal{L}_1(\vth_t) - \tfrac{\eta\mu_0}{2}\norm{\vg_{1,t}}^2$, and telescoping over $t < T$ gives $\min_{t<T}\norm{\vg_{1,t}}^2 = O(1/T)$. Two caveats are worth noting. First, the rate certifies primary stationarity only: by Theorem~\ref{thm:cms}, CMS corresponds to $\dts(\vth)=\vzero$, while the rescaling halts the algorithm whenever $\vd^*(\vth) = \dts\cdot\norm{\vg_1}/\norm{\dts} = \vzero$, which admits both the CMS branch and a primary-stationary non-CMS branch in which $\vg_1 = \vzero$ but $\dts \neq \vzero$. The asymmetric/regularizer regime of Eq.~\eqref{eq:objectives}, in which secondaries are jointly stationary at primary minima, makes the second branch non-generic. Second, the alignment hypothesis $\widetilde\alpha_t \geq \mu_0$ holds with $\mu_0 = 1$ in the conflict-free regime but can fail in severe conflict, so the result is conditional rather than unconditional. Strengthening to a single-limit-point statement via a Kurdyka--{\L}ojasiewicz argument is non-trivial: the natural Lyapunov candidate $\Phi = \sum_i \mathcal{L}_i$ has critical points at weighted-sum stationarity rather than at CMS, so a merit function whose critical points coincide with CMS would need to be constructed first. We leave this to future work.

\section{Feasibility at \texorpdfstring{$K > 2$}{K>2}}
\label{app:Kgt2}

The theory of Sec.~\ref{sec:pcd} applies uniformly for all $K \geq 2$. The one phenomenon specific to $K > 2$ is that the feasible set $F(\tau)$ can become empty when the secondary gradients are mutually incompatible. The Farkas characterization of Theorem~\ref{thm:feas} gives the tight condition. The following angle-based sufficient bound is useful when only pairwise cosines are available.

\begin{proposition}[Sufficient condition for feasibility]\label{prop:Kfeas}
Let $\theta_{\max} = \max_{j\neq k\in[K]_{-1}}\theta_{jk}$ be the maximum pairwise angle among the (nonzero, normalized) secondary gradients, with the convention $\arccos(1/0) = 0$ for $K=2$. If
\[
\theta_{\max} \;<\; \pi - \arccos\!\left(\tfrac{1}{K-2}\right) ,
\]
then $\vzero \notin \cone\{\gt_2,\ldots,\gt_K\}$, hence $F(\tau)\neq\emptyset$ for all $\tau\in[0,1]$.
\end{proposition}

\begin{proof}
Suppose for contradiction $\vzero \in \cone\{\gt_j\}_{j\geq 2}$ with each $\gt_j$ a unit vector. Then $\vzero = \sum_{j=2}^K \alpha_j \gt_j$ for some $\alpha_j \geq 0$, not all zero. Set $S = \sum_j \alpha_j > 0$. Computing $\norm{\sum_j \alpha_j \gt_j}^2 = 0$ and expanding,
\[
\sum_j \alpha_j^2 \;=\; -2\sum_{j<k}\alpha_j\alpha_k \cos\theta_{jk} .
\]
By hypothesis $\cos\theta_{jk} > -1/(K-2)$ for all $j<k$, so $-\cos\theta_{jk} < 1/(K-2)$ and
\[
\sum_j \alpha_j^2 \;<\; \tfrac{2}{K-2}\sum_{j<k}\alpha_j\alpha_k \;=\; \tfrac{1}{K-2}\Big(S^2 - \sum_j \alpha_j^2\Big),
\]
i.e.\ $(K-1)\sum_j \alpha_j^2 < S^2$. But Cauchy--Schwarz on $S = \sum_j 1\cdot\alpha_j$ gives $S^2 \leq (K-1)\sum_j \alpha_j^2$, a contradiction. Feasibility then follows from the Farkas characterization (Theorem~\ref{thm:feas}).
\end{proof}

For $K=3$, the bound reduces to $\theta_{\max} < \pi$: feasibility holds unless the two secondaries are exactly anti-parallel, recovering Cor.~\ref{cor:infeas}(ii). For $K\geq 4$, the bound is tight at the equilateral configuration. The condition is sufficient but not necessary. For instance, with $K=4$ and three planar secondaries at angles $0^\circ,\,100^\circ,\,170^\circ$, $\theta_{\max} = 170^\circ > 120^\circ$ but all three lie in a common open half-plane, so $\vzero \notin \cone$. The tight characterization remains the Farkas alternative.

\section{Expanded Results: Structured Pruning}
\label{app:expanded-structured-pruning}

We report full per-configuration results omitted from the main body, complementing Tables~\ref{tab:densenet-cifar10} and~\ref{tab:resnet-cifar100}. Table~\ref{tab:pruning-accuracy-all} reports test accuracy at every sparsity level for the six architecture/dataset configurations not covered in Tables~\ref{tab:densenet-cifar10} and~\ref{tab:resnet-cifar100} (eight in total) and all baselines. Table~\ref{tab:pruning-efficiency-all} reports the unpruned cost of each configuration together with every method's measured efficiency at the 80\% and 99\% operating points. Intermediate points are visualized in the Pareto frontiers (Fig.~\ref{fig:pareto-all}), and $\tau$-sweeps appear in Fig.~\ref{fig:tau-sweep-all}.

\begin{table}[!htbp]
\centering\small
\setlength{\tabcolsep}{4pt}
\renewcommand{\arraystretch}{1.05}
\begin{tabular}{ll | ccccccc}
\toprule
\textbf{Configuration} & \textbf{Target} & \textbf{PCD (Ours)} & \textbf{WS} & \textbf{CAGrad} & \textbf{MGDA} & \textbf{PCGrad} & \textbf{FAMO} & \textbf{AuxiNash} \\
\midrule
\multirow{6}{*}{\shortstack[l]{DenseNet-121\\CIFAR-100}}
 & Unpruned & \textbf{75.2} & \textbf{75.2} & \textbf{75.2} & \textbf{75.2} & \textbf{75.2} & \textbf{75.2} & \textbf{75.2} \\
 & 80\%     & \textbf{74.8} & 38.5 & 28.0 & 12.1 & 21.1 & 2.3 & 35.8 \\
 & 85\%     & \textbf{74.8} & 38.0 & 27.4 & 12.1 & 21.0 & 2.4 & 35.8 \\
 & 90\%     & \textbf{72.7} & 36.8 & 26.7 & 11.8 & 20.7 & 2.3 & 35.7 \\
 & 95\%     & \textbf{71.6} & 36.0 & 26.7 & 11.7 & 19.9 & 2.2 & 35.7 \\
 & 99\%     & \textbf{46.5} & 34.1 & 26.7 & 11.5 & 19.0 & 2.1 & 35.7 \\
\midrule
\multirow{6}{*}{\shortstack[l]{ResNet-34\\CIFAR-10}}
 & Unpruned & \textbf{94.4} & \textbf{94.4} & \textbf{94.4} & \textbf{94.4} & \textbf{94.4} & \textbf{94.4} & \textbf{94.4} \\
 & 80\%     & \textbf{93.0} & 70.8 & 61.6 & 10.3 & 51.6 & 23.8 & 73.7 \\
 & 85\%     & \textbf{93.0} & 70.8 & 61.6 & 10.3 & 51.6 & 23.8 & 73.7 \\
 & 90\%     & \textbf{93.0} & 70.8 & 61.6 & 10.3 & 51.6 & 23.8 & 73.7 \\
 & 95\%     & \textbf{92.0} & 70.8 & 61.6 & 10.3 & 51.6 & 23.8 & 73.7 \\
 & 99\%     & \textbf{75.8} & 63.5 & 53.8 & 10.2 & 51.2 & 23.8 & 61.3 \\
\midrule
\multirow{6}{*}{\shortstack[l]{Inception\\CIFAR-10}}
 & Unpruned & \textbf{91.0} & \textbf{91.0} & \textbf{91.0} & \textbf{91.0} & \textbf{91.0} & \textbf{91.0} & \textbf{91.0} \\
 & 80\%     & \textbf{89.6} & 64.2 & 33.8 & 19.1 & 32.3 & 10.0 & 56.3 \\
 & 85\%     & \textbf{89.6} & 64.2 & 33.8 & 18.9 & 32.3 & 10.0 & 56.3 \\
 & 90\%     & \textbf{89.6} & 64.2 & 33.8 & 18.5 & 32.3 & 10.0 & 56.3 \\
 & 95\%     & \textbf{86.2} & 64.2 & 33.8 & 18.5 & 32.3 & 10.0 & 56.3 \\
 & 99\%     & \textbf{78.4} & 64.2 & 33.8 & 18.2 & 32.3 & 10.0 & 56.3 \\
\midrule
\multirow{6}{*}{\shortstack[l]{Inception\\CIFAR-100}}
 & Unpruned & \textbf{68.5} & \textbf{68.5} & \textbf{68.5} & \textbf{68.5} & \textbf{68.5} & \textbf{68.5} & \textbf{68.5} \\
 & 80\%     & \textbf{68.5} & 14.0 & 8.5 & 2.4 & 8.9 & 1.0 & 20.4 \\
 & 85\%     & \textbf{61.4} & 14.0 & 8.5 & 2.4 & 8.9 & 1.0 & 20.5 \\
 & 90\%     & \textbf{61.6} & 14.0 & 8.5 & 2.4 & 8.9 & 1.0 & 20.5 \\
 & 95\%     & \textbf{57.2} & 14.0 & 8.5 & 2.4 & 8.9 & 1.0 & 20.5 \\
 & 99\%     & \textbf{28.6} & 13.4 & 8.5 & 2.4 & 9.4 & 1.0 & 17.3 \\
\midrule
\multirow{6}{*}{\shortstack[l]{MobileNetV2\\CIFAR-10}}
 & Unpruned & \textbf{93.1} & \textbf{93.1} & \textbf{93.1} & \textbf{93.1} & \textbf{93.1} & \textbf{93.1} & \textbf{93.1} \\
 & 80\%     & \textbf{93.0} & 41.7 & 26.2 & 11.7 & 25.7 & 10.0 & 60.2 \\
 & 85\%     & \textbf{93.0} & 41.6 & 26.2 & 11.7 & 25.6 & 10.0 & 60.2 \\
 & 90\%     & \textbf{92.2} & 41.6 & 29.2 & 11.7 & 25.6 & 10.0 & 60.0 \\
 & 95\%     & \textbf{88.2} & 41.6 & 26.1 & 12.7 & 5.5 & 10.0 & 60.0 \\
 & 99\%     & \textbf{18.3} & 16.7 & 16.7 & 12.7 & 11.4 & 10.0 & 10.9 \\
\midrule
\multirow{6}{*}{\shortstack[l]{MobileNetV2\\CIFAR-100}}
 & Unpruned & \textbf{72.0} & \textbf{72.0} & \textbf{72.0} & \textbf{72.0} & \textbf{72.0} & \textbf{72.0} & \textbf{72.0} \\
 & 80\%     & \textbf{68.0} & 15.3 & 6.2 & 1.2 & 4.6 & 1.0 & 17.0 \\
 & 85\%     & \textbf{64.8} & 15.3 & 6.2 & 1.2 & 4.5 & 1.0 & 16.9 \\
 & 90\%     & \textbf{55.3} & 15.2 & 6.2 & 1.2 & 4.5 & 1.0 & 16.9 \\
 & 95\%     & \textbf{24.6} & 15.2 & 6.3 & 1.2 & 4.5 & 1.0 & 16.9 \\
 & 99\%     & 1.2 & \textbf{1.9} & 1.2 & 1.1 & 1.1 & 1.0 & 1.0 \\
\bottomrule
\end{tabular}
\caption{Test accuracy (\%) on the structured-pruning task across the 6 architecture/dataset configurations not covered in the main body (8 in total), all baselines, and all sparsity targets. Means over 5 seeds, best values in bold. Per-method deployment costs at the 80\% and 99\% targets are reported separately in Tab.~\ref{tab:pruning-efficiency-all}.}
\label{tab:pruning-accuracy-all}
\end{table}

\begin{table}[!htbp]
\centering\scriptsize
\setlength{\tabcolsep}{3.5pt}
\caption{Deployment efficiency at the boundary operating points, per method. ``Unpruned'' lists raw FLOPs, latency, and file size, and the $80\%$ and $99\%$ sparsity columns list the absolute value of each method's physically pruned network at that operating point with the reduction factor relative to the unpruned model in parentheses. Intermediate $85$--$95\%$ points are visualized in Fig.~\ref{fig:pareto-all}.}
\label{tab:pruning-efficiency-all}
\begin{tabular}{ll|ccc|ccc}
\toprule
\textbf{Configuration} & \textbf{Method} & \multicolumn{3}{c|}{\textbf{@ 80\% sparsity}} & \multicolumn{3}{c}{\textbf{@ 99\% sparsity}} \\
 & & FLOPs (G) & Lat.\ (ms) & Size (MB) & FLOPs (G) & Lat.\ (ms) & Size (MB) \\
\midrule
\multirow{8}{*}{DenseNet-121 / C-100} & Unpruned & 1.82 & 32.3 & 27.6 & -- & -- & -- \\
 & PCD & 0.40 (4.6$\times$) & 20.9 (1.5$\times$) & 5.7 (4.9$\times$) & 0.022 (81.7$\times$) & 6.7 (4.8$\times$) & 0.47 (58.9$\times$) \\
 & WS & 0.22 (8.3$\times$) & 17.1 (1.9$\times$) & 5.7 (4.8$\times$) & 0.015 (122.2$\times$) & 6.4 (5.1$\times$) & 0.47 (58.5$\times$) \\
 & CAGrad & 0.20 (9.1$\times$) & 16.2 (2.0$\times$) & 5.7 (4.8$\times$) & 0.018 (98.8$\times$) & 7.8 (4.2$\times$) & 0.47 (59.3$\times$) \\
 & MGDA & 0.21 (8.8$\times$) & 16.8 (1.9$\times$) & 5.7 (4.8$\times$) & 0.022 (82.2$\times$) & 8.1 (4.0$\times$) & 0.48 (57.2$\times$) \\
 & PCGrad & 0.21 (8.5$\times$) & 16.8 (1.9$\times$) & 5.7 (4.8$\times$) & 0.021 (84.7$\times$) & 8.2 (3.9$\times$) & 0.48 (57.9$\times$) \\
 & FAMO & 0.21 (8.7$\times$) & 16.7 (1.9$\times$) & 5.7 (4.8$\times$) & 0.026 (71.0$\times$) & 8.6 (3.8$\times$) & 0.52 (53.5$\times$) \\
 & AuxiNash & 0.21 (8.7$\times$) & 16.5 (2.0$\times$) & 5.7 (4.9$\times$) & 0.022 (84.1$\times$) & 6.6 (4.9$\times$) & 0.45 (61.1$\times$) \\
\midrule
\multirow{8}{*}{ResNet-34 / C-10} & Unpruned & 2.33 & 24.4 & 81.3 & -- & -- & -- \\
 & PCD & 0.67 (3.5$\times$) & 11.2 (2.2$\times$) & 16.3 (5.0$\times$) & 0.13 (18.4$\times$) & 6.1 (4.0$\times$) & 1.3 (61.4$\times$) \\
 & WS & 0.68 (3.4$\times$) & 11.6 (2.1$\times$) & 16.3 (5.0$\times$) & 0.048 (48.9$\times$) & 5.4 (4.5$\times$) & 1.1 (72.4$\times$) \\
 & CAGrad & 0.65 (3.6$\times$) & 11.2 (2.2$\times$) & 16.3 (5.0$\times$) & 0.041 (56.2$\times$) & 5.3 (4.6$\times$) & 1.1 (74.3$\times$) \\
 & MGDA & 0.44 (5.3$\times$) & 9.7 (2.5$\times$) & 16.3 (5.0$\times$) & 0.036 (63.8$\times$) & 5.3 (4.6$\times$) & 1.1 (75.4$\times$) \\
 & PCGrad & 0.43 (5.4$\times$) & 9.6 (2.6$\times$) & 16.3 (5.0$\times$) & 0.039 (59.3$\times$) & 5.4 (4.5$\times$) & 1.1 (72.8$\times$) \\
 & FAMO & 0.44 (5.3$\times$) & 9.8 (2.5$\times$) & 16.2 (5.0$\times$) & 0.036 (63.8$\times$) & 5.3 (4.7$\times$) & 1.1 (75.3$\times$) \\
 & AuxiNash & 0.70 (3.3$\times$) & 11.4 (2.1$\times$) & 16.3 (5.0$\times$) & 0.053 (43.7$\times$) & 5.3 (4.6$\times$) & 1.1 (71.7$\times$) \\
\midrule
\multirow{8}{*}{Inception / C-10} & Unpruned & 0.53 & 8.9 & 5.4 & -- & -- & -- \\
 & PCD & 0.12 (4.6$\times$) & 4.9 (1.8$\times$) & 1.1 (4.8$\times$) & 0.010 (55.3$\times$) & 2.3 (3.9$\times$) & 0.10 (52.2$\times$) \\
 & WS & 0.065 (8.1$\times$) & 3.3 (2.7$\times$) & 1.1 (4.9$\times$) & 0.005 (96.7$\times$) & 2.0 (4.5$\times$) & 0.10 (55.7$\times$) \\
 & CAGrad & 0.060 (8.9$\times$) & 3.0 (3.0$\times$) & 1.1 (4.7$\times$) & 0.004 (141.1$\times$) & 1.6 (5.5$\times$) & 0.10 (55.3$\times$) \\
 & MGDA & 0.050 (10.6$\times$) & 3.1 (2.9$\times$) & 1.1 (4.9$\times$) & 0.003 (196.5$\times$) & 1.6 (5.6$\times$) & 0.09 (60.6$\times$) \\
 & PCGrad & 0.063 (8.5$\times$) & 3.2 (2.8$\times$) & 1.1 (4.9$\times$) & 0.003 (159.6$\times$) & 1.6 (5.5$\times$) & 0.09 (58.0$\times$) \\
 & FAMO & 0.041 (13.0$\times$) & 2.4 (3.7$\times$) & 1.1 (5.0$\times$) & 0.002 (242.9$\times$) & 1.5 (5.7$\times$) & 0.09 (61.6$\times$) \\
 & AuxiNash & 0.042 (12.8$\times$) & 2.4 (3.7$\times$) & 1.1 (4.9$\times$) & 0.003 (169.5$\times$) & 1.7 (5.1$\times$) & 0.09 (57.2$\times$) \\
\midrule
\multirow{8}{*}{Inception / C-100} & Unpruned & 0.53 & 8.9 & 5.6 & -- & -- & -- \\
 & PCD & 0.094 (5.7$\times$) & 4.6 (1.9$\times$) & 1.2 (4.8$\times$) & 0.002 (245.2$\times$) & 1.7 (5.3$\times$) & 0.09 (63.8$\times$) \\
 & WS & 0.076 (7.0$\times$) & 3.9 (2.3$\times$) & 1.2 (4.9$\times$) & 0.002 (347.5$\times$) & 1.6 (5.5$\times$) & 0.09 (65.7$\times$) \\
 & CAGrad & 0.094 (5.7$\times$) & 3.7 (2.4$\times$) & 1.2 (4.8$\times$) & 0.004 (119.1$\times$) & 1.6 (5.4$\times$) & 0.10 (59.1$\times$) \\
 & MGDA & 0.050 (10.6$\times$) & 2.9 (3.1$\times$) & 1.1 (4.9$\times$) & 0.003 (195.3$\times$) & 1.6 (5.5$\times$) & 0.10 (57.9$\times$) \\
 & PCGrad & 0.088 (6.0$\times$) & 3.8 (2.4$\times$) & 1.2 (4.9$\times$) & 0.005 (106.1$\times$) & 1.8 (5.1$\times$) & 0.10 (59.0$\times$) \\
 & FAMO & 0.041 (13.0$\times$) & 2.9 (3.1$\times$) & 1.2 (4.8$\times$) & 0.002 (325.0$\times$) & 1.6 (5.6$\times$) & 0.09 (61.4$\times$) \\
 & AuxiNash & 0.050 (10.5$\times$) & 3.1 (2.9$\times$) & 1.2 (4.9$\times$) & 0.002 (269.7$\times$) & 1.6 (5.4$\times$) & 0.09 (60.2$\times$) \\
\midrule
\multirow{8}{*}{MobileNetV2 / C-10} & Unpruned & 0.19 & 9.0 & 8.8 & -- & -- & -- \\
 & PCD & 0.086 (2.2$\times$) & 7.4 (1.2$\times$) & 1.8 (4.8$\times$) & 0.003 (60.5$\times$) & 3.5 (2.6$\times$) & 0.18 (49.3$\times$) \\
 & WS & 0.040 (4.8$\times$) & 5.8 (1.5$\times$) & 1.8 (4.8$\times$) & 0.006 (31.7$\times$) & 3.9 (2.3$\times$) & 0.22 (39.5$\times$) \\
 & CAGrad & 0.033 (5.8$\times$) & 5.2 (1.7$\times$) & 1.8 (4.8$\times$) & 0.003 (57.5$\times$) & 3.5 (2.6$\times$) & 0.19 (46.9$\times$) \\
 & MGDA & 0.026 (7.2$\times$) & 5.2 (1.7$\times$) & 1.8 (4.8$\times$) & 0.002 (86.8$\times$) & 3.2 (2.8$\times$) & 0.18 (48.9$\times$) \\
 & PCGrad & 0.037 (5.2$\times$) & 5.6 (1.6$\times$) & 1.8 (4.8$\times$) & 0.005 (38.0$\times$) & 3.6 (2.5$\times$) & 0.19 (47.3$\times$) \\
 & FAMO & 0.029 (6.6$\times$) & 5.2 (1.7$\times$) & 1.8 (4.8$\times$) & 0.002 (81.2$\times$) & 3.5 (2.6$\times$) & 0.19 (46.2$\times$) \\
 & AuxiNash & 0.031 (6.1$\times$) & 5.2 (1.7$\times$) & 1.8 (4.9$\times$) & 0.006 (32.5$\times$) & 3.7 (2.4$\times$) & 0.22 (39.5$\times$) \\
\midrule
\multirow{8}{*}{MobileNetV2 / C-100} & Unpruned & 0.19 & 9.1 & 9.2 & -- & -- & -- \\
 & PCD & 0.055 (3.4$\times$) & 6.5 (1.4$\times$) & 1.9 (4.8$\times$) & 0.013 (14.9$\times$) & 4.2 (2.1$\times$) & 0.60 (15.5$\times$) \\
 & WS & 0.039 (4.8$\times$) & 5.7 (1.6$\times$) & 1.9 (4.8$\times$) & 0.006 (33.7$\times$) & 3.8 (2.4$\times$) & 0.26 (36.2$\times$) \\
 & CAGrad & 0.032 (5.9$\times$) & 5.3 (1.7$\times$) & 1.9 (4.8$\times$) & 0.002 (83.2$\times$) & 3.1 (2.9$\times$) & 0.19 (49.5$\times$) \\
 & MGDA & 0.026 (7.3$\times$) & 5.1 (1.8$\times$) & 1.9 (4.8$\times$) & 0.003 (74.4$\times$) & 3.4 (2.7$\times$) & 0.21 (43.5$\times$) \\
 & PCGrad & 0.032 (6.0$\times$) & 5.3 (1.7$\times$) & 1.9 (4.8$\times$) & 0.003 (71.2$\times$) & 3.4 (2.7$\times$) & 0.19 (49.2$\times$) \\
 & FAMO & 0.029 (6.5$\times$) & 5.1 (1.8$\times$) & 1.9 (4.8$\times$) & 0.003 (73.0$\times$) & 3.7 (2.5$\times$) & 0.21 (43.3$\times$) \\
 & AuxiNash & 0.031 (6.2$\times$) & 5.2 (1.8$\times$) & 1.9 (4.8$\times$) & 0.004 (49.3$\times$) & 3.3 (2.8$\times$) & 0.28 (32.7$\times$) \\
\bottomrule
\end{tabular}
\end{table}

\begin{figure}[!htbp]
\centering
\begin{subfigure}[t]{\textwidth}
  \centering
  \includegraphics[width=0.96\linewidth]{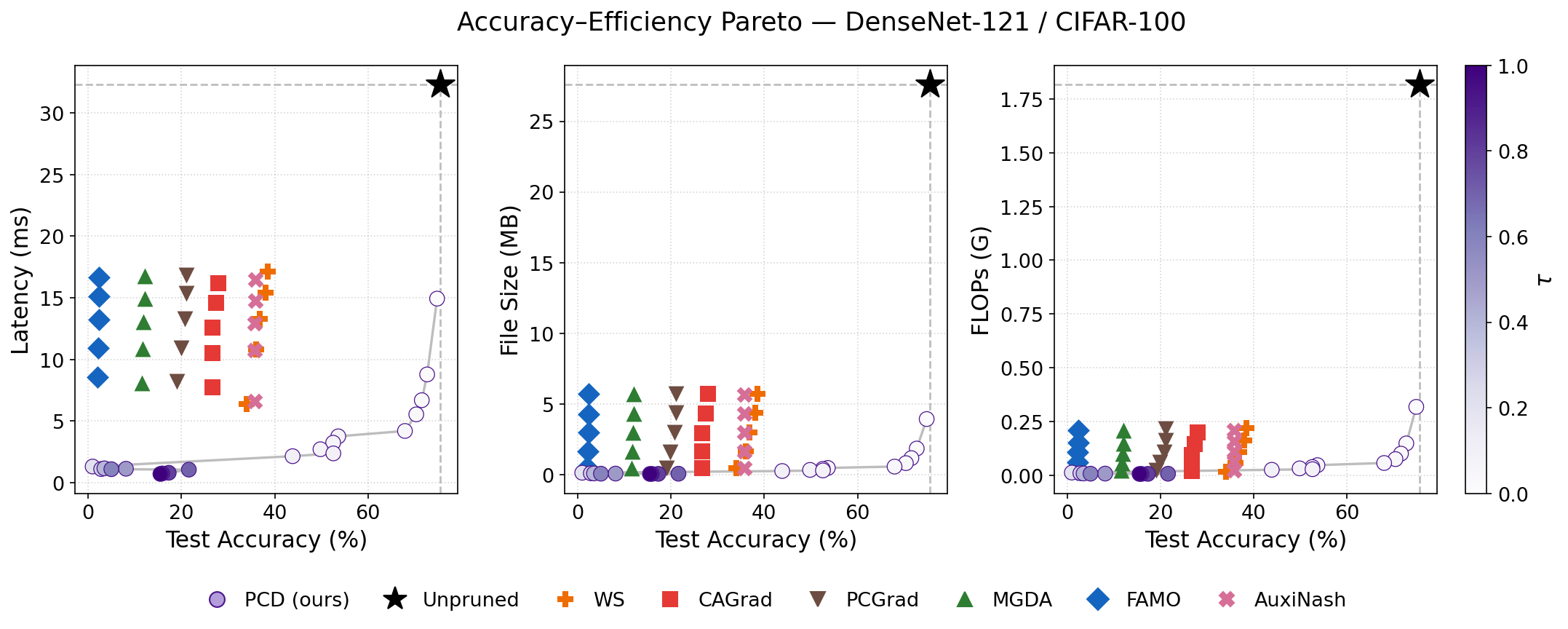}
  \caption{DenseNet-121 / CIFAR-100}
\end{subfigure}

\vspace{0.6em}

\begin{subfigure}[t]{\textwidth}
  \centering
  \includegraphics[width=0.96\linewidth]{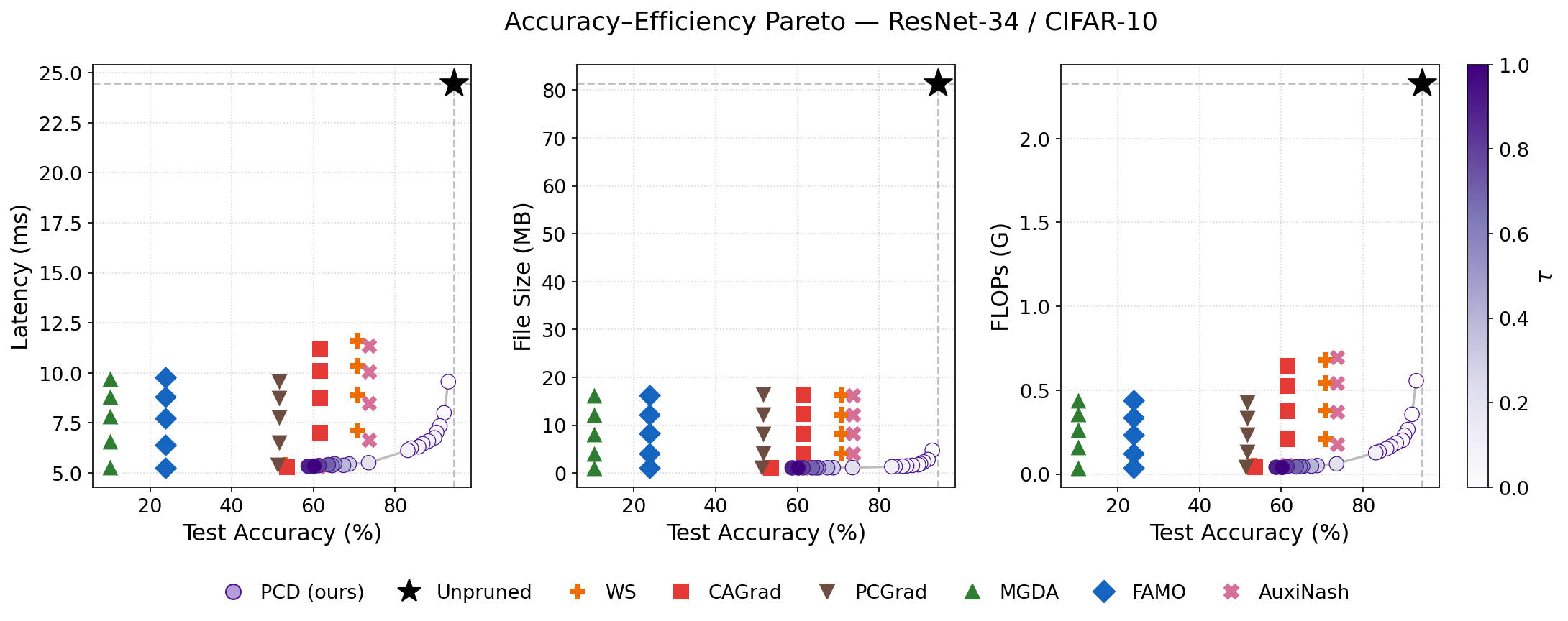}
  \caption{ResNet-34 / CIFAR-10}
\end{subfigure}

\vspace{0.6em}

\begin{subfigure}[t]{\textwidth}
  \centering
  \includegraphics[width=0.96\linewidth]{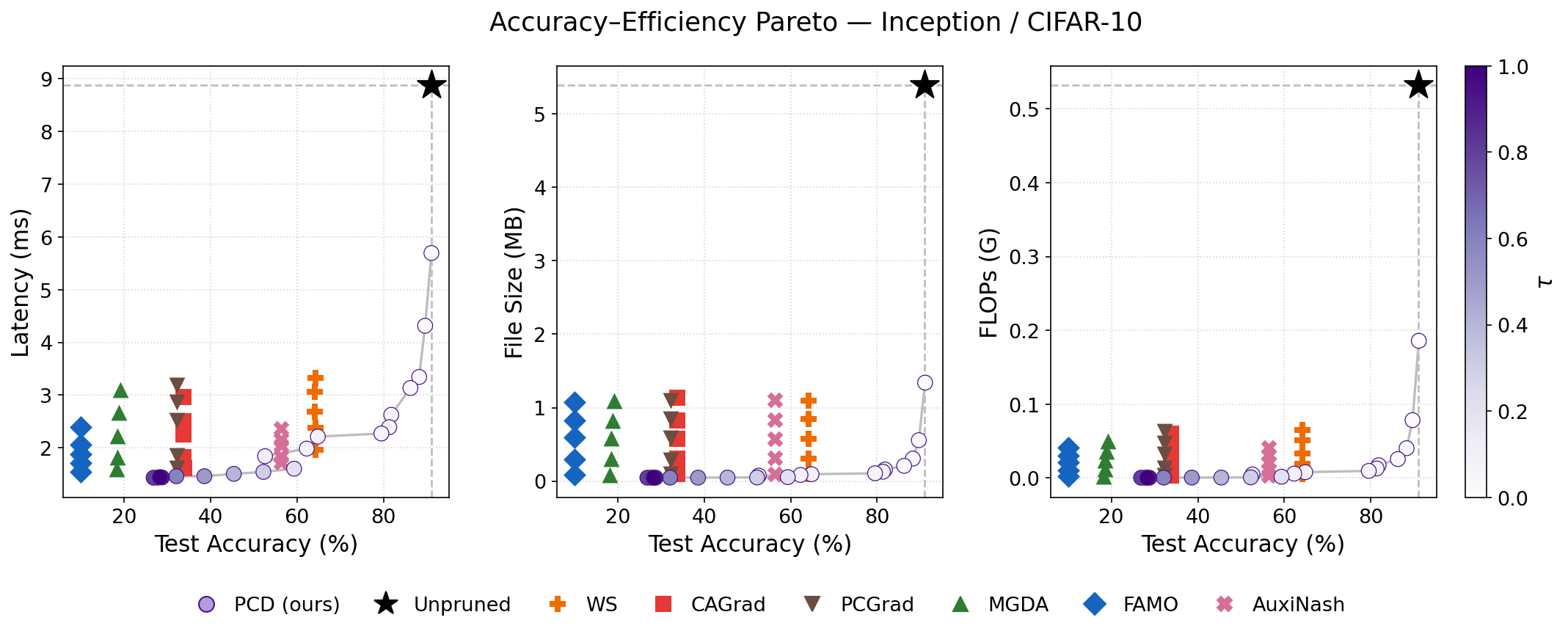}
  \caption{Inception / CIFAR-10}
\end{subfigure}

\caption{Accuracy--efficiency Pareto frontiers for PCD and every baseline across the 6 configurations not covered in the main body, plotted against inference latency (ms), file size (MB), and FLOPs (G). Baseline costs are measured per method on the physically pruned network at each target. Unpruned networks marked with $\star$.}
\label{fig:pareto-all}
\end{figure}

\begin{figure}[!htbp]
\ContinuedFloat
\centering
\begin{subfigure}[t]{\textwidth}
  \centering
  \includegraphics[width=0.96\linewidth]{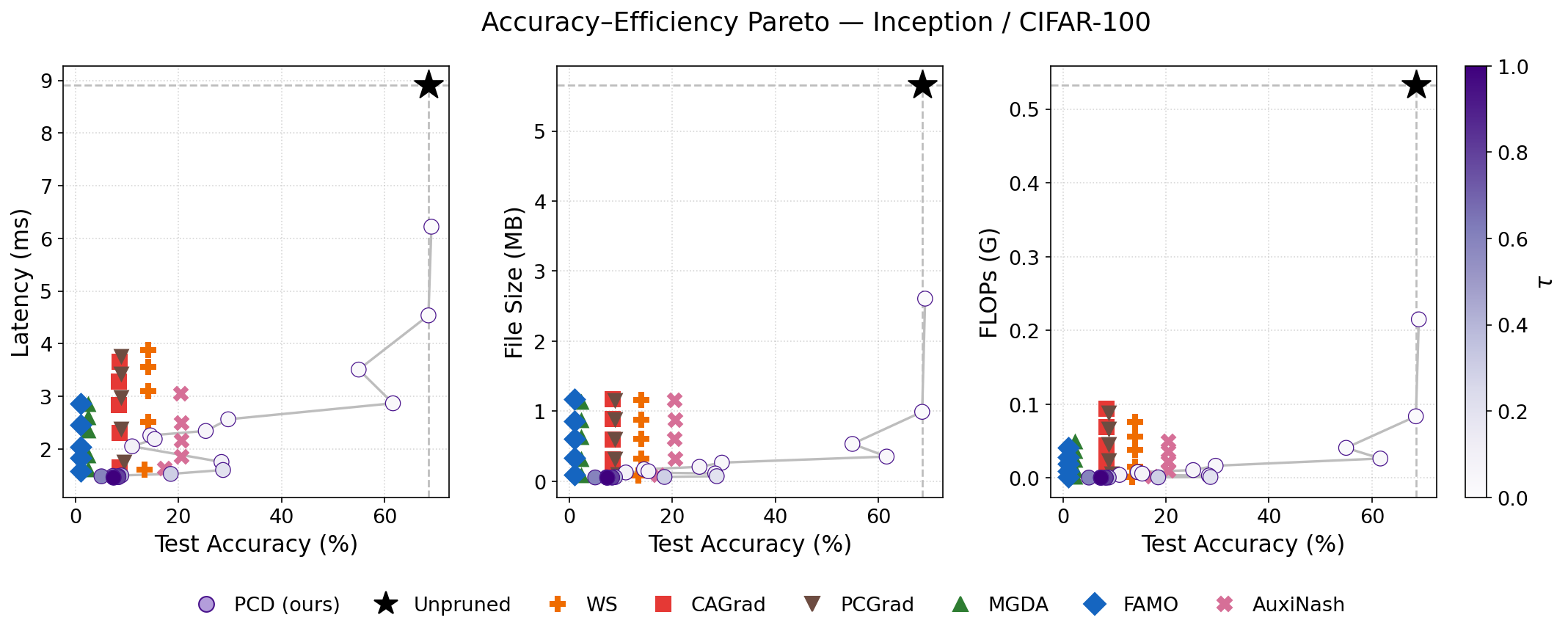}
  \caption{Inception / CIFAR-100}
\end{subfigure}

\vspace{0.6em}

\begin{subfigure}[t]{\textwidth}
  \centering
  \includegraphics[width=0.96\linewidth]{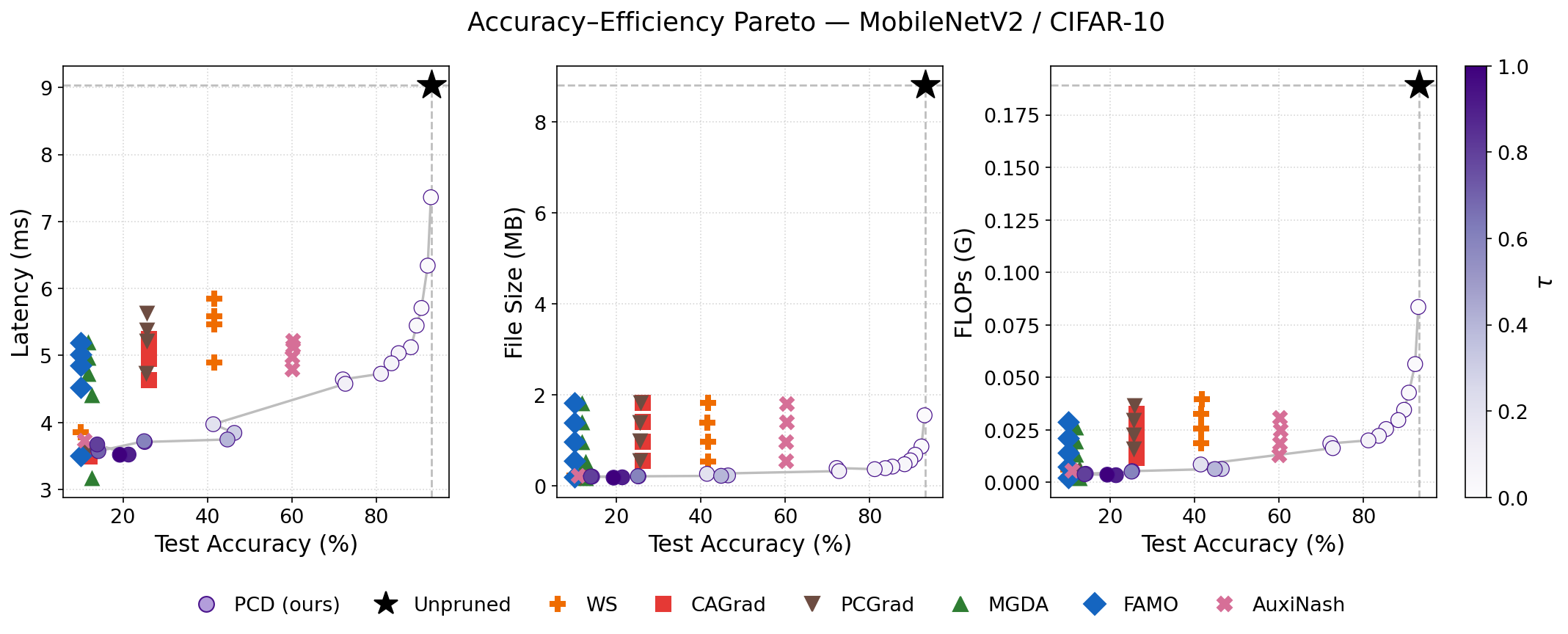}
  \caption{MobileNetV2 / CIFAR-10}
\end{subfigure}

\vspace{0.6em}

\begin{subfigure}[t]{\textwidth}
  \centering
  \includegraphics[width=0.96\linewidth]{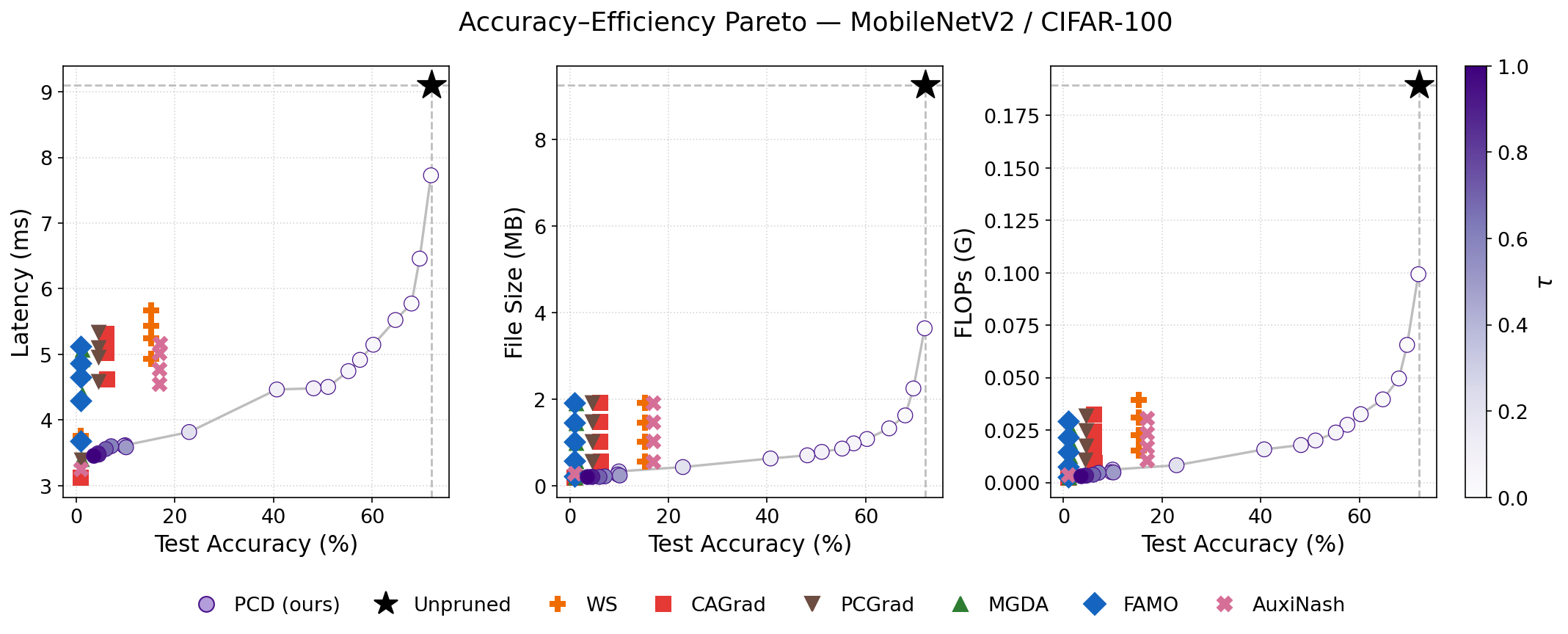}
  \caption{MobileNetV2 / CIFAR-100}
\end{subfigure}
\end{figure}

\begin{figure}[!htbp]
\centering
\begin{subfigure}[t]{0.48\textwidth}
  \includegraphics[width=\linewidth]{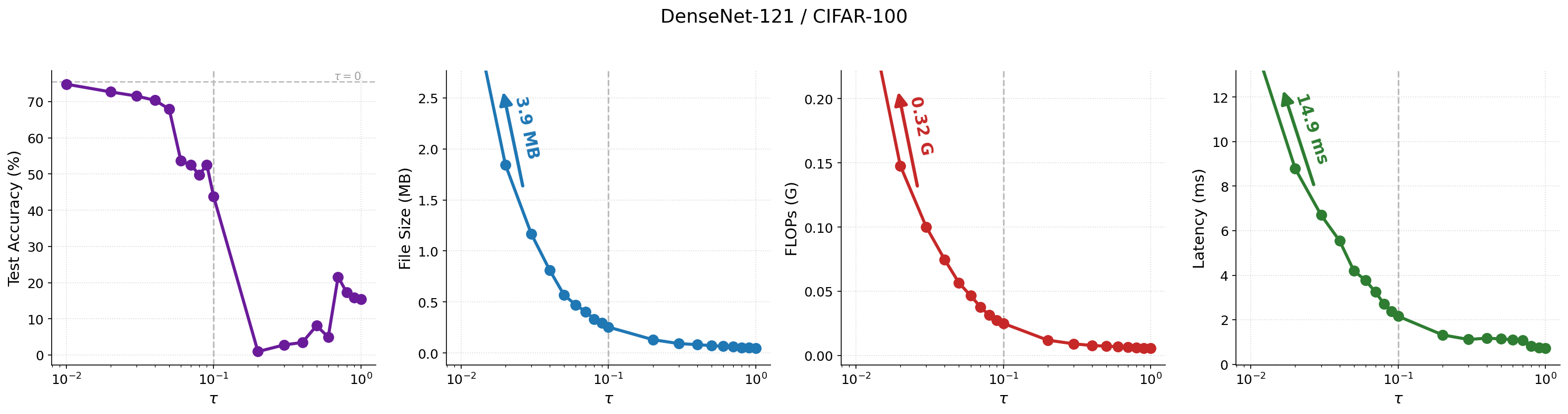}
  \caption{DenseNet-121 / CIFAR-100}
\end{subfigure}\hfill
\begin{subfigure}[t]{0.48\textwidth}
  \includegraphics[width=\linewidth]{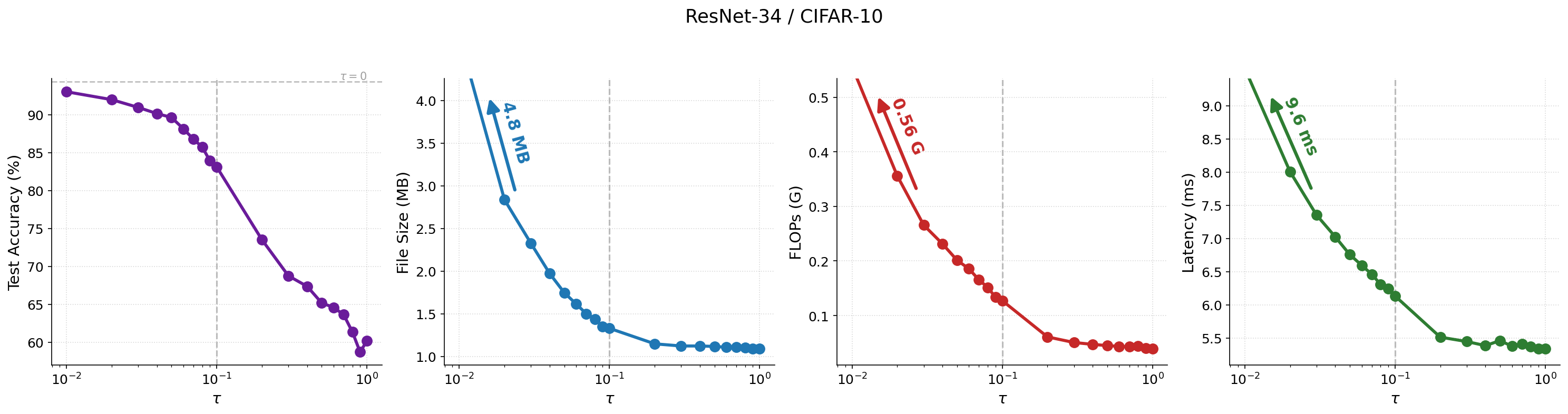}
  \caption{ResNet-34 / CIFAR-10}
\end{subfigure}

\vspace{1em}

\begin{subfigure}[t]{0.48\textwidth}
  \includegraphics[width=\linewidth]{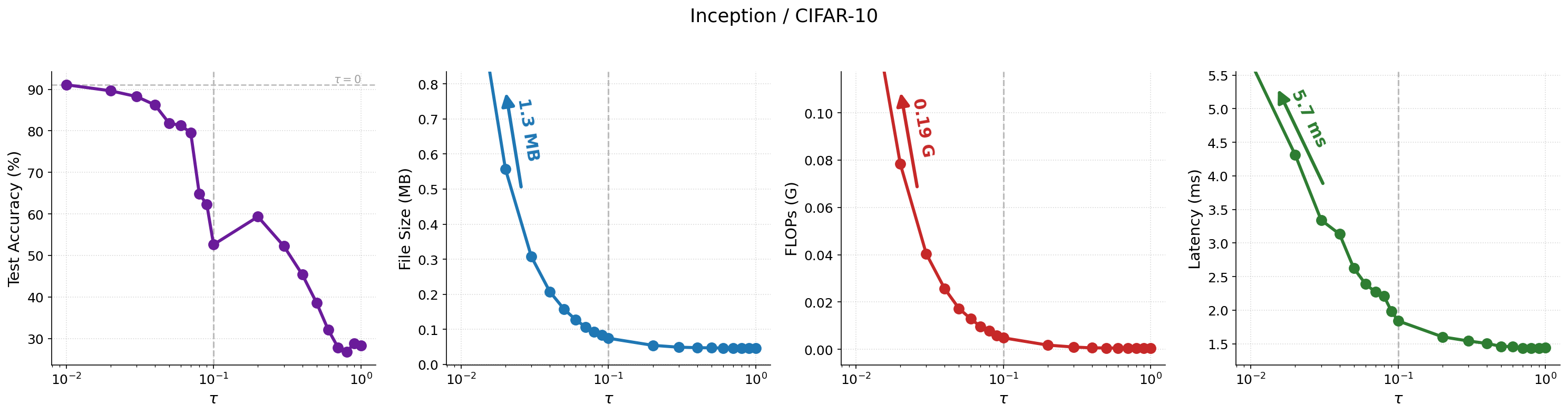}
  \caption{Inception / CIFAR-10}
\end{subfigure}\hfill
\begin{subfigure}[t]{0.48\textwidth}
  \includegraphics[width=\linewidth]{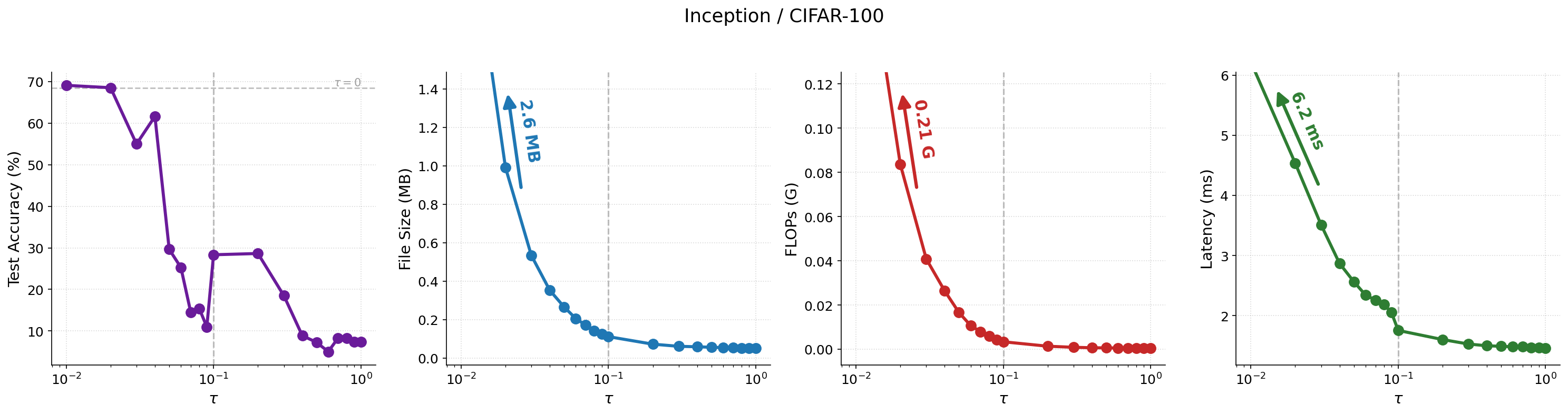}
  \caption{Inception / CIFAR-100}
\end{subfigure}

\vspace{1em}

\begin{subfigure}[t]{0.48\textwidth}
  \includegraphics[width=\linewidth]{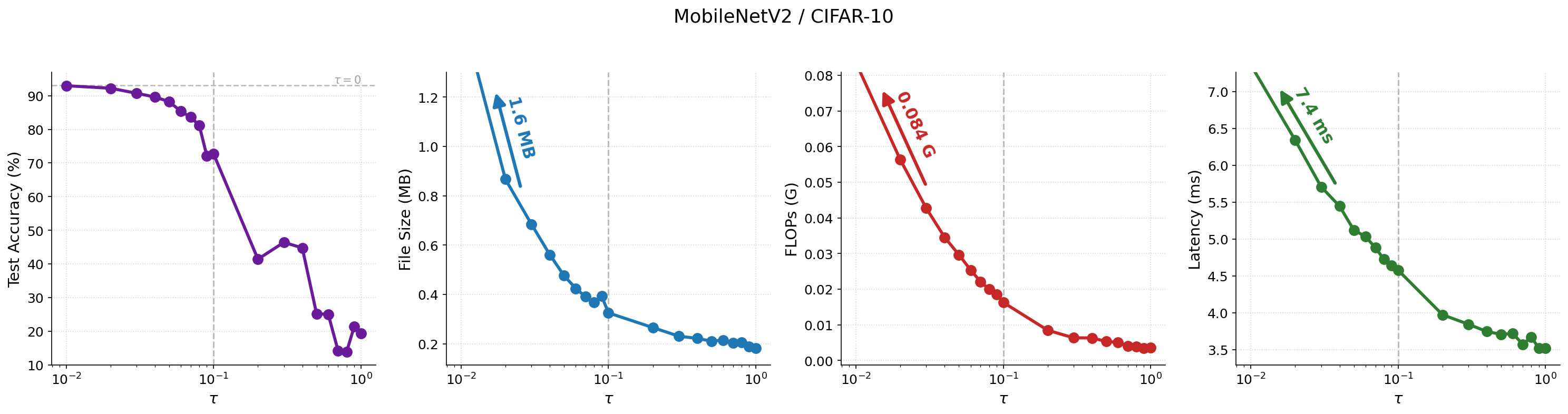}
  \caption{MobileNetV2 / CIFAR-10}
\end{subfigure}\hfill
\begin{subfigure}[t]{0.48\textwidth}
  \includegraphics[width=\linewidth]{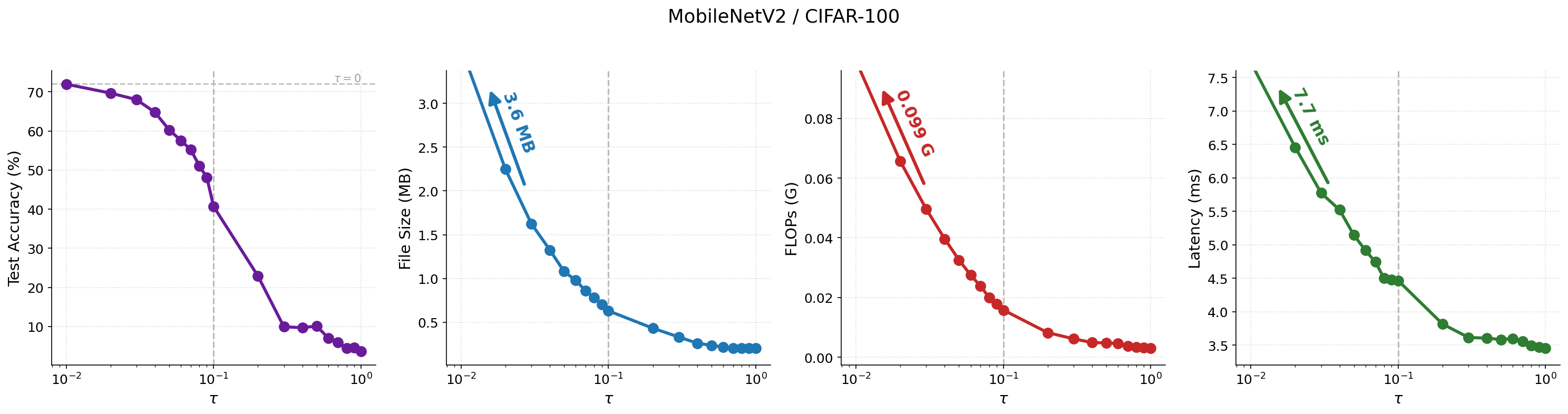}
  \caption{MobileNetV2 / CIFAR-100}
\end{subfigure}

\caption{Accuracy, file size (MB), FLOPs (G), and latency (ms) as a function of the priority tolerance $\tau$ for all 6 configurations not included in the main body of the paper. Consistently, the majority of compression benefit is realized for $\tau \in (0, 0.1]$. Arrows pointing up represent efficiency metric values at $\tau = 0.01$.}
\label{fig:tau-sweep-all}
\end{figure}

\subsection{Ablation: Gradient Normalization in PCD}
\label{app:norm-ablation}

We ablate the normalization scheme applied to the gradients before the PCD QP. 5 conditions are compared: no normalization ($s_i = 1$), instantaneous unit-norm normalization ($s_i = 1/\|\vg_i\|$), and EMA normalization with $\beta \in \{0.9,\,0.99,\,0.999\}$. Each condition is run on DenseNet-121 and ResNet-34, on CIFAR-10 and CIFAR-100, at $\tau \in \{0.01, 0.05\}$, over 3 seeds, with the shared settings of other experiments throughout. Table~\ref{tab:norm-ablation} reports the full grid. 

\begin{table}[h]
\centering\small
\setlength{\tabcolsep}{5pt}
\caption{Normalization ablation across all 4 configurations and both tolerance values, mean $\pm$ sd over three seeds.}
\label{tab:norm-ablation}
\begin{tabular}{ll|cc|cc}
\toprule
 & & \multicolumn{2}{c|}{$\tau = 0.01$} & \multicolumn{2}{c}{$\tau = 0.05$} \\
\cmidrule(lr){3-4}\cmidrule(lr){5-6}
\textbf{Configuration} & \textbf{Scheme} & \textbf{Acc.\ (\%)} & \textbf{Spars.\ (\%)} & \textbf{Acc.\ (\%)} & \textbf{Spars.\ (\%)} \\
\midrule
\multirow{5}{*}{DenseNet-121 / C-10} & None & $82.7 \pm 0.5$ & $98.1 \pm 0.1$ & $68.9 \pm 0.3$ & $99.2 \pm 0.0$ \\
 & Unit norm & $94.0 \pm 0.2$ & $80.4 \pm 1.3$ & $89.8 \pm 0.3$ & $95.7 \pm 0.3$ \\
 & EMA $\beta{=}0.9$ & $94.1 \pm 0.1$ & $81.5 \pm 0.7$ & $89.7 \pm 0.8$ & $95.7 \pm 0.2$ \\
 & EMA $\beta{=}0.99$ & $94.1 \pm 0.2$ & $81.6 \pm 0.5$ & $89.9 \pm 0.5$ & $95.8 \pm 0.3$ \\
 & EMA $\beta{=}0.999$ & $93.9 \pm 0.0$ & $82.5 \pm 1.2$ & $90.1 \pm 0.4$ & $95.5 \pm 0.3$ \\
\midrule
\multirow{5}{*}{DenseNet-121 / C-100} & None & $57.3 \pm 0.5$ & $95.6 \pm 0.1$ & $45.2 \pm 0.2$ & $97.7 \pm 0.3$ \\
 & Unit norm & $74.8 \pm 0.3$ & $73.8 \pm 1.3$ & $68.5 \pm 0.3$ & $92.6 \pm 0.2$ \\
 & EMA $\beta{=}0.9$ & $75.0 \pm 0.5$ & $74.4 \pm 0.4$ & $67.4 \pm 1.2$ & $92.7 \pm 0.2$ \\
 & EMA $\beta{=}0.99$ & $75.0 \pm 0.3$ & $74.7 \pm 0.9$ & $67.5 \pm 0.9$ & $92.8 \pm 0.3$ \\
 & EMA $\beta{=}0.999$ & $74.9 \pm 0.3$ & $74.5 \pm 1.0$ & $69.1 \pm 0.4$ & $92.3 \pm 0.2$ \\
\midrule
\multirow{5}{*}{ResNet-34 / C-10} & None & $81.2 \pm 0.5$ & $98.3 \pm 0.2$ & $73.1 \pm 0.7$ & $99.3 \pm 0.1$ \\
 & Unit norm & $93.4 \pm 0.2$ & $63.2 \pm 2.9$ & $89.6 \pm 0.3$ & $96.0 \pm 0.1$ \\
 & EMA $\beta{=}0.9$ & $93.3 \pm 0.0$ & $80.7 \pm 2.5$ & $89.7 \pm 0.3$ & $96.2 \pm 0.1$ \\
 & EMA $\beta{=}0.99$ & $93.4 \pm 0.2$ & $81.9 \pm 2.6$ & $89.5 \pm 0.2$ & $96.3 \pm 0.1$ \\
 & EMA $\beta{=}0.999$ & $93.2 \pm 0.1$ & $83.9 \pm 1.0$ & $89.3 \pm 0.6$ & $96.3 \pm 0.2$ \\
\midrule
\multirow{5}{*}{ResNet-34 / C-100} & None & $54.3 \pm 0.5$ & $98.3 \pm 0.2$ & $47.5 \pm 0.9$ & $99.2 \pm 0.1$ \\
 & Unit norm & $73.0 \pm 0.8$ & $70.9 \pm 4.0$ & $61.7 \pm 2.8$ & $96.7 \pm 0.2$ \\
 & EMA $\beta{=}0.9$ & $73.1 \pm 1.2$ & $76.0 \pm 3.6$ & $59.8 \pm 3.2$ & $96.6 \pm 0.8$ \\
 & EMA $\beta{=}0.99$ & $73.1 \pm 0.6$ & $75.7 \pm 2.8$ & $60.0 \pm 1.8$ & $97.0 \pm 0.7$ \\
 & EMA $\beta{=}0.999$ & $73.0 \pm 0.6$ & $76.3 \pm 3.2$ & $61.7 \pm 1.5$ & $96.5 \pm 0.6$ \\
\bottomrule
\end{tabular}
\end{table}

We conclude that normalization is fundamental for the performance of PCD. Disabling it costs between $11.2$ and $23.9$ accuracy points, in exchange for sparsity gains that only reflect the collapse of the network. Second, this ablation indicates that EMA is the appropriate default for this application class, specifically because of the gains in compression. Accuracy is comparable across normalized results, but at $\tau = 0.01$ EMA delivers $+0.7\%$ to $+20.7\%$ of group sparsity at an accuracy cost of at most $0.20$. Finally, these findings suggest that $\beta$ has no material value: across $\beta \in \{0.9, 0.99, 0.999\}$, accuracy varies by $0.16\%$ to $1.93\%$ against seed standard deviations of $0.03$ to $3.20$. $\beta$ is therefore not a calibration constant requiring tuning, and we adopt $\beta = 0.999$ only by convention with adaptive gradient methods. It should also be observed that no normalization is exactly DBGD's  $\beta$-branch update with $\beta = \tau$ (Sec.~\ref{sec:pcd-qp}). The clean-protocol comparison against DBGD is in Sec.~\ref{sec:tuned-baselines}.

\subsection{Baselines under PCD's Normalization}
\label{app:norm-baselines}

We analyze whether the symmetric baselines improve when each objective gradient is normalized with the same scheme PCD uses. We evaluate every baseline of Sec.~\ref{sec:structured-pruning} under EMA normalization ($\beta = 0.999$, $\varepsilon = 10^{-8}$) on ResNet-34/CIFAR-100, over the same hyperparameter grid as the original submission, reporting every operating point (Table~\ref{tab:norm-baselines}). AuxiNash is excluded because it carries its own balance scheme, and FAMO because it is a convergence-efficiency mechanism rather than a conflict-resolution rule. 

\begin{table}[h]
\centering\small
\setlength{\tabcolsep}{5pt}
\caption{Every baseline operating point on ResNet-34/CIFAR-100, raw vs.\ under PCD's EMA normalization (3 seeds). The grid is the same as Sec.~\ref{sec:structured-pruning}'s $w$-grid. Normalization is a calibration effect: large for MGDA ($+32.06$), small or negative elsewhere. $\Delta$ columns are normalized $-$ raw.}
\begin{tabular}{l|cc|cc|rr}
\toprule
 & \multicolumn{2}{c|}{\textbf{Raw}} & \multicolumn{2}{c|}{\textbf{Normalized}} & & \\
\cmidrule(lr){2-3}\cmidrule(lr){4-5}
\textbf{Operating point} & \textbf{Acc.\ (\%)} & \textbf{Spars.\ (\%)} & \textbf{Acc.\ (\%)} & \textbf{Spars.\ (\%)} & $\Delta$\textbf{acc.} & $\Delta$\textbf{spars.} \\
\midrule
WS $w{=}0.1$ & 42.18 & 99.26 & \textbf{47.87} & 98.67 & $+5.69$ & $-0.59$ \\
WS $w{=}0.2$ & 35.00 & 99.45 & 43.18 & 99.56 & $+8.18$ & $+0.11$ \\
WS $w{=}0.4$ & 29.90 & 99.55 & 35.31 & 99.68 & $+5.41$ & $+0.13$ \\
WS $w{=}0.6$ & 28.95 & 99.66 & 24.98 & 99.87 & $-3.97$ & $+0.21$ \\
WS $w{=}0.8$ & 27.37 & 99.74 & 18.69 & 99.91 & $-8.68$ & $+0.17$ \\
WS $w{=}1.0$ & 1.02 & 100.00 & 1.26 & 100.00 & $+0.24$ & $+0.00$ \\
CAGrad $c{=}0.1$ & 31.28 & 99.64 & 31.38 & 99.78 & $+0.09$ & $+0.14$ \\
CAGrad $c{=}0.2$ & 30.38 & 99.70 & 32.20 & 99.75 & $+1.82$ & $+0.05$ \\
CAGrad $c{=}0.4$ & 30.34 & 99.72 & 33.37 & 99.74 & $+3.03$ & $+0.02$ \\
CAGrad $c{=}0.6$ & 34.28 & 99.73 & 26.17 & 99.79 & $-8.10$ & $+0.06$ \\
CAGrad $c{=}0.8$ & 34.06 & 99.65 & 27.22 & 99.65 & $-6.84$ & $+0.00$ \\
MGDA & 1.23 & 100.00 & \textbf{33.29} & 99.73 & $\mathbf{+32.06}$ & $-0.27$ \\
PCGrad & 30.77 & 99.60 & 31.16 & 99.77 & $+0.39$ & $+0.17$ \\
\midrule
PCD $\tau{=}0.01$ & -- & -- & \textbf{72.96} & 76.25 &  &  \\
PCD $\tau{=}0.05$ & -- & -- & 61.70 & 96.52 &  &  \\
PCD $\tau{=}0.1$ & -- & -- & 49.66 & 98.58 &  &  \\
PCD $\tau{=}0.2$ & -- & -- & 46.48 & 99.44 &  &  \\
\bottomrule
\end{tabular}
\label{tab:norm-baselines}
\end{table}

We observe that normalization does benefit the baselines, particularly for MGDA (where we see a jump from $1.23\%$ accuracy at effectively $100\%$ sparsity to $33.29\%$ under EMA normalization). Remaining methods, however, do not show the same gains. In particular, normalization negatively impacts 4 of the 13 operating points across the baselines. The remaining gap in performance can therefore not be explained by normalization, but rather a result of the PCD framework itself. More importantly, we highlight what operating points a baseline method can reach at all: on this grid, the baselines (raw or normalized), occupy a narrow band above $98.67\%$ sparsity, where accuracy has collapsed. PCD's $\tau$, however, traverses $76$--$99\%$ sparsity. Despite the log-$\lambda_{\mathrm{eff}}$ scalarization of Sec.~\ref{sec:tuned-baselines} extending the range for weighted sum, the same cannot be said for other methods.

\subsection{Mechanism Measurements During Training}
\label{app:mechanism}

We further log 5 quantities at every step of all training runs to quantify what PCD does during structured-pruning training (ResNet-34/CIFAR-100, EMA $\beta = 0.999$, averaged over 300 epochs):

\begin{table}[h]
\centering\small
\caption{Per-step mechanism measurements. \emph{Active} = fraction of steps with the secondary constraint active; \emph{slack} = constraint slack before $\to$ after the solve; \emph{direction deviation} = $\norm{\vd^* - \vg_1}/\norm{\vg_1}$; \emph{CE efficiency} = $\cos(\vd^*, \vg_1)$ (cosines are invariant to the positive per-objective scalings, so the same value results from normalized or raw arguments); \emph{conflict} = frequency of $\langle \vg_1, \vg_2\rangle < 0$ and the cosine's typical magnitude.}
\begin{tabular}{c|cccccc}
\toprule
$\tau$ & \textbf{Active} & \textbf{Slack, before $\to$ after} & \textbf{Dir.\ dev.} & \textbf{CE eff.} & \textbf{Conflict freq.} & \textbf{Conflict severity} \\
\midrule
0.01 & 1.000 & $-0.010 \to 4\mathrm{e}{-7}$ & 0.0168 & 1.000 & 0.462 & $5\mathrm{e}{-6}$ \\
0.05 & 1.000 & $-0.050 \to 2\mathrm{e}{-23}$ & 0.0510 & 0.999 & 0.410 & $2\mathrm{e}{-5}$ \\
0.10 & 1.000 & $-0.100 \to 5\mathrm{e}{-22}$ & 0.1067 & 0.994 & 0.445 & $-3\mathrm{e}{-5}$ \\
0.20 & 1.000 & $-0.200 \to -5\mathrm{e}{-22}$ & 0.2187 & 0.976 & 0.465 & $-1\mathrm{e}{-5}$ \\
\bottomrule
\end{tabular}
\label{tab:mechanism}
\end{table}

This ablation reveals firstly that the constraint is active at every step.  A pure cross-entropy step makes zero first-order progress on the group lasso, and any $\tau > 0$ requests more than zero progress as gradients are conflicting at every step. PCD also delivers exactly the progress it specifies: the slack is $-\tau$ before the solve and zero to machine precision after ($\leq 10^{-20}$ in 14 of 16 configurations), which is why sparsity increases monotonically in $\tau$.

Furthermore, the findings suggest that the primary retains priority at every step, with the deployed direction staying within $13^\circ$ of the pure cross-entropy direction even at $\tau = 0.2$. Finally, we observe that despite conflict being detected throughout, it is negligible in magnitude. The sign of the inner product is near-random on a quantity whose cosine never exceeds $0.0011$. There is little conflict to resolve, which is why the conflict-resolution baselines cluster within $2.2\%$ of one another under matched normalization (App.~\ref{app:norm-baselines}), and why the conflict-equilibrium mechanism of Sec.~\ref{sec:synthetic} does not describe these compression experiments.

\subsection{Tuned-Baseline Study: Full Range and Per-Configuration Tables}
\label{app:tuned-expanded}

Fig.~\ref{fig:tuned-full} shows the tuned-baseline study of Sec.~\ref{sec:tuned-baselines} over the full sparsity range. The two structural failures are directly visible. Cooper is unable to enter the compression regime (between $65$--$84\%$ native sparsity). Tightening $\epsilon$ destabilizes before it compresses, and at fixed $\epsilon = 0.30$, a $100\times$ sweep of the dual learning rate moves sparsity only between $43.5$ and $52.8\%$. Similar collapses can be observed for proximal methods. Tables~\ref{tab:tuned-rc100}--\ref{tab:tuned-dc10} report every operating point of every method with seed standard deviations, together with the sparsity-matched PCD reference for each baseline point. Over the full range the record is $44/4/54$. 3 of the 4 losses are light-penalty points at $\leq 13\%$ sparsity (Cooper at $\epsilon = 0.9$ on both DenseNet configurations, proximal at $\lambda = 0.1$ on DenseNet/CIFAR-10, each $+0.3\%$ to $+0.5\%$  over the dense model, acting as regularizers in a regime with no compression content, where PCD's grid places no operating point). 

\begin{figure}[h]
    \centering
    \includegraphics[width=0.9\linewidth]{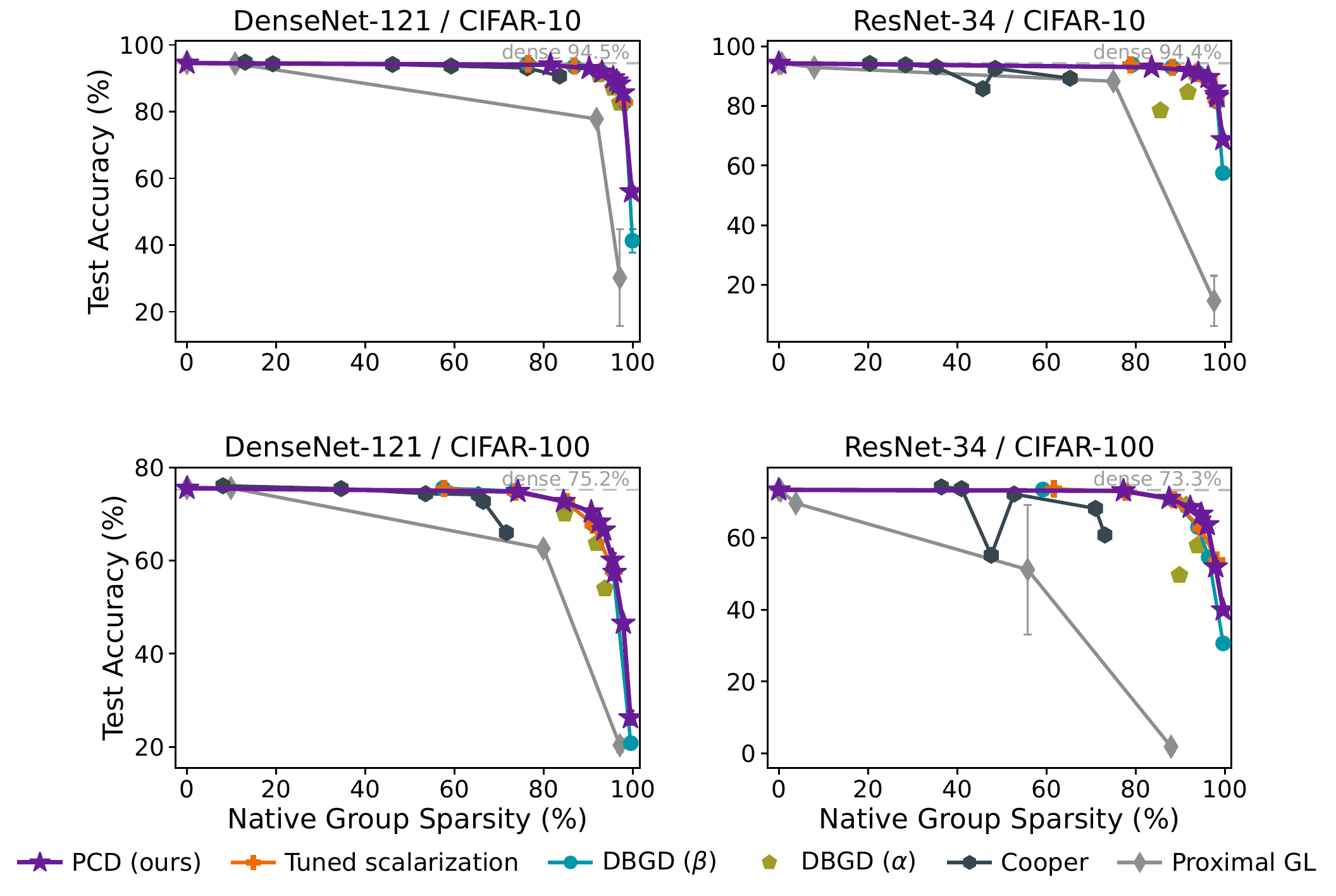}
    \caption{The tuned-baseline study over the full sparsity range (companion to Fig.~\ref{fig:tuned-zoom}, same protocol and thinning). Cooper (\textcolor{CooperC}{slate}) never reaches the compression regime, and proximal group lasso (\textcolor{ProxGLc}{gray}) collapses when pushed there.}
    \label{fig:tuned-full}
\end{figure}

\begin{table}[!htbp]\centering\small
\setlength{\tabcolsep}{2.5pt}\renewcommand{\arraystretch}{0.96}
\caption{Tuned-baseline study, ResNet-34 / CIFAR-100 (dense 73.3\%). Top block: PCD's full 20-point $\tau$ sweep (hyperparameter $=\tau$). Bottom block: every baseline operating point, each paired with its sparsity-matched PCD reference ($\Delta$acc and $\Delta$spars $=$ matched PCD value $-$ baseline value, matching by nearest mean sparsity).}
\begin{tabular}{lr|cc|ccc|rr}
\toprule
 & & \multicolumn{2}{c|}{\textbf{Operating point}} & \multicolumn{3}{c|}{\textbf{Sparsity-matched PCD}} & \multicolumn{2}{c}{\textbf{Difference}} \\
\cmidrule(lr){3-4}\cmidrule(lr){5-7}\cmidrule(lr){8-9}
\textbf{Method} & \textbf{Hyperparam.} & \textbf{Acc.\ (\%)} & \textbf{Spars.\ (\%)} & $\tau$ & \textbf{Acc.\ (\%)} & \textbf{Spars.\ (\%)} & $\Delta$\textbf{acc} & $\Delta$\textbf{spars} \\
\midrule
PCD (ours) & 0 & 73.30 $\pm$ 0.14 & 0.00 $\pm$ 0.00 & -- & -- & -- & -- & -- \\
 & 0.01 & 73.03 $\pm$ 0.49 & 77.26 $\pm$ 1.67 & -- & -- & -- & -- & -- \\
 & 0.02 & 70.92 $\pm$ 0.61 & 87.50 $\pm$ 1.50 & -- & -- & -- & -- & -- \\
 & 0.03 & 68.42 $\pm$ 0.98 & 92.24 $\pm$ 1.11 & -- & -- & -- & -- & -- \\
 & 0.04 & 66.53 $\pm$ 0.71 & 94.70 $\pm$ 0.39 & -- & -- & -- & -- & -- \\
 & 0.05 & 63.60 $\pm$ 2.25 & 96.06 $\pm$ 0.89 & -- & -- & -- & -- & -- \\
 & 0.06 & 58.51 $\pm$ 2.11 & 96.72 $\pm$ 0.67 & -- & -- & -- & -- & -- \\
 & 0.07 & 56.71 $\pm$ 1.81 & 97.22 $\pm$ 0.33 & -- & -- & -- & -- & -- \\
 & 0.08 & 54.47 $\pm$ 2.04 & 97.66 $\pm$ 0.51 & -- & -- & -- & -- & -- \\
 & 0.09 & 51.89 $\pm$ 1.48 & 97.94 $\pm$ 0.93 & -- & -- & -- & -- & -- \\
 & 0.1 & 50.97 $\pm$ 2.02 & 98.54 $\pm$ 0.40 & -- & -- & -- & -- & -- \\
 & 0.2 & 46.33 $\pm$ 0.63 & 99.42 $\pm$ 0.04 & -- & -- & -- & -- & -- \\
 & 0.3 & 42.44 $\pm$ 0.55 & 99.58 $\pm$ 0.04 & -- & -- & -- & -- & -- \\
 & 0.4 & 39.91 $\pm$ 0.64 & 99.62 $\pm$ 0.04 & -- & -- & -- & -- & -- \\
 & 0.5 & 37.50 $\pm$ 1.37 & 99.66 $\pm$ 0.05 & -- & -- & -- & -- & -- \\
 & 0.6 & 36.74 $\pm$ 1.23 & 99.64 $\pm$ 0.05 & -- & -- & -- & -- & -- \\
 & 0.7 & 36.78 $\pm$ 1.41 & 99.66 $\pm$ 0.05 & -- & -- & -- & -- & -- \\
 & 0.8 & 34.97 $\pm$ 1.41 & 99.68 $\pm$ 0.04 & -- & -- & -- & -- & -- \\
 & 0.9 & 35.27 $\pm$ 1.31 & 99.68 $\pm$ 0.04 & -- & -- & -- & -- & -- \\
 & 1 & 31.20 $\pm$ 1.82 & 99.78 $\pm$ 0.04 & -- & -- & -- & -- & -- \\
\midrule
Tuned scalarization & 0.0001 & 73.60 $\pm$ 0.92 & 61.66 $\pm$ 4.83 & 0.01 & 73.03 & 77.26 & -0.57 & +15.60 \\
 & 0.0003 & 72.71 $\pm$ 0.56 & 77.77 $\pm$ 1.49 & 0.01 & 73.03 & 77.26 & +0.32 & -0.51 \\
 & 0.001 & 70.49 $\pm$ 1.19 & 88.51 $\pm$ 1.56 & 0.02 & 70.92 & 87.50 & +0.44 & -1.01 \\
 & 0.003 & 62.53 $\pm$ 1.12 & 94.64 $\pm$ 1.52 & 0.04 & 66.53 & 94.70 & +4.01 & +0.06 \\
 & 0.01 & 53.76 $\pm$ 0.13 & 98.13 $\pm$ 0.54 & 0.09 & 51.89 & 97.94 & -1.87 & -0.19 \\
\midrule
DBGD ($\beta$) & 0.0001 & 73.39 $\pm$ 1.13 & 59.21 $\pm$ 5.99 & 0.01 & 73.03 & 77.26 & -0.36 & +18.05 \\
 & 0.0003 & 72.98 $\pm$ 0.50 & 77.67 $\pm$ 1.16 & 0.01 & 73.03 & 77.26 & +0.05 & -0.41 \\
 & 0.001 & 70.44 $\pm$ 0.08 & 88.52 $\pm$ 0.43 & 0.02 & 70.92 & 87.50 & +0.48 & -1.02 \\
 & 0.003 & 62.96 $\pm$ 0.53 & 93.96 $\pm$ 0.22 & 0.04 & 66.53 & 94.70 & +3.57 & +0.74 \\
 & 0.01 & 54.58 $\pm$ 0.23 & 96.36 $\pm$ 1.97 & 0.05 & 63.60 & 96.06 & +9.02 & -0.30 \\
 & 1 & 30.60 $\pm$ 1.60 & 99.62 $\pm$ 0.04 & 0.4 & 39.91 & 99.62 & +9.31 & -0.00 \\
\midrule
DBGD ($\alpha$) & 0.01 & 69.08 $\pm$ 0.34 & 91.35 $\pm$ 0.69 & 0.03 & 68.42 & 92.24 & -0.66 & +0.89 \\
 & 0.1 & 57.81 $\pm$ 0.28 & 93.83 $\pm$ 0.38 & 0.04 & 66.53 & 94.70 & +8.73 & +0.87 \\
 & 1 & 49.56 $\pm$ 1.37 & 89.82 $\pm$ 0.66 & 0.02 & 70.92 & 87.50 & +21.37 & -2.32 \\
\midrule
Cooper & 0.02 & 55.12 $\pm$ 0.05 & 47.63 $\pm$ 8.15 & 0.01 & 73.03 & 77.26 & +17.91 & +29.63 \\
 & 0.05 & 60.74 $\pm$ 1.19 & 73.10 $\pm$ 6.91 & 0.01 & 73.03 & 77.26 & +12.29 & +4.16 \\
 & 0.15 & 68.13 $\pm$ 0.47 & 71.00 $\pm$ 3.51 & 0.01 & 73.03 & 77.26 & +4.89 & +6.26 \\
 & 0.3 & 72.13 $\pm$ 0.52 & 52.76 $\pm$ 7.94 & 0.01 & 73.03 & 77.26 & +0.90 & +24.50 \\
 & 0.5 & 73.63 $\pm$ 0.87 & 40.98 $\pm$ 3.73 & 0.01 & 73.03 & 77.26 & -0.61 & +36.28 \\
 & 0.9 & 74.11 $\pm$ 0.84 & 36.46 $\pm$ 0.86 & 0 & 73.30 & 0.00 & -0.81 & -36.46 \\
\midrule
Cooper (dual lr) & 1e-05 & 72.72 $\pm$ 0.41 & 43.46 $\pm$ 2.04 & 0.01 & 73.03 & 77.26 & +0.30 & +33.80 \\
 & 0.001 & 71.16 $\pm$ 0.92 & 46.80 $\pm$ 5.37 & 0.01 & 73.03 & 77.26 & +1.87 & +30.46 \\
\midrule
Proximal GL & 0.01 & 73.35 $\pm$ 0.26 & 0.00 $\pm$ 0.00 & 0 & 73.30 & 0.00 & -0.06 & +0.00 \\
 & 0.1 & 73.09 $\pm$ 0.45 & 0.42 $\pm$ 0.03 & 0 & 73.30 & 0.00 & +0.20 & -0.42 \\
 & 0.3 & 69.57 $\pm$ 0.14 & 3.84 $\pm$ 0.41 & 0 & 73.30 & 0.00 & +3.73 & -3.84 \\
 & 1 & 51.11 $\pm$ 18.00 & 55.80 $\pm$ 9.61 & 0.01 & 73.03 & 77.26 & +21.92 & +21.46 \\
 & 3 & 1.87 $\pm$ 1.50 & 87.94 $\pm$ 16.01 & 0.02 & 70.92 & 87.50 & +69.06 & -0.44 \\
\bottomrule\end{tabular}
\label{tab:tuned-rc100}\end{table}

\begin{table}[!htbp]\centering\small
\setlength{\tabcolsep}{2.5pt}\renewcommand{\arraystretch}{0.96}
\caption{Tuned-baseline study, ResNet-34 / CIFAR-10 (dense 94.4\%). Top block: PCD's full 20-point $\tau$ sweep (hyperparameter $=\tau$). Bottom block: every baseline operating point, each paired with its sparsity-matched PCD reference ($\Delta$acc and $\Delta$spars $=$ matched PCD value $-$ baseline value, matching by nearest mean sparsity).}
\begin{tabular}{lr|cc|ccc|rr}
\toprule
 & & \multicolumn{2}{c|}{\textbf{Operating point}} & \multicolumn{3}{c|}{\textbf{Sparsity-matched PCD}} & \multicolumn{2}{c}{\textbf{Difference}} \\
\cmidrule(lr){3-4}\cmidrule(lr){5-7}\cmidrule(lr){8-9}
\textbf{Method} & \textbf{Hyperparam.} & \textbf{Acc.\ (\%)} & \textbf{Spars.\ (\%)} & $\tau$ & \textbf{Acc.\ (\%)} & \textbf{Spars.\ (\%)} & $\Delta$\textbf{acc} & $\Delta$\textbf{spars} \\
\midrule
PCD (ours) & 0 & 94.37 $\pm$ 0.08 & 0.00 $\pm$ 0.00 & -- & -- & -- & -- & -- \\
 & 0.01 & 93.02 $\pm$ 0.33 & 83.56 $\pm$ 1.51 & -- & -- & -- & -- & -- \\
 & 0.02 & 91.97 $\pm$ 0.14 & 91.84 $\pm$ 0.11 & -- & -- & -- & -- & -- \\
 & 0.03 & 90.95 $\pm$ 0.46 & 94.04 $\pm$ 0.19 & -- & -- & -- & -- & -- \\
 & 0.04 & 90.13 $\pm$ 0.20 & 95.40 $\pm$ 0.12 & -- & -- & -- & -- & -- \\
 & 0.05 & 89.62 $\pm$ 0.26 & 96.24 $\pm$ 0.09 & -- & -- & -- & -- & -- \\
 & 0.06 & 88.14 $\pm$ 0.36 & 96.84 $\pm$ 0.09 & -- & -- & -- & -- & -- \\
 & 0.07 & 86.76 $\pm$ 0.34 & 97.34 $\pm$ 0.11 & -- & -- & -- & -- & -- \\
 & 0.08 & 85.75 $\pm$ 0.72 & 97.66 $\pm$ 0.11 & -- & -- & -- & -- & -- \\
 & 0.09 & 83.94 $\pm$ 0.69 & 98.14 $\pm$ 0.18 & -- & -- & -- & -- & -- \\
 & 0.1 & 83.12 $\pm$ 0.49 & 98.26 $\pm$ 0.15 & -- & -- & -- & -- & -- \\
 & 0.2 & 73.52 $\pm$ 0.78 & 99.28 $\pm$ 0.04 & -- & -- & -- & -- & -- \\
 & 0.3 & 68.74 $\pm$ 1.53 & 99.54 $\pm$ 0.05 & -- & -- & -- & -- & -- \\
 & 0.4 & 67.35 $\pm$ 1.68 & 99.54 $\pm$ 0.05 & -- & -- & -- & -- & -- \\
 & 0.5 & 65.19 $\pm$ 2.50 & 99.60 $\pm$ 0.00 & -- & -- & -- & -- & -- \\
 & 0.6 & 64.61 $\pm$ 2.66 & 99.62 $\pm$ 0.08 & -- & -- & -- & -- & -- \\
 & 0.7 & 63.63 $\pm$ 1.48 & 99.64 $\pm$ 0.05 & -- & -- & -- & -- & -- \\
 & 0.8 & 61.37 $\pm$ 1.70 & 99.70 $\pm$ 0.00 & -- & -- & -- & -- & -- \\
 & 0.9 & 58.75 $\pm$ 2.15 & 99.76 $\pm$ 0.05 & -- & -- & -- & -- & -- \\
 & 1 & 60.16 $\pm$ 2.60 & 99.68 $\pm$ 0.04 & -- & -- & -- & -- & -- \\
\midrule
Tuned scalarization & 0.0001 & 93.93 $\pm$ 0.12 & 78.91 $\pm$ 0.50 & 0.01 & 93.02 & 83.56 & -0.91 & +4.65 \\
 & 0.0003 & 93.02 $\pm$ 0.14 & 88.27 $\pm$ 0.24 & 0.02 & 91.97 & 91.84 & -1.05 & +3.57 \\
 & 0.001 & 90.98 $\pm$ 0.08 & 93.74 $\pm$ 0.12 & 0.03 & 90.95 & 94.04 & -0.03 & +0.30 \\
 & 0.003 & 88.73 $\pm$ 0.22 & 96.41 $\pm$ 0.07 & 0.05 & 89.62 & 96.24 & +0.89 & -0.17 \\
 & 0.01 & 82.13 $\pm$ 0.18 & 98.10 $\pm$ 0.54 & 0.09 & 83.94 & 98.14 & +1.80 & +0.04 \\
\midrule
DBGD ($\beta$) & 0.0001 & 93.91 $\pm$ 0.18 & 79.28 $\pm$ 0.38 & 0.01 & 93.02 & 83.56 & -0.89 & +4.28 \\
 & 0.0003 & 93.02 $\pm$ 0.07 & 88.13 $\pm$ 0.08 & 0.02 & 91.97 & 91.84 & -1.05 & +3.71 \\
 & 0.001 & 91.21 $\pm$ 0.22 & 93.86 $\pm$ 0.07 & 0.03 & 90.95 & 94.04 & -0.26 & +0.18 \\
 & 0.003 & 88.88 $\pm$ 0.03 & 96.43 $\pm$ 0.05 & 0.05 & 89.62 & 96.24 & +0.75 & -0.19 \\
 & 0.01 & 81.50 $\pm$ 0.42 & 98.22 $\pm$ 0.08 & 0.1 & 83.12 & 98.26 & +1.62 & +0.04 \\
 & 1 & 57.51 $\pm$ 1.27 & 99.56 $\pm$ 0.05 & 0.3 & 68.74 & 99.54 & +11.24 & -0.02 \\
\midrule
DBGD ($\alpha$) & 0.01 & 91.11 $\pm$ 0.23 & 93.27 $\pm$ 0.99 & 0.03 & 90.95 & 94.04 & -0.16 & +0.77 \\
 & 0.1 & 84.64 $\pm$ 0.74 & 91.71 $\pm$ 2.86 & 0.02 & 91.97 & 91.84 & +7.34 & +0.13 \\
 & 1 & 78.49 $\pm$ 1.35 & 85.54 $\pm$ 3.96 & 0.01 & 93.02 & 83.56 & +14.53 & -1.98 \\
\midrule
Cooper & 0.02 & 85.82 $\pm$ 0.16 & 45.74 $\pm$ 2.85 & 0.01 & 93.02 & 83.56 & +7.19 & +37.82 \\
 & 0.05 & 89.32 $\pm$ 0.51 & 65.31 $\pm$ 14.12 & 0.01 & 93.02 & 83.56 & +3.69 & +18.25 \\
 & 0.15 & 92.65 $\pm$ 0.42 & 48.51 $\pm$ 4.51 & 0.01 & 93.02 & 83.56 & +0.36 & +35.05 \\
 & 0.3 & 93.14 $\pm$ 0.31 & 35.31 $\pm$ 1.64 & 0 & 94.37 & 0.00 & +1.22 & -35.31 \\
 & 0.5 & 93.84 $\pm$ 0.21 & 28.42 $\pm$ 0.75 & 0 & 94.37 & 0.00 & +0.52 & -28.42 \\
 & 0.9 & 94.28 $\pm$ 0.03 & 20.40 $\pm$ 1.49 & 0 & 94.37 & 0.00 & +0.08 & -20.40 \\
\midrule
Proximal GL & 0.01 & 94.24 $\pm$ 0.04 & 0.00 $\pm$ 0.00 & 0 & 94.37 & 0.00 & +0.13 & +0.00 \\
 & 0.1 & 94.38 $\pm$ 0.18 & 0.66 $\pm$ 0.05 & 0 & 94.37 & 0.00 & -0.02 & -0.66 \\
 & 0.3 & 92.97 $\pm$ 0.08 & 7.95 $\pm$ 2.65 & 0 & 94.37 & 0.00 & +1.40 & -7.95 \\
 & 1 & 88.32 $\pm$ 0.15 & 75.00 $\pm$ 4.49 & 0.01 & 93.02 & 83.56 & +4.70 & +8.56 \\
 & 3 & 14.59 $\pm$ 8.46 & 97.57 $\pm$ 2.55 & 0.08 & 85.75 & 97.66 & +71.15 & +0.09 \\
\bottomrule\end{tabular}
\label{tab:tuned-rc10}\end{table}

\begin{table}[!htbp]\centering\small
\setlength{\tabcolsep}{2.5pt}\renewcommand{\arraystretch}{0.96}
\caption{Tuned-baseline study, DenseNet-121 / CIFAR-100 (dense 75.2\%). Top block: PCD's full 20-point $\tau$ sweep (hyperparameter $=\tau$). Bottom block: every baseline operating point, each paired with its sparsity-matched PCD reference ($\Delta$acc and $\Delta$spars $=$ matched PCD value $-$ baseline value, matching by nearest mean sparsity).}
\begin{tabular}{lr|cc|ccc|rr}
\toprule
 & & \multicolumn{2}{c|}{\textbf{Operating point}} & \multicolumn{3}{c|}{\textbf{Sparsity-matched PCD}} & \multicolumn{2}{c}{\textbf{Difference}} \\
\cmidrule(lr){3-4}\cmidrule(lr){5-7}\cmidrule(lr){8-9}
\textbf{Method} & \textbf{Hyperparam.} & \textbf{Acc.\ (\%)} & \textbf{Spars.\ (\%)} & $\tau$ & \textbf{Acc.\ (\%)} & \textbf{Spars.\ (\%)} & $\Delta$\textbf{acc} & $\Delta$\textbf{spars} \\
\midrule
PCD (ours) & 0 & 75.52 $\pm$ 0.35 & 0.00 $\pm$ 0.00 & -- & -- & -- & -- & -- \\
 & 0.01 & 74.82 $\pm$ 0.28 & 74.20 $\pm$ 0.89 & -- & -- & -- & -- & -- \\
 & 0.02 & 72.69 $\pm$ 0.45 & 84.42 $\pm$ 0.56 & -- & -- & -- & -- & -- \\
 & 0.03 & 71.57 $\pm$ 0.64 & 88.28 $\pm$ 0.52 & -- & -- & -- & -- & -- \\
 & 0.04 & 70.45 $\pm$ 0.39 & 90.64 $\pm$ 0.45 & -- & -- & -- & -- & -- \\
 & 0.05 & 68.27 $\pm$ 0.61 & 92.42 $\pm$ 0.33 & -- & -- & -- & -- & -- \\
 & 0.06 & 66.54 $\pm$ 0.81 & 93.42 $\pm$ 0.33 & -- & -- & -- & -- & -- \\
 & 0.07 & 64.88 $\pm$ 0.81 & 94.14 $\pm$ 0.26 & -- & -- & -- & -- & -- \\
 & 0.08 & 62.14 $\pm$ 1.11 & 94.94 $\pm$ 0.23 & -- & -- & -- & -- & -- \\
 & 0.09 & 60.11 $\pm$ 1.45 & 95.42 $\pm$ 0.16 & -- & -- & -- & -- & -- \\
 & 0.1 & 57.46 $\pm$ 1.65 & 95.86 $\pm$ 0.18 & -- & -- & -- & -- & -- \\
 & 0.2 & 46.55 $\pm$ 0.51 & 97.82 $\pm$ 0.04 & -- & -- & -- & -- & -- \\
 & 0.3 & 41.94 $\pm$ 0.35 & 98.42 $\pm$ 0.04 & -- & -- & -- & -- & -- \\
 & 0.4 & 38.08 $\pm$ 0.82 & 98.78 $\pm$ 0.04 & -- & -- & -- & -- & -- \\
 & 0.5 & 35.66 $\pm$ 0.48 & 98.98 $\pm$ 0.04 & -- & -- & -- & -- & -- \\
 & 0.6 & 33.97 $\pm$ 0.65 & 99.12 $\pm$ 0.04 & -- & -- & -- & -- & -- \\
 & 0.7 & 31.42 $\pm$ 0.27 & 99.20 $\pm$ 0.07 & -- & -- & -- & -- & -- \\
 & 0.8 & 29.95 $\pm$ 0.84 & 99.28 $\pm$ 0.04 & -- & -- & -- & -- & -- \\
 & 0.9 & 28.40 $\pm$ 0.49 & 99.32 $\pm$ 0.04 & -- & -- & -- & -- & -- \\
 & 1 & 26.26 $\pm$ 1.51 & 99.42 $\pm$ 0.04 & -- & -- & -- & -- & -- \\
\midrule
Tuned scalarization & 0.0001 & 75.42 $\pm$ 0.28 & 57.64 $\pm$ 0.06 & 0.01 & 74.82 & 74.20 & -0.60 & +16.56 \\
 & 0.0003 & 74.60 $\pm$ 0.12 & 73.62 $\pm$ 0.22 & 0.01 & 74.82 & 74.20 & +0.22 & +0.58 \\
 & 0.001 & 72.65 $\pm$ 0.67 & 85.13 $\pm$ 0.12 & 0.02 & 72.69 & 84.42 & +0.04 & -0.71 \\
 & 0.003 & 67.57 $\pm$ 0.58 & 91.02 $\pm$ 0.41 & 0.04 & 70.45 & 90.64 & +2.88 & -0.38 \\
 & 0.01 & 57.64 $\pm$ 0.39 & 95.65 $\pm$ 0.18 & 0.1 & 57.46 & 95.86 & -0.19 & +0.21 \\
\midrule
DBGD ($\beta$) & 0.0001 & 75.53 $\pm$ 0.28 & 57.39 $\pm$ 0.08 & 0.01 & 74.82 & 74.20 & -0.71 & +16.81 \\
 & 0.0003 & 74.94 $\pm$ 0.06 & 73.50 $\pm$ 0.08 & 0.01 & 74.82 & 74.20 & -0.12 & +0.70 \\
 & 0.001 & 72.47 $\pm$ 0.22 & 84.98 $\pm$ 0.09 & 0.02 & 72.69 & 84.42 & +0.21 & -0.56 \\
 & 0.003 & 67.90 $\pm$ 0.15 & 90.97 $\pm$ 0.14 & 0.04 & 70.45 & 90.64 & +2.56 & -0.33 \\
 & 0.01 & 57.71 $\pm$ 0.79 & 95.57 $\pm$ 0.01 & 0.09 & 60.11 & 95.42 & +2.40 & -0.15 \\
 & 1 & 20.80 $\pm$ 0.39 & 99.49 $\pm$ 0.06 & 1 & 26.26 & 99.42 & +5.46 & -0.07 \\
\midrule
DBGD ($\alpha$) & 0.01 & 70.02 $\pm$ 0.66 & 84.68 $\pm$ 0.25 & 0.02 & 72.69 & 84.42 & +2.66 & -0.26 \\
 & 0.1 & 63.67 $\pm$ 0.17 & 91.82 $\pm$ 0.26 & 0.05 & 68.27 & 92.42 & +4.60 & +0.60 \\
 & 1 & 53.97 $\pm$ 0.19 & 93.69 $\pm$ 2.68 & 0.06 & 66.54 & 93.42 & +12.57 & -0.27 \\
\midrule
Cooper & 0.02 & 65.97 $\pm$ 0.83 & 71.61 $\pm$ 4.12 & 0.01 & 74.82 & 74.20 & +8.85 & +2.59 \\
 & 0.05 & 72.70 $\pm$ 0.19 & 66.43 $\pm$ 3.64 & 0.01 & 74.82 & 74.20 & +2.12 & +7.77 \\
 & 0.15 & 74.14 $\pm$ 0.92 & 65.32 $\pm$ 3.29 & 0.01 & 74.82 & 74.20 & +0.68 & +8.88 \\
 & 0.3 & 74.35 $\pm$ 0.46 & 53.50 $\pm$ 0.31 & 0.01 & 74.82 & 74.20 & +0.48 & +20.70 \\
 & 0.5 & 75.44 $\pm$ 0.37 & 34.56 $\pm$ 0.05 & 0 & 75.52 & 0.00 & +0.08 & -34.56 \\
 & 0.9 & 76.05 $\pm$ 0.11 & 8.08 $\pm$ 0.10 & 0 & 75.52 & 0.00 & -0.53 & -8.08 \\
\midrule
Proximal GL & 0.01 & 75.56 $\pm$ 0.13 & 0.00 $\pm$ 0.00 & 0 & 75.52 & 0.00 & -0.04 & +0.00 \\
 & 0.1 & 75.77 $\pm$ 0.11 & 0.09 $\pm$ 0.04 & 0 & 75.52 & 0.00 & -0.24 & -0.09 \\
 & 0.3 & 75.57 $\pm$ 0.34 & 9.87 $\pm$ 0.35 & 0 & 75.52 & 0.00 & -0.05 & -9.87 \\
 & 1 & 62.59 $\pm$ 1.13 & 79.91 $\pm$ 1.35 & 0.02 & 72.69 & 84.42 & +10.10 & +4.51 \\
 & 3 & 20.35 $\pm$ 1.48 & 97.05 $\pm$ 0.60 & 0.2 & 46.55 & 97.82 & +26.20 & +0.77 \\
\bottomrule\end{tabular}
\label{tab:tuned-dc100}\end{table}

\begin{table}[!htbp]\centering\small
\setlength{\tabcolsep}{2.5pt}\renewcommand{\arraystretch}{0.96}
\caption{Tuned-baseline study, DenseNet-121 / CIFAR-10 (dense 94.5\%). Top block: PCD's full 20-point $\tau$ sweep (hyperparameter $=\tau$). Bottom block: every baseline operating point, each paired with its sparsity-matched PCD reference ($\Delta$acc and $\Delta$spars $=$ matched PCD value $-$ baseline value, matching by nearest mean sparsity).}
\begin{tabular}{lr|cc|ccc|rr}
\toprule
 & & \multicolumn{2}{c|}{\textbf{Operating point}} & \multicolumn{3}{c|}{\textbf{Sparsity-matched PCD}} & \multicolumn{2}{c}{\textbf{Difference}} \\
\cmidrule(lr){3-4}\cmidrule(lr){5-7}\cmidrule(lr){8-9}
\textbf{Method} & \textbf{Hyperparam.} & \textbf{Acc.\ (\%)} & \textbf{Spars.\ (\%)} & $\tau$ & \textbf{Acc.\ (\%)} & \textbf{Spars.\ (\%)} & $\Delta$\textbf{acc} & $\Delta$\textbf{spars} \\
\midrule
PCD (ours) & 0 & 94.54 $\pm$ 0.14 & 0.00 $\pm$ 0.00 & -- & -- & -- & -- & -- \\
 & 0.01 & 94.03 $\pm$ 0.19 & 81.46 $\pm$ 0.71 & -- & -- & -- & -- & -- \\
 & 0.02 & 92.92 $\pm$ 0.19 & 90.18 $\pm$ 0.29 & -- & -- & -- & -- & -- \\
 & 0.03 & 91.83 $\pm$ 0.33 & 93.04 $\pm$ 0.23 & -- & -- & -- & -- & -- \\
 & 0.04 & 90.94 $\pm$ 0.30 & 94.50 $\pm$ 0.20 & -- & -- & -- & -- & -- \\
 & 0.05 & 90.13 $\pm$ 0.41 & 95.54 $\pm$ 0.17 & -- & -- & -- & -- & -- \\
 & 0.06 & 89.18 $\pm$ 0.48 & 96.34 $\pm$ 0.15 & -- & -- & -- & -- & -- \\
 & 0.07 & 88.31 $\pm$ 0.65 & 96.86 $\pm$ 0.18 & -- & -- & -- & -- & -- \\
 & 0.08 & 87.60 $\pm$ 0.53 & 97.26 $\pm$ 0.13 & -- & -- & -- & -- & -- \\
 & 0.09 & 86.76 $\pm$ 0.40 & 97.50 $\pm$ 0.07 & -- & -- & -- & -- & -- \\
 & 0.1 & 85.65 $\pm$ 0.92 & 97.74 $\pm$ 0.19 & -- & -- & -- & -- & -- \\
 & 0.2 & 75.91 $\pm$ 1.64 & 98.86 $\pm$ 0.05 & -- & -- & -- & -- & -- \\
 & 0.3 & 68.07 $\pm$ 0.83 & 99.24 $\pm$ 0.05 & -- & -- & -- & -- & -- \\
 & 0.4 & 65.46 $\pm$ 1.27 & 99.38 $\pm$ 0.04 & -- & -- & -- & -- & -- \\
 & 0.5 & 62.73 $\pm$ 0.80 & 99.50 $\pm$ 0.00 & -- & -- & -- & -- & -- \\
 & 0.6 & 61.05 $\pm$ 0.53 & 99.56 $\pm$ 0.05 & -- & -- & -- & -- & -- \\
 & 0.7 & 59.31 $\pm$ 0.26 & 99.62 $\pm$ 0.04 & -- & -- & -- & -- & -- \\
 & 0.8 & 56.92 $\pm$ 1.25 & 99.66 $\pm$ 0.05 & -- & -- & -- & -- & -- \\
 & 0.9 & 55.96 $\pm$ 1.50 & 99.70 $\pm$ 0.00 & -- & -- & -- & -- & -- \\
 & 1 & 54.98 $\pm$ 0.82 & 99.70 $\pm$ 0.00 & -- & -- & -- & -- & -- \\
\midrule
Tuned scalarization & 0.0001 & 94.31 $\pm$ 0.20 & 76.48 $\pm$ 0.14 & 0.01 & 94.03 & 81.46 & -0.28 & +4.98 \\
 & 0.0003 & 93.59 $\pm$ 0.12 & 86.86 $\pm$ 0.14 & 0.02 & 92.92 & 90.18 & -0.67 & +3.32 \\
 & 0.001 & 92.10 $\pm$ 0.42 & 92.95 $\pm$ 0.09 & 0.03 & 91.83 & 93.04 & -0.27 & +0.09 \\
 & 0.003 & 88.79 $\pm$ 0.27 & 96.20 $\pm$ 0.05 & 0.06 & 89.18 & 96.34 & +0.39 & +0.14 \\
 & 0.01 & 82.75 $\pm$ 0.14 & 98.13 $\pm$ 0.02 & 0.1 & 85.65 & 97.74 & +2.90 & -0.39 \\
\midrule
DBGD ($\beta$) & 0.0001 & 94.32 $\pm$ 0.06 & 76.37 $\pm$ 0.25 & 0.01 & 94.03 & 81.46 & -0.29 & +5.09 \\
 & 0.0003 & 93.53 $\pm$ 0.09 & 86.82 $\pm$ 0.08 & 0.02 & 92.92 & 90.18 & -0.61 & +3.36 \\
 & 0.001 & 91.90 $\pm$ 0.21 & 92.97 $\pm$ 0.04 & 0.03 & 91.83 & 93.04 & -0.07 & +0.07 \\
 & 0.003 & 88.75 $\pm$ 0.28 & 96.24 $\pm$ 0.05 & 0.06 & 89.18 & 96.34 & +0.43 & +0.10 \\
 & 0.01 & 82.87 $\pm$ 0.48 & 98.15 $\pm$ 0.10 & 0.1 & 85.65 & 97.74 & +2.78 & -0.41 \\
 & 1 & 41.29 $\pm$ 3.52 & 99.87 $\pm$ 0.02 & 0.9 & 55.96 & 99.70 & +14.67 & -0.17 \\
\midrule
DBGD ($\alpha$) & 0.01 & 91.09 $\pm$ 0.31 & 92.33 $\pm$ 0.10 & 0.03 & 91.83 & 93.04 & +0.73 & +0.71 \\
 & 0.1 & 87.06 $\pm$ 0.12 & 95.48 $\pm$ 0.12 & 0.05 & 90.13 & 95.54 & +3.06 & +0.06 \\
 & 1 & 82.57 $\pm$ 0.34 & 97.03 $\pm$ 0.19 & 0.07 & 88.31 & 96.86 & +5.75 & -0.17 \\
\midrule
Cooper & 0.02 & 90.66 $\pm$ 0.21 & 83.53 $\pm$ 1.14 & 0.01 & 94.03 & 81.46 & +3.38 & -2.07 \\
 & 0.05 & 93.10 $\pm$ 0.24 & 76.26 $\pm$ 2.95 & 0.01 & 94.03 & 81.46 & +0.94 & +5.20 \\
 & 0.15 & 93.67 $\pm$ 0.13 & 59.24 $\pm$ 12.34 & 0.01 & 94.03 & 81.46 & +0.36 & +22.22 \\
 & 0.3 & 94.13 $\pm$ 0.10 & 46.09 $\pm$ 3.17 & 0.01 & 94.03 & 81.46 & -0.10 & +35.37 \\
 & 0.5 & 94.32 $\pm$ 0.12 & 19.26 $\pm$ 1.44 & 0 & 94.54 & 0.00 & +0.22 & -19.26 \\
 & 0.9 & 94.81 $\pm$ 0.13 & 13.05 $\pm$ 0.37 & 0 & 94.54 & 0.00 & -0.27 & -13.05 \\
\midrule
Proximal GL & 0.01 & 94.54 $\pm$ 0.17 & 0.00 $\pm$ 0.00 & 0 & 94.54 & 0.00 & -0.00 & +0.00 \\
 & 0.1 & 94.82 $\pm$ 0.12 & 0.03 $\pm$ 0.02 & 0 & 94.54 & 0.00 & -0.28 & -0.03 \\
 & 0.3 & 94.39 $\pm$ 0.13 & 10.82 $\pm$ 0.70 & 0 & 94.54 & 0.00 & +0.14 & -10.82 \\
 & 1 & 77.80 $\pm$ 1.05 & 91.83 $\pm$ 1.56 & 0.03 & 91.83 & 93.04 & +14.03 & +1.21 \\
 & 3 & 30.20 $\pm$ 14.48 & 97.05 $\pm$ 0.94 & 0.07 & 88.31 & 96.86 & +58.12 & -0.19 \\
\bottomrule\end{tabular}
\label{tab:tuned-dc10}\end{table}

\clearpage
\section{Expanded Results: Unstructured Sparsity and Low-Rankness}
\label{app:k3-expanded}

Firstly, we refer readers to Table~\ref{tab:k3-sd} for seed variability results that complement Table~\ref{tab:lowrank}. 

\begin{table}[h]\centering\small
\caption{Seed variability (mean $\pm$ sd over 3 seeds) at the operating points of Table~\ref{tab:lowrank}: PCD at fixed $\tau = 0.02$, each baseline at its highest-mean-test-accuracy point. Accuracy and effective rank are the identical quantities to Table~\ref{tab:lowrank}; sparsity here is the \emph{native} final-epoch value, which differs from the thresholded protocol by up to a few points. Effective rank is the least stable axis for every method.}
\begin{tabular}{ll|ccc}
\toprule
\textbf{Configuration} & \textbf{Method} & \textbf{Acc.\ (\%)} & \textbf{Spars.\ (\%, native)} & \textbf{Eff.\ rank} \\
\midrule
\multirow{7}{*}{ResNet-34 / C-10} & PCD & $91.8 \pm 0.16$ & $96.1 \pm 0.15$ & $30.9 \pm 0.35$ \\
 & WS & $63.8 \pm 0.79$ & $91.6 \pm 0.04$ & $155.6 \pm 8.72$ \\
 & CAGrad & $64.8 \pm 0.32$ & $99.4 \pm 0.12$ & $159.6 \pm 4.35$ \\
 & MGDA & $13.5 \pm 6.14$ & $100.0 \pm 0.02$ & $226.4 \pm 2.59$ \\
 & PCGrad & $51.4 \pm 1.74$ & $94.0 \pm 0.50$ & $166.8 \pm 18.94$ \\
 & FAMO & $43.2 \pm 2.84$ & $99.8 \pm 0.16$ & $185.8 \pm 3.81$ \\
 & AuxiNash & $75.7 \pm 0.41$ & $99.9 \pm 0.01$ & $153.1 \pm 15.62$ \\
\midrule
\multirow{7}{*}{ResNet-34 / C-100} & PCD & $71.0 \pm 0.15$ & $93.7 \pm 0.79$ & $24.5 \pm 9.70$ \\
 & WS & $34.2 \pm 2.41$ & $91.6 \pm 0.02$ & $137.9 \pm 23.51$ \\
 & CAGrad & $36.0 \pm 1.04$ & $100.0 \pm 0.01$ & $186.3 \pm 17.41$ \\
 & MGDA & $1.1 \pm 0.39$ & $100.0 \pm 0.00$ & $227.6 \pm 0.13$ \\
 & PCGrad & $30.0 \pm 0.51$ & $93.3 \pm 0.66$ & $160.0 \pm 4.76$ \\
 & FAMO & $20.3 \pm 2.36$ & $99.5 \pm 0.81$ & $190.5 \pm 4.07$ \\
 & AuxiNash & $48.7 \pm 1.16$ & $99.9 \pm 0.02$ & $184.8 \pm 33.41$ \\
\midrule
\multirow{7}{*}{Inception / C-10} & PCD & $89.2 \pm 0.40$ & $86.8 \pm 0.89$ & $31.3 \pm 9.77$ \\
 & WS & $52.1 \pm 0.69$ & $91.6 \pm 0.06$ & $39.7 \pm 3.79$ \\
 & CAGrad & $38.2 \pm 2.97$ & $99.9 \pm 0.01$ & $51.2 \pm 7.71$ \\
 & MGDA & $23.1 \pm 0.22$ & $100.0 \pm 0.01$ & $76.4 \pm 0.02$ \\
 & PCGrad & $33.8 \pm 2.07$ & $92.4 \pm 0.99$ & $55.8 \pm 11.89$ \\
 & FAMO & $10.0 \pm 0.00$ & $100.0 \pm 0.00$ & $100.3 \pm 0.15$ \\
 & AuxiNash & $60.3 \pm 0.39$ & $99.8 \pm 0.02$ & $18.0 \pm 2.46$ \\
\midrule
\multirow{7}{*}{Inception / C-100} & PCD & $67.8 \pm 0.32$ & $80.6 \pm 1.98$ & $27.7 \pm 4.95$ \\
 & WS & $29.8 \pm 0.59$ & $91.6 \pm 0.14$ & $28.1 \pm 7.77$ \\
 & CAGrad & $14.1 \pm 1.08$ & $99.9 \pm 0.04$ & $59.0 \pm 11.78$ \\
 & MGDA & $4.3 \pm 0.36$ & $100.0 \pm 0.01$ & $70.3 \pm 5.16$ \\
 & PCGrad & $12.4 \pm 0.28$ & $91.9 \pm 0.55$ & $52.4 \pm 12.08$ \\
 & FAMO & $1.0 \pm 0.00$ & $100.0 \pm 0.00$ & $100.3 \pm 0.18$ \\
 & AuxiNash & $31.9 \pm 1.43$ & $98.5 \pm 1.35$ & $12.1 \pm 8.71$ \\
\bottomrule
\end{tabular}
\label{tab:k3-sd}
\end{table}

\subsection{Matched-Sparsity Comparison}

Accuracy is only comparable at equal compression. Building each method's accuracy-vs-(unstructured)-sparsity curve from the 3-seed means and reading every curve on a common grid with no extrapolation gives PCD's accuracy advantage at matched unstructured sparsity (Table~\ref{tab:k3-matched}). MGDA, PCGrad, and FAMO never reach PCD's sparsity range, so no matched comparison with them is defined. The same holds for the weighted sum on both Inception configurations. The effective-rank axis is weaker and does not support a dominance claim: the matched-range minimum is negative against AuxiNash on 3 of the 4 configurations, consistent with the Inception rank results in Table~\ref{tab:lowrank}.

\begin{table}[h]
\centering\small
\caption{PCD's accuracy advantage (percentage points) at matched \emph{unstructured} sparsity, per configuration and baseline, ranges over the common sparsity grid.}
\begin{tabular}{l|ccc}
\toprule
\textbf{Configuration} & \textbf{vs.\ CAGrad} & \textbf{vs.\ AuxiNash} & \textbf{vs.\ weighted sum} \\
\midrule
ResNet-34 / CIFAR-10  & $+26.3$ to $+29.7$ & $+19.5$ to $+71.7$ & $+35.4$ \\
ResNet-34 / CIFAR-100 & $+30.4$ to $+34.9$ & $+14.4$ to $+53.6$ & $+54.2$ \\
Inception / CIFAR-10  & $+10.2$ to $+47.3$ & $+9.2$ to $+47.2$  & no overlap \\
Inception / CIFAR-100 & $+7.6$ to $+39.0$  & $+7.4$ to $+26.0$  & no overlap \\
\bottomrule
\end{tabular}
\label{tab:k3-matched}
\end{table}

\subsection{Hypervolume Robustness}
\label{app:hv-robustness}

The dominated hypervolume of Fig.~\ref{fig:hypervolume} uses each method's full operating set ($15$ points for PCD, $10$ per swept baseline, $1$ for the hyperparameter-free baselines). Since the metric can reward cardinality and depends on a normalization frame, we repeat it under three adversarial variants. 

\emph{Equal cardinality}: reducing PCD to $10$ values of $\tau$ and evaluating all $\binom{15}{10} = 3003$ subsets, the \emph{worst} subset still exceeds the strongest baseline on every configuration (ratios $1.40/1.27/1.37/1.36$). This is a hard lower bound, since no choice of 10 $\tau$ values does worse. We do not instead place PCD on the baselines' grid: $\tau$ and $w$ are not interchangeable parameters, and that grid would delete the region $\tau < 0.1$ where PCD's useful operating points lie. 

\emph{One point per method}: selecting for each method the single point that maximizes its own hypervolume (selection by accuracy alone would use one of three axes and would hand PCD its no-compression $\tau = 0$ point), the ratios are $2.60/2.14/1.32/1.48$. 

\emph{Normalization bounds}: rebuilding each axis's bounds so that the catastrophically failed runs (MGDA, FAMO) cannot set them, and separately recomputing in one global frame shared by all four configurations, PCD's ratio over the strongest baseline goes from the published $3.41/2.78/2.32/2.57$ to $3.41/3.38/2.54/2.75$ and $2.43/2.78/2.80/4.03$ respectively: the failed runs \emph{compressed} the reported advantage rather than inflating it, and the ordering of the methods never changes under any variant.

\begin{table}[h]
\centering\small
\caption{Hypervolume robustness. Left: equal-cardinality test over all $\binom{15}{10}$ subsets of PCD's grid. Right: PCD $\div$ strongest baseline under alternative normalization frames.}
\begin{tabular}{l|cccc|c||ccc}
\toprule
 & \textbf{PCD 15} & \textbf{Worst 10} & \textbf{Mean} & \textbf{Best bl.} & \textbf{Worst$\div$bl.} & \textbf{Publ.} & \textbf{No M/F} & \textbf{Global} \\
\midrule
ResNet-34/C10  & 35.38 & 14.48 & 32.04 & 10.37 & 1.40 & 3.41 & 3.41 & 2.43 \\
ResNet-34/C100 & 39.03 & 17.90 & 35.77 & 14.07 & 1.27 & 2.78 & 3.38 & 2.78 \\
Inception/C10  & 13.08 & 7.74  & 11.74 & 5.64  & 1.37 & 2.32 & 2.54 & 2.80 \\
Inception/C100 & 12.30 & 6.50  & 10.92 & 4.79  & 1.36 & 2.57 & 2.75 & 4.03 \\
\bottomrule
\end{tabular}
\label{tab:hv-robustness}
\end{table}

\section{Adam and the Sign of the Guarantee}
\label{app:optimizer}

Proposition~\ref{prop:progress} concerns the direction $\vd^*$ handed to the optimizer. Adam's diagonal preconditioner and momentum do not need to preserve the sign of $\langle \gt_j, -\Delta\vth\rangle$ for the realized step $\Delta\vth$ (Sec.~\ref{sec:pcd-scope}). We measure this directly: at every optimizer step of 3-seed runs on two configurations and two tolerances, we log the \emph{sign-flip fraction} (steps where $\operatorname{sign}\langle \gt_2, -\Delta\vth\rangle \neq \operatorname{sign}\langle\gt_2, \vd^*\rangle$), the alignment the QP promises ($\cos(\gt_2, \vd^*)$), and the alignment the optimizer realizes ($\cos(\gt_2, -\Delta\vth)$), under the deployed Adam and under an SGD-with-momentum ($0.9$) control.

\begin{table}[h]
\centering\small
\caption{Sign preservation of the secondary-progress guarantee under the deployed optimizer (3 seeds, all steps of 300-epoch runs).}
\begin{tabular}{ll|c|ccc}
\toprule
\textbf{Optimizer} & \textbf{Configuration} & $\tau$ & \textbf{Sign flips} & $\cos(\gt_2, \vd^*)$ & $\cos(\gt_2, -\Delta\vth)$ \\
\midrule
Adam & ResNet-34/C100 & 0.01 & 1.5\,\%  & 0.017 & 0.262 \\
Adam & ResNet-34/C100 & 0.05 & 0.5\,\%  & 0.051 & 0.415 \\
Adam & DenseNet-121/C10 & 0.01 & 7.7\,\% & 0.047 & 0.262 \\
Adam & DenseNet-121/C10 & 0.05 & 0.2\,\% & 0.051 & 0.347 \\
\midrule
SGD+m & ResNet-34/C100 & 0.01 & 70.4\,\% & 0.016 & $-0.002$ \\
SGD+m & ResNet-34/C100 & 0.05 & 47.2\,\% & 0.050 & 0.004 \\
SGD+m & DenseNet-121/C10 & 0.01 & 73.3\,\% & 0.019 & $-0.007$ \\
SGD+m & DenseNet-121/C10 & 0.05 & 54.3\,\% & 0.050 & 0.003 \\
\bottomrule
\end{tabular}
\label{tab:sign-flips}
\end{table}

\begin{figure}[h]
    \centering
    \includegraphics[width=0.9\linewidth]{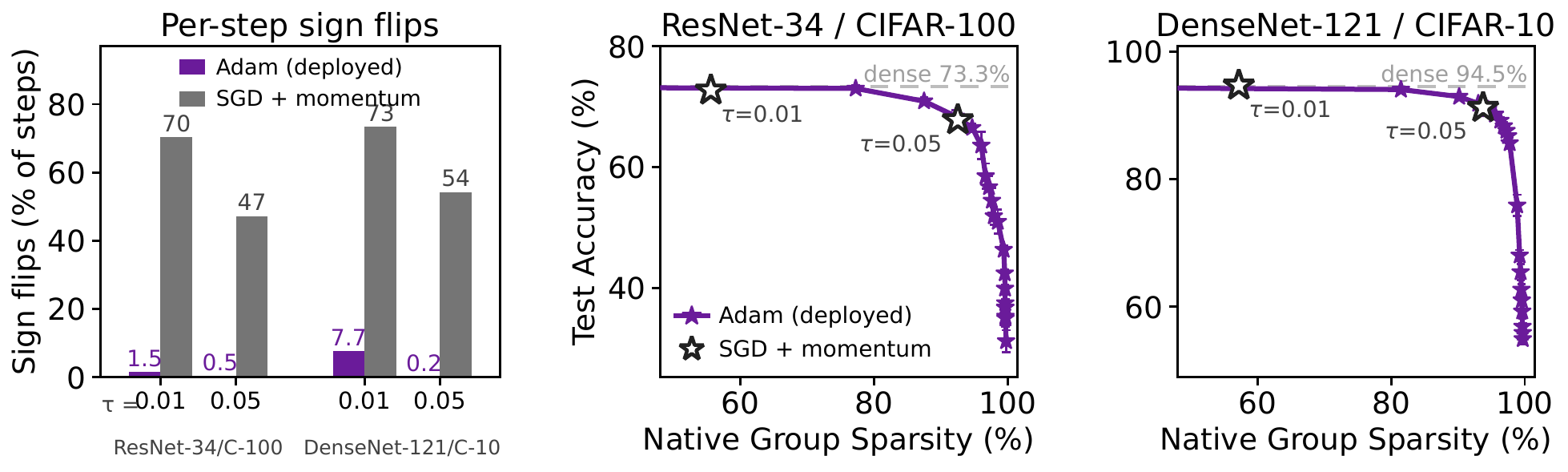}
    \caption{The optimizer gap, measured. \textbf{Left:} fraction of steps on which the realized parameter step disagrees in sign with the guaranteed secondary progress: rare under the deployed Adam, near-random under SGD with momentum. \textbf{Middle, right:} the SGD-with-momentum control (open stars) nevertheless lands on the same accuracy--sparsity frontier as deployed PCD, at both tolerances.}
    \label{fig:optimizer-control}
\end{figure}

Under the deployed Adam optimizer, the per-step sign survives on $92.3$--$99.8\%$ of steps, and the \emph{realized} alignment with the secondary is at least $5\times$ stronger than the QP's guaranteed margin (comparing the last two columns of Table~\ref{tab:sign-flips}). In this regime, Adam amplifies rather than erodes the constraint. Under SGD with momentum, the buffer dwarfs the small guaranteed margin $\tau\norm{\gt_2}^2$ and the per-step sign flips on $47$--$73\%$ of steps, with realized alignment $\approx 0$. However, the conclusions do not depend on the per-step sign: the SGD control lands on the same frontier at matched sparsity ($67.81 \pm 0.16$ at $92.55\%$ vs.\ PCD-Adam's $68.42$ at $92.24\%$ on ResNet-34/CIFAR-100, $91.19 \pm 0.29$ at $93.76\%$ between the $\tau = 0.03$ and $0.04$ points on DenseNet-121/CIFAR-10). Where the frontier is reached does differ from when: measured as the first epoch within one point of the run's own final accuracy, the SGD control plateaus $22$--$66$ epochs later than the Adam deployment across the four (configuration, $\tau$) cells, under the identical $300$-epoch cosine schedule, while final accuracies differ by at most $2.6\%$ in either direction. The mechanism is a trajectory-level bias, not a per-step invariant, and the study's guarantees are statements about $\vd^*$ accordingly. This gap is shared by every gradient-manipulation method that hands a constructed direction to an adaptive optimizer. We point readers to \citet{gong2021bi}'s appendix, which examines Adam and RMSProp trajectories for this algorithm family.

\section{Reproducibility Details}
All runs in this paper use Adam as the default optimizer (unless otherwise stated), with $\beta_1 = 0.9$ and $\beta_2 =0.999$. The default learning rate used was $10^{-3}$, a batch size of 128, cosine annealing, and everything was trained for 300 epochs. No warm-up (training pure cross-entropy) nor any weight decay was used. PCD's own constants are the EMA normalization used a $\beta=0.999$ with bias correction, and a solver stabilizer of $\varepsilon = 10^{-8}$. 5 seeds are used throughout the paper unless otherwise stated (3 in the case of the ablations). 

Operating points shown in figures are selected based on mean test accuracy: each swept baseline is reported at its highest mean test accuracy, and we use a fixed $\tau =0.02$ as a representative, unless otherwise stated or swept. No held-out validation split was used. 

Seed standard deviations accompany every table of the tuned baseline study (App.~\ref{app:tuned-expanded}), the normalization grid (App.~\ref{app:norm-ablation}), and every cell of the $K=3$ operating-point table (Table~\ref{tab:k3-sd}). For Table~\ref{tab:lowrank}, PCD's accuracy standard deviation is $0.15$--$0.40$ across the four configurations while the baselines' average $1.26$ with a maximum of $6.14$ (MGDA, ResNet-34/CIFAR-10). For ease of readability, only means are reported in the main table. We note that no conclusions in the study changes based on seed noise. 

\paragraph{Compute.} We use ResNet-34 training on CIFAR-100 as the representative here for $K = 2$. On NVIDIA A10G GPUs, the full PCD deployment (EMA normalization, QP solve, magnitude rescale and per-step instrumentation) measures $3.64 \pm 0.03$ hours, against weighted sum's $3.58 \pm 0.03$ hours. This is less than a $2\%$ overhead. For $K = 3$, PCD's wall-clock ratio to the weighted sum is $1.63$ to $1.95$, whilst being $\approx 1$ in comparison to the other multi-objective methods studied. We note that in the $K=3$ experiments, the repeated SVDs are \emph{not} specific to PCD. Per-step \texttt{svdvals} occur during the nuclear-norm loss that is observed by all 7 methods. PCD stores the $K$ objective gradients (as in the case of any multi-objective method), $K$ EMA scalars, and the $K\times K$ Gram matrix. No model-sized buffer beyond the stored gradients exist, and peak usage is indistinguishable from the weighted sum's at our scale. 

A more comprehensive reproducibility checklist will accompany this paper's code implementation.

\end{document}